\documentclass{article}

\PassOptionsToPackage{numbers, compress}{natbib}
\usepackage[final]{neurips_2022}

\usepackage[utf8]{inputenc} % allow utf-8 input
\usepackage[T1]{fontenc}    % use 8-bit T1 fonts
\usepackage{hyperref}       % hyperlinks
\usepackage{url}            % simple URL typesetting
\usepackage{booktabs}       % professional-quality tables
\usepackage{amsfonts}       % blackboard math symbols
\usepackage{nicefrac}       % compact symbols for 1/2, etc.
\usepackage{microtype}      % microtypography
\usepackage[table]{xcolor}         % colors
\usepackage{amssymb,amsmath,amsthm,amsopn}
\usepackage{ifxetex,ifluatex}
\usepackage{longtable,booktabs}
\usepackage{graphicx}
\usepackage{tabularx}
\usepackage{cleveref}
\usepackage{nicefrac}
\usepackage{mathtools}
\usepackage{dsfont}
\usepackage{multirow}
\usepackage{ifthen}
\usepackage{soul}
\usepackage[table]{xcolor}
\usepackage{multirow}
\usepackage{makecell}
\usepackage{caption}

\usepackage{caption}
\usepackage{subcaption}
\usepackage{wrapfig}

\newtheorem{theorem}{Theorem}
\newtheorem{proposition}[theorem]{Proposition}
\newtheorem{definition}[theorem]{Definition}
\newtheorem{lemma}[theorem]{Lemma}
\newtheorem{corollary}[theorem]{Corollary}
\newtheorem{remark}[theorem]{Remark}

\definecolor{Gray}{gray}{0.9}
\newcommand{\first}[1]{\textbf{\textcolor{red}{#1}}}
\newcommand{\second}[1]{\textbf{\textcolor{violet}{#1}}}
\newcommand{\third}[1]{\textbf{\textcolor{black}{#1}}}

\usepackage[inline]{enumitem}

\usepackage[labeled]{multibib}
\newcites{A}{References of Appendix}

\newcommand{\card}[1]{
\ifthenelse{\equal{#1}{}}{Empty.}{Nonempty.}
}

\newcommand{\WL}[1]{\ifthenelse{\equal{#1}{}}{\texttt{WL}}{\texttt{#1-WL}}}
\newcommand{\FWL}[1]{\texttt{#1-FWL}}
\newcommand{\IGN}[1]{\ifthenelse{\equal{#1}{}}{\texttt{IGN}}{\texttt{#1-IGN}}}
\newcommand{\IGNs}[1]{\ifthenelse{\equal{#1}{}}{\texttt{IGN}s}{\texttt{#1-IGN}s}}
\newcommand\SUN{\texttt{SUN}}
\newcommand{\ReIGN}[1]{\texttt{ReIGN(#1)}}

\newcommand{\onehot}[1]{\mathds{1}_{#1}}

\newcommand{\orbit}[1]{$\mathbf{o_{#1}}$}
\newcommand{\morbit}[1]{\mathbf{o}_{#1}}

\newcommand{\agg}{\mathop{\square}}

\newcommand{\cp}[4]{\boldsymbol{\kappa}^{#1:#2}_{#3:#4}\ }
\newcommand{\hemcp}[5]{\boldsymbol{\kappa}^{#1:#2}_{\nicefrac{#3:#4}{(#5)}}\ }
\newcommand{\selcp}[2]{\boldsymbol{\kappa}^{#1}_{#2}\ }

\newcommand{\mlp}[1]{\boldsymbol{\varphi}
    ^{\ifthenelse{\equal{#1}{}}{}{\boldsymbol{(#1)}}}\ }
\newcommand{\mlpcp}[5]{\boldsymbol{\varphi}
    ^{\ifthenelse{\equal{#5}{}}{}{\boldsymbol{(#5)}} #1:#2}
    _{#3:#4}\ }
\newcommand{\mlphemcp}[6]{\boldsymbol{\varphi}
    ^{\ifthenelse{\equal{#5}{}}{}{\boldsymbol{(#6)}} #1:#2}
    _{\nicefrac{#3:#4}{(#5)}}\ }
\newcommand{\selmlpcp}[3]{\boldsymbol{\varphi}
    ^{\ifthenelse{\equal{#3}{}}{}{\boldsymbol{(#3)}} #1}
    _{#2}\ }
    
\newcommand{\glue}[2]{[#1 \enspace #2]\ }
\newcommand{\sumglue}[2]{[#1 \thinspace | | \thinspace #2]\ }

\newcommand{\pool}[2]{\boldsymbol{\pi}^{#1}_{#2}\ }

\newcommand{\broad}[2]{\boldsymbol{\beta}^{#1}_{#2}\ }

\newcommand{\msgfn}[4]{\boldsymbol{\mu}^{#1}_{#2} \big ( 
    \ifthenelse{\equal{#3}{}}{}{[#3]} #4_{#1} \big )}

\newcommand{\revision}[1]{{\color{black}#1}}
\newcommand{\camera}[1]{{\color{black}#1}}
\newtheorem{conjecture}{Conjecture}

\title{Understanding and Extending Subgraph GNNs by Rethinking Their Symmetries}

\author{%
Fabrizio Frasca\thanks{Equal contribution. Author ordering determined by coin flip.} \\
Imperial College London \& Twitter \\
\texttt{ffrasca@twitter.com}
\And
Beatrice Bevilacqua$^{*}$ \\
Purdue University \\
\texttt{bbevilac@purdue.edu}
\And
Michael M. Bronstein \\
University of Oxford \& Twitter \\
\texttt{mbronstein@twitter.com}
\And
Haggai Maron \\
NVIDIA Research \\
\texttt{hmaron@nvidia.com}
}

\begin{document}

\maketitle

\begin{abstract}
Subgraph GNNs are a recent class of expressive Graph Neural Networks (GNNs) which model graphs as collections of subgraphs. So far, the design space of possible Subgraph GNN architectures as well as their basic theoretical properties are still largely unexplored. In this paper, we study the most prominent form of subgraph methods, which employs node-based subgraph selection policies such as ego-networks or node marking and deletion. We address two central questions: (1) \emph{What is the upper-bound of the expressive power of these methods?} and (2) \emph{What is the family of equivariant message passing layers on these sets of subgraphs?}. 
Our first step in answering these questions is a novel symmetry analysis which shows that modelling the symmetries of node-based subgraph collections requires a significantly smaller symmetry group than the one adopted in previous works. This analysis is then used to establish a link between Subgraph GNNs and Invariant Graph Networks (IGNs). 
We answer the questions above by first bounding the expressive power of subgraph methods by 3-WL, and then proposing a general family of message-passing layers for subgraph methods that generalises all previous node-based Subgraph GNNs. Finally, we design a novel Subgraph GNN dubbed \SUN{}, which theoretically unifies previous architectures while providing better empirical performance on multiple benchmarks.
\end{abstract}

\section{Introduction}
% motivation

Message Passing Neural Networks (MPNNs) are arguably the most commonly used version of Graph Neural Networks (GNNs). 
The limited expressive power of MPNNs~\citep{morris2019weisfeiler,xu2019how} has led to a plethora of works aimed at designing expressive GNNs while maintaining the simplicity and scalability of MPNNs~\citep{bouritsas2022improving,morris2021weisfeiler,sato2020survey,Li2022}.
Several recent studies have proposed a new class of such architectures~\citep{cotta2021reconstruction,zhang2021nested, bevilacqua2022equivariant,zhao2022from,papp2021dropgnn,papp2022theoretical}, dubbed \emph{Subgraph GNNs}, which apply MPNNs to collections (`bags') of subgraphs extracted from the original input graph and then aggregate the resulting representations. Subgraphs are selected according to a predefined policy; in the most popular ones, each subgraph is tied to a specific node in the original graph, for example by deleting it or extracting its local ego-network. Subgraph GNNs have demonstrated outstanding empirical performance, with state-of-the-art results on popular benchmarks like the ZINC molecular property prediction~\citep{zhao2022from,bevilacqua2022equivariant}.

% challenge and previous work
While offering great promise, it is fair to say that we still lack a full understanding of Subgraph GNNs. Firstly, on the theoretical side, it is known that subgraph methods are strictly stronger than the Weisfeiler-Leman (WL) test~\citep{weisfeiler1968reduction,morris2021weisfeiler}, but an upper-bound on their expressive power is generally unknown. Secondly, on a more practical level, Subgraph GNN architectures differ considerably in the way information is aggregated and shared across the subgraphs, and an understanding of the possible aggregation and sharing rules is missing. Both aspects are important: an understanding of the former can highlight the limitations of emerging architectures, a study of the latter paves the way for improved Subgraph GNNs.

% Goal of the paper  and main tool
\textbf{Main contributions.}
The goal of this paper is to provide a deeper understanding of node-based Subgraph GNNs in light of the two aforementioned aspects. The main theoretical tool underpinning our contributions is a novel analysis of the symmetry group that acts on the sets of subgraphs. While several previous approaches~\citep{papp2021dropgnn, cotta2021reconstruction, bevilacqua2022equivariant} have (often implicitly) assumed that a subgraph architecture should be equivariant to independent node and subgraph permutations, we leverage the fact that node-based policies induce an inherent bijection between the subgraphs and the nodes. This observation allows us to align the two groups and model the symmetry with a single (smaller) permutation group that acts on nodes and subgraphs \emph{jointly}. Other works~\citep{zhao2022from,you2021identity,zhang2021nested} have (again, implicitly) recognised such node-subgraph correspondence but without studying the implications on the symmetry group, and resorting, as a result, to a partial and heuristic choice of equivariant operations.

The use of this stricter symmetry group raises a fruitful connection with $k$-order Invariant Graph Networks (\IGNs{k})~\citep{maron2018invariant,maron2019provably}, a well studied family of architectures for processing graphs and hypergraphs designed to be equivariant to the same symmetry group. This connection allows us to transfer and reinterpret previous results on \IGNs{} to our Subgraph GNN setup.
%
% bounding the expressive power
As our first contribution we show that the expressive power of Subgraph GNNs with node-based policies is bounded by that of the \WL{3} test. This is shown by proving that all previous Subgraph GNNs can be implemented by a \IGN{3} and by leveraging the fact that the expressive power of these models \revision{is} bounded by  \WL{3}~\citep{geerts2020expressive,azizian2020characterizing}.

% proposing the layer space, new models and empirical validation
Our second contribution is the proposal of a general layer formulation for Subgraph GNNs, based on the observation that these methods maintain an $n \times n$ representation of $n$ subgraphs with $n$ nodes,  following the same symmetry structure of \IGNs{2} (same permutation applied to both rows and columns of this representation). We propose a novel extension of \IGNs{2} capturing both local (message-passing-like) and global operations. This extension easily recovers previous methods \revision{facilitating} their comparison. \revision{Also}, we present a number of new operations that previous methods did not implement. We build upon these observations \revision{to} devise a new Subgraph GNN dubbed \SUN{}, (\emph{Subgraph Union Network}). We prove that \SUN{} generalises all previous node-based Subgraph GNNs and we \revision{empirically compare it} to these methods, showing it can outperform them.

\section{Previous and related work}
\textbf{Expressive power of GNNs.} 
The expressive power of GNNs is a central research focus since it was realised that message-passing type GNNs are constrained by the expressivity of the WL isomorphism test~\citep{morris2019weisfeiler,xu2019how}. Other than the aforementioned subgraph-based methods, numerous approaches for more powerful GNNs have been proposed, including positional and structural encodings~\citep{abboud2020surprising,puny2020global,bouritsas2022improving,dwivedi2021graph,kreuzer2021rethinking,lim2022sign}, higher-order message-passing schemes~\citep{morris2019weisfeiler,morris2020weisfeiler,bodnar2021weisfeilerA,bodnar2021weisfeilerB}, equivariant models~\citep{hy2019covariant,maron2018invariant,maron2019provably,vignac2020building,de2020natural, thiede2021autobahn,morris2022speqnets}. We refer readers to the recent survey by \citet{morris2021weisfeiler} for additional details. Finally we note that, in a related and concurrent work,~\citet{Qian2022osan} propose a theoretical framework to study the expressive power of subgraph-based GNNs by relating them to the \WL{k} hierarchy, and explore how to sample subgraphs in a data-driven fashion.

\textbf{Invariant graph networks.} 
\IGNs{} were recently introduced in a series of works by \citet{maron2018invariant, maron2019provably, maron2019universality} as an alternative to MPNNs for processing graph and hyper-graph data. For $k\geq2$, \IGNs{k} \revision{represent hyper-graphs with hyper-edges up to size $k$ with $k$-order tensor $\mathcal{Y} \in \mathbb{R}^{n^k}$, where each entry holds information about a specific hyper-edge. On these they} apply linear $S_n$-equivariant layers interspersed with pointwise nonlinearities. 
\revision{These models have been thoroughly studied in terms of: (i) their expressive power; (ii) the space of their equivariant linear layers. As for (i), \IGNs{} were shown to have exactly the same graph separation power as the \WL{k} graph isomorphism test~\citep{maron2019provably,azizian2020characterizing,geerts2020expressive} and, for sufficiently large $k$, to have a universal approximation property w.r.t. $S_n$-invariant and equivariant functions~\citep{maron2019universality,keriven2019universal,ravanbakhsh2020universal}. Concerning (ii), the work in \citep{maron2018invariant} completely characterised the space of linear layers equivariant to $S_n$ from $\mathbb{R}^{n^k}$ to $\mathbb{R}^{n^{k'}}$: the authors derived a basis of $\text{bell}(k+k')$ linear operators consisting of indicator tensors of equality patterns over the multi-index set $\{1,\dots,n\}^{k+k'} = [n]^{k+k'}$.
\citet{Albooyeh2019incidence} showed these layers can be (re-)written as sums of pooling-broadcasting operations between elements of $\mathcal{Y}$ indexed by the orbits \footnote{\revision{For group $G$ acting on set $X$, the orbits of the action of $G$ on $X$ are defined as $\{G \cdot x \thinspace | \thinspace x \in X \}$. These \emph{partition} $X$ into subsets whose elements can (only) reach all other elements in the subset via the group action.}} of the action of $S_n$ on $[n]^{k}$ and $[n]^{k'}$. Take, e.g., $k=k'=2$. In this case there are only two orbits: $\{i,i\},~i\in [n]$ corresponding to on-diagonal terms, and $\{i,j\},~i\neq j\in [n]$, off-diagonal terms. According to~\citet{Albooyeh2019incidence} any equivariant linear layer $L:\mathbb{R}^{n^2} \rightarrow \mathbb{R}^{n^2}$ can be represented as a composition of pooling and broadcasting operations on the elements indexed by these orbits. One example is the linear map that sums the on-diagonal elements and broadcasts the result to the off-diagonal ones: $L(\mathcal{Y})_{ij} = \sum_k \mathcal{Y}_{kk}$ for $i \neq j$, $0$ otherwise. See \Cref{app:upperbound}, for additional details. These results particularly important as they underpin most of our theoretical derivations.}
\revision{Lastly, a more comprehensive coverage of \IGNs{} can be found in \citep{morris2021weisfeiler}.}

\camera{
\textbf{Subgraph GNNs.} Despite motivated by diverse premises, a collection of concurrent methods share the overarching design whereby graphs are modelled through the application of a GNN to their subgraphs. \citet{bevilacqua2022equivariant} first explicitly formulated the concept of bags of subgraphs generated by a predefined policy and studied layers to process them in an equivariant manner: the same GNN can encode each subgraph independently (DS-GNN), or information can be shared between these computations in view of the alignment of nodes across the bag~\citep{maron2020learning} (DSS-GNN). Building upon the Reconstruction Conjecture~\citep{kelly1957a,ulam1960a}, Reconstruction GNNs~\citep{cotta2021reconstruction} obtain node-deleted subgraphs, process them with a GNN and then aggregate the resulting representations by means of a set model. Nested GNNs~\citep{zhang2021nested} and GNN-As-Kernel models (GNN-AK)~\citep{zhao2022from} shift their computation from rooted subtrees to rooted subgraphs, effectively representing nodes by means of GNNs applied to their enclosing ego-networks. Similarly to DSS-GNNs~\citep{bevilacqua2022equivariant}, GNN-AK models may feature information sharing modules aggregating node representations across subgraphs. ID-GNNs~\citep{you2021identity} also process ego-network subgraphs, but their roots are `marked' so to specifically alter the exchange of messages involving them. Intuitively, the use of subgraphs implicitly breaks those local symmetries which determine the notorious expressiveness bottleneck of MPNNs. We note that other works can be interpreted as Subgraph GNNs, including those by~\citet{papp2021dropgnn,papp2022theoretical}.
}

\section{Node-based Subgraph GNNs}\label{sec:subgraph_gnn} % NB: 'Node-based' is in camera-revision
% In this section, we formulate Subgraph GNNs and survey their recent instantiations. 

\textbf{Notation.} Let $G = (A, X)$ be a member of the family $\mathcal{G}$ of \emph{node-attributed}, undirected, finite, simple graphs\footnote{We do not consider  edge features, although an extension to such a setting would be possible.}. The {\em adjacency matrix} $A \in \mathbb{R}^{n \times n}$ represents $G$'s edge set $E$ over its set of $n$ nodes $V$. The {\em feature matrix} $X \in \mathbb{R}^{n \times d}$ gathers the node features; we denote by $x_j \in \mathbb{R}^{d \times 1}$ the features of node $j$ corresponding to the $j$-th row of $X$. $B_G$ is used to denote a multiset (bag) of $m$  subgraphs of $G$. Adjacency and feature matrices for subgraphs in $B_G$ are arranged in tensors $\mathcal{A} \in \mathbb{R}^{m \times n \times n}$ and $ \mathcal{X} \in \mathbb{R}^{m \times n \times d}$. Superscript $^{i,(t)}$ refers to representations on subgraph $i$ at the $t$-th layer of a stacking, as in $x_{j}^{i, (t)}$. Finally, we denote $[n] = \{ 1, \mathellipsis, n\}$. All proofs are deferred to \Cref{app:upperbound,app:space}.

\textbf{Formalising Subgraph GNNs.} Subgraph GNNs compute a representation of $G \in \mathcal{G}$ as 
\begin{equation}\label{eq:subgraph_gnn}
    (A, X) \mapsto \big ( \mu \circ \rho \circ \mathcal{S} \circ \pi \big ) ( A, X ).
\end{equation}
Here, $\pi: G \mapsto \{G^1,...,G^m\} = \{(A^1, X^1),...,(A^m, X^m)\} = B^{(0)}_G$ is a {\em selection policy} generating a bag of subgraphs from $G$; $\mathcal{S} = L_T \circ \mathellipsis \circ L_1: B^{(0)}_G \mapsto B^{(T)}_G$ is a {\em stacking} of $T$ (node- and subgraph-) permutation equivariant layers; $\rho: (G, B^{(T)}_G) \mapsto x_G$ is a permutation invariant {\em pooling function}, $\mu$ is an MLP. \camera{The layers in $\mathcal{S}$ comprise a \emph{base-encoder} in the form of a GNN applied to subgraphs; throughout this paper, we assume it to be a \WL{1} maximally expressive MPNN such as the one in \citet{morris2019weisfeiler}.} Subgraph GNNs differ in the implementation of $\pi$, $\mathcal{S}$ and, in some cases, $\rho$. For example, in \revision{(n-1)}-Reconstruction GNNs~\citep{cotta2021reconstruction}, $\pi$ selects node-deleted subgraphs and $\mathcal{S}$ applies a Siamese MPNN to each subgraph independently. To exemplify the variability in $\mathcal{S}$, \camera{DSS-GNN}~\citep{bevilacqua2022equivariant} extends this method with cross-subgraph node and connectivity aggregation. \camera{More details are on how currently known Subgraph GNNs are captured by \Cref{eq:subgraph_gnn} can be found in \Cref{app:subgraphs}.}

\textbf{Node-based selection policies.} In this work, we focus on a specific family of \emph{node-based} subgraph selection policies, wherein every subgraph is associated with a unique node in the graph. 
Formally, we call a subgraph selection policy \emph{node-based} if it is of the form $\pi(G) = \{ f(G,v) \}_{v \in V}$, for some \emph{selection function} $f(G,v)$ that takes a graph $G$ and a node $v$ as inputs and outputs a subgraph $G^v$. In the following, we refer to $v$ as the \emph{root} of subgraph $G^v$. We require $f$ to be a bijection and we note that such policies produce $m=n$ different subgraphs. Amongst the most common examples are {\em node-deletion} (ND), {\em node-marking} (NM), and {\em ego-networks} (EGO) policies. For input graph $G$, $f_{\mathrm{ND}}(G,v)$ removes node $v$ and the associated connectivity; $f_{\mathrm{NM}}(G,v)$ adds a special `mark' attribute to $v$'s features (with no connectivity alterations), and  $f_{\mathrm{EGO}(h)}(G,v)$ returns the subgraph induced by the $h$-hop-neighbourhood around the root $v$. EGO policies can be `marked': $f_{\mathrm{EGO+}(h)}(G,v)$ extracts the $h$-hop ego-net around $v$ \emph{and} marks this node as done by $f_{\mathrm{NM}}$. For convenience, we denote the class of such node-based selection policies by $\Pi$:
\begin{definition}[Known node-based selection policies $\Pi$]\label{def:policies}
    Let $\Sigma$ be the set of all node-based subgraph selection policies operating on $\mathcal{G}$. Class $\Pi \subset \Sigma$ collects the node-based policies node-deletion (ND), node-marking (NM), ego-nets (EGO) and marked ego-nets (EGO+) of any depth: $\Pi = \{ \pi_{\mathrm{ND}}, \pi_{\mathrm{NM}}, \pi_{\mathrm{EGO}(h)}, \pi_{\mathrm{EGO+}(h)} \thinspace | \thinspace h > 0 \}$.
\end{definition}

\camera{\textbf{Node-based Subgraph GNNs} are those Subgraph GNNs which, implicitly or explicitly, process bags generated by node-based policies. We group known formulations in the following family:}
\begin{definition}[Known \camera{node-based} Subgraph GNNs $\Upsilon$]\label{def:subgraph_gnns}
    Let $\Xi$ be the set of all node-based Subgraph GNNs. Class $\Upsilon \subset \Xi$ collects known Subgraph GNNs when equipped with \WL{1} base-encoders: $\Upsilon = \{$\emph{\revision{(n-1)}-Reconstr.GNN}, \emph{GNN-AK}, \emph{GNN-AK-ctx}, \emph{NGNN}, \emph{ID-GNN}, \camera{\emph{DS-GNN}$_\Pi$, \emph{DSS-GNN}$_\Pi \}$. \emph{DS-GNN}$_\Pi$, \emph{DSS-GNN}$_\Pi$ refer to \emph{DS-} and \emph{DSS-GNN} models equipped with any $\pi \in \Pi$}.%;  GNN-AK, GNN-AK-ctx, NGNN are equipped with $\pi_{EGO}$; ID-GNN with policy $\pi_{EGO+}$,  \revision{(n-1)}-Reconstr.GNN with policy $\pi_{ND}$.
\end{definition}

Importantly, all these methods apply MPNNs to subgraphs of the original graph, but differ in the way information is shared between subgraphs/nodes. In all cases, their expressive power is strictly larger than \WL{1}, but an \emph{upper}-bound is currently unknown.

\section{Symmetries of node-based subgraph selection policies}

\begin{figure}
    \centering
    \includegraphics[width=\linewidth]{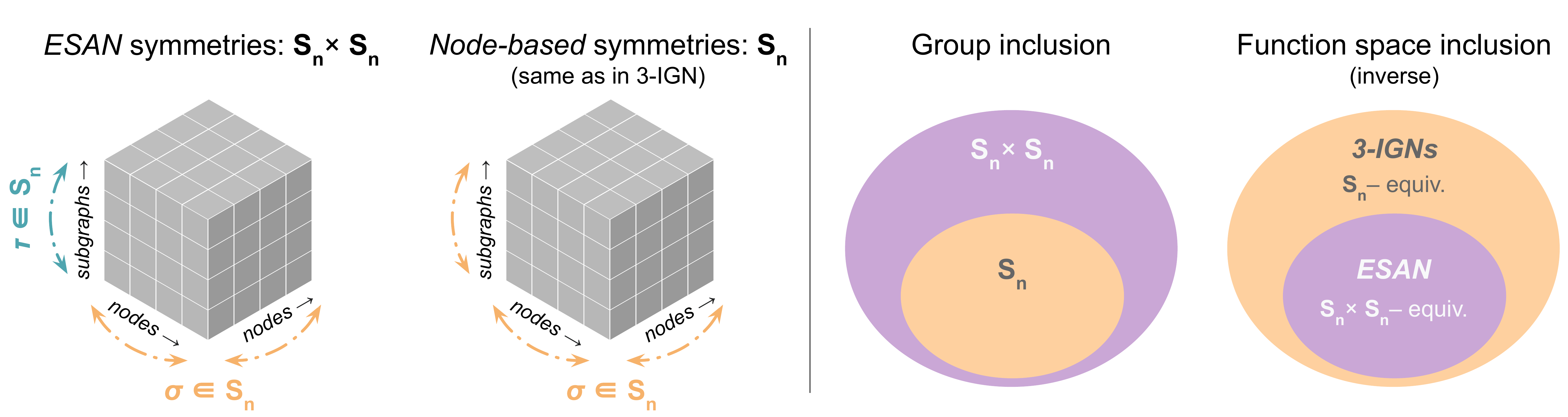}
    \caption{Symmetries of bags of subgraphs (left) and corresponding function space diagrams (right). In ESAN~\citep{bevilacqua2022equivariant} symmetries are modelled as a direct product of node and subgraph permutation groups; however, node-based policies enable the use of one single permutation group, the same as in \texttt{3-IGN}s. \texttt{3-IGN}s are less constrained, thus more expressive than ESAN and other Subgraph GNNs. See diagram on the right and formal statement in Section \ref{sec:upperbound}.}%\vspace{-5mm}
    \label{fig:symmetries}
\end{figure}

In an effort to characterise the representational power of node-based Subgraph GNNs, we first study the symmetry group of the objects they process: %A standard GNN takes as input a tensor $(A, X)$ %where $A \in \mathbb{R}^{n \times n}$ is the graph's adjacency matrix and $X \in \mathbb{R}^{n\times d}$ is the node feature matrix. 
%In contrast, 
%Subgraph GNNs operate on 
`bags of subgraphs' represented as tensors $(\mathcal{A}, \mathcal{X}) \in \mathbb{R}^{m \times n \times n} \times \mathbb{R}^{m \times n \times d}$,  assuming $n$ nodes across $m$ subgraphs. 
%Here, $\mathcal{A} \in \mathbb{R}^{n\times n \times m}$ represents a set of $m$ adjacency matrices, and $\mathcal{X}\in \mathbb{R}^{n\times d \times m}$ represents a set of $m$ node feature matrices. 
%
% To model the symmetries of $(\mathcal{A},\mathcal{X})$, 
Previous approaches~\citep{cotta2021reconstruction,bevilacqua2022equivariant,papp2021dropgnn} used two permutation groups: one copy of the symmetric group $S_n$ models {\em node permutations}, while another copy $S_m$ models {\em subgraph permutations} in the bag. These two were combined by a group product\footnote{\citet{bevilacqua2022equivariant} use a direct-product,  assuming nodes in subgraphs are consistently ordered. \citet{cotta2021reconstruction} use the larger wreath-product assuming node ordering in the subgraph is unknown.} acting \emph{independently} on the nodes and subgraphs in $(\mathcal{A},\mathcal{X})$. For example, \citet{bevilacqua2022equivariant} model the symmetry as:
 \begin{equation}\label{eq:old_symmetry}
  ((\textcolor{teal}{\tau},\textcolor{orange}{\sigma}) \cdot \mathcal{A})_{ijk} = \mathcal{A}_{\textcolor{teal}{\tau}^{-1}(i)\textcolor{orange}{\sigma}^{-1}(j)\textcolor{orange}{\sigma}^{-1}(k)}, \enspace \thinspace ((\textcolor{teal}{\tau},\textcolor{orange}{\sigma}) \cdot \mathcal{X})_{ijl}=\mathcal{X}_{\textcolor{teal}{\tau}^{-1}(i)\textcolor{orange}{\sigma}^{-1}(j)l}, \enspace \quad (\textcolor{teal}{\tau},\textcolor{orange}{\sigma})\in S_m\times S_n
\end{equation}

Our contributions stem from the following crucial observation: %We make a crucial observation: 
When using node-based policies, the subgraphs in $(\mathcal{A},\mathcal{X})$ can be {\em ordered consistently with the nodes} by leveraging the bijection $f: v \mapsto G_v$ characterising this policy class. In other words, $f$ suggests a node-subgraph alignment inducing a new structure on $(\mathcal{A},\mathcal{X})$, whereby the subgraph order is not independent of that of nodes anymore. Importantly, this new structure is preserved \emph{only} by those permutations operating identically on both nodes and subgraphs. Following this observation, the symmetry of a node-based bag of subgraphs is modelled more accurately using only {\em one single permutation group} $S_n$ jointly acting on both nodes and subgraphs:
\begin{equation}\label{eq:symmetry}
 (\textcolor{orange}{\sigma} \cdot \mathcal{A})_{ijk}=\mathcal{A}_{\textcolor{orange}{\sigma}^{-1}(i)\textcolor{orange}{\sigma}^{-1}(j)\textcolor{orange}{\sigma}^{-1}(k)}, \enspace \thinspace (\textcolor{orange}{\sigma} \cdot \mathcal{X})_{ijl}=\mathcal{X}_{\textcolor{orange}{\sigma}^{-1}(i)\textcolor{orange}{\sigma}^{-1}(j)l}, \enspace \quad \textcolor{orange}{\sigma} \in S_n
\end{equation}
It should be noted that $S_n$ is {\em  significantly smaller} than  $S_n \times S_n$\footnote{More formally, $S_n$'s orbits on the indices in \eqref{eq:symmetry} refine the orbits of the product group in \eqref{eq:old_symmetry}.}. Informally, the latter group contains many permutations which are not in the former: those acting differently on nodes and subgraphs and, thus, not preserving the new structure of $(\mathcal{A},\mathcal{X})$. Since they are restricted by a smaller set of equivariance constraints, we expect GNNs designed to be equivariant to $S_n$ to be be more expressive than those equivariant to the larger groups considered by previous works~\citep{maron2019universality} \revision{(see Figure~\ref{fig:symmetries})}.

The insight we obtain from Equation \eqref{eq:symmetry} is profound: it reveals that the symmetry structure of $\mathcal{A}$ exactly matches the symmetries of third-order tensors used by \IGNs{3}, and similarly, that the symmetry structure for $\mathcal{X}$ matches the symmetries of second-order tensors used by \IGNs{2}. 
In the following, we will make use of this insight and the fact that \IGNs{} are well-studied objects to prove an upper-bound on the expressive power of Subgraph GNNs and to design principled extensions to these models.
\begin{figure}[t]
    \centering%\vspace{-1mm}
    \includegraphics[width=0.75\linewidth]{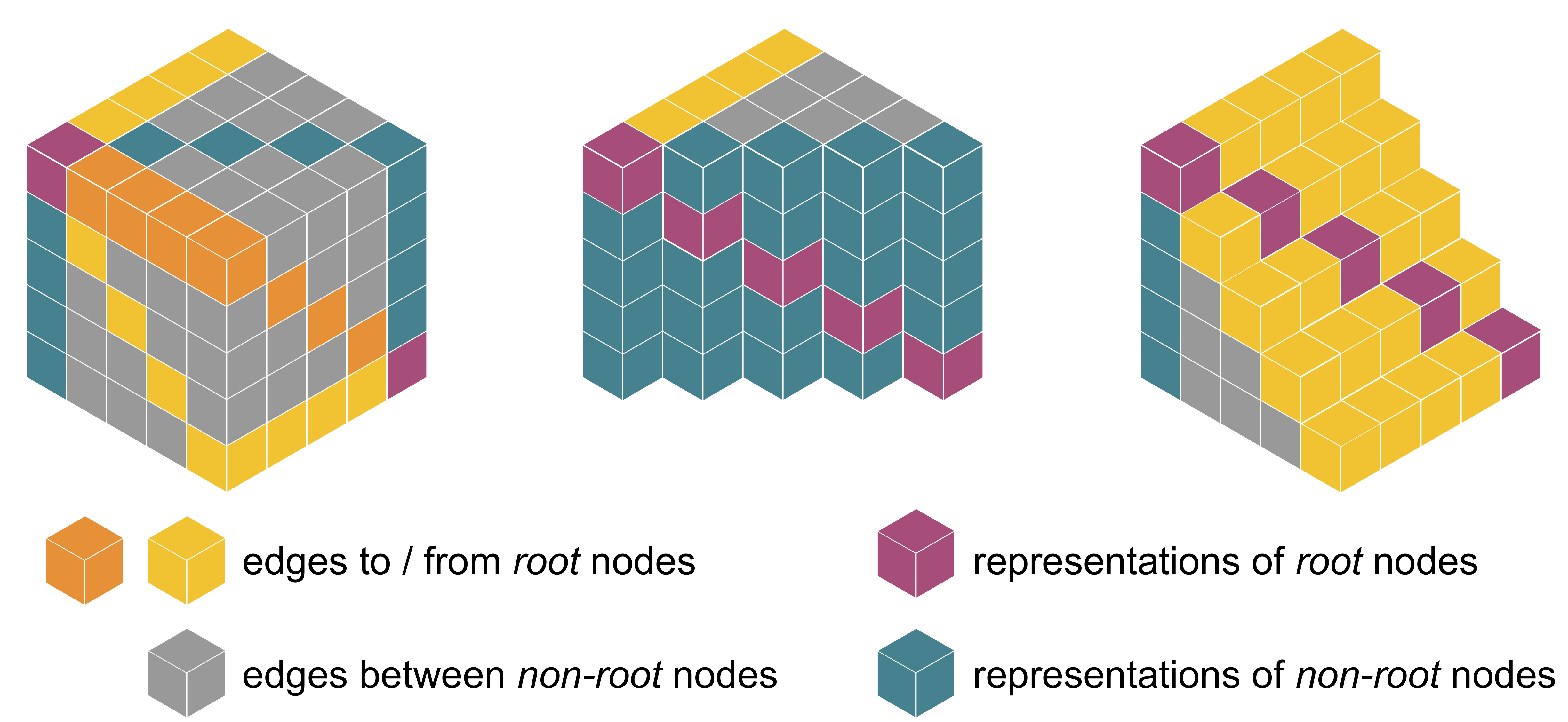}%\vspace{-2mm}
    \caption{\camera{Depiction of cubed tensor $\mathcal{Y}$, its orbit-induced partitioning and the related semantics when $\mathcal{Y}$ is interpreted as a bag of node-based subgraphs, $n = 5$. Elements in the same partition are depicted with the same colour. Left: the whole tensor. Middle and right: sections; elements in purple and green constitute sub-tensor $\mathcal{X}$, the remaining ones sub-tensor $\mathcal{A}$.}}%\vspace{-5mm}
    \label{fig:orbit_dec_sem}
\end{figure}
\revision{We remark that bags of node-based subgraphs can also be represented as tensors $\mathcal{Y} \in \mathbb{R}^{n^3 \times d}$, the same objects on which \IGNs{3} operate. Here, $\mathcal{X}$ is embedded in the main diagonal plane of $\mathcal{Y}$, $\mathcal{A}$ in its remaining entries. Within this context, it is informative to study the semantics of the $5$ orbits induced by the action of $S_n$ on $\mathcal{Y}$'s multi-index set $[n]^3$: each of these uniquely identify root nodes, non-root nodes, edges to and from root nodes as well as edges between non-root nodes (\camera{see \Cref{fig:orbit_dec_sem} and} additional details in \Cref{app:3IGN_data_struct}). We build upon this observation, along with the layer construction by~\citet{Albooyeh2019incidence}, to prove many of the results presented in the following.}\label{sec:symmetries}

\section{A representational bound for Subgraph GNNs}\label{sec:upperbound}

In this section we prove that the expressive power of known node-based Subgraph GNNs is bounded by \WL{3} by showing that they can be implemented by \IGNs{3}, which have the same expressive power as \WL{3}.
Underpinning the possibility of \IGNs{} to upper-bound a certain Subgraph GNN $\mathcal{N}$ in its expressive power is the ability of \IGNs{} to (i) implement $\mathcal{N}$'s subgraph selection policy ($\pi$) and (ii) implement $\mathcal{N}$'s (generalised) message-passing and pooling equations ($\mu \circ \rho \circ \mathcal{S}$). This would ensure that whenever $\mathcal{N}$ assigns distinct representations to two non-isomorphic graphs, an \IGN{} implementing $\mathcal{N}$ would do the same.
We start by introducing a \revision{recurring, useful concept.}
\begin{definition}[``implements'']\label{def:implements}
    Let $f: D_f \rightarrow C_f$, $g: D_g \rightarrow C_g$ be two functions and such that $D_g \subseteq D_f, C_g \subseteq C_f$. We say $f$ \emph{implements} $g$ (and write $f \cong g$) when $\forall x \in D_g, f(x) = g(x)$.
\end{definition}

Our first result shows that \IGNs{3} can implement the selection policies in class $\Pi$ (Definition~\ref{def:policies}), which, to the best of our knowledge, represent \emph{all known} node-based policies utilised by previously proposed Subgraph GNNs.
\begin{lemma}[\camera{\IGNs{3} implement known node-based selection policies}]\label{lemma:3IGN_implements_policies}
    For any $\pi \in \Pi$ there exists a stacking of \IGN{3} layers $\mathcal{M}_{\pi}$ s.t. $\mathcal{M}_{\pi} \cong \pi$.%$\forall G \in \mathcal{G}, \mathcal{M}_{\pi}(G) = \pi(G)$.
\end{lemma}
Intuitively, \IGNs{3} start from a $\mathbb{R}^{n^2}$ representation of $G$ and, first, move to a $\mathbb{R}^{n^3}$ tensor `copying' this latter along its first (subgraph) dimension. This is realised via an appropriate broadcast operation. Then, they proceed by adding a `mark' to the features of some nodes and/or by nullifying elements corresponding to some edges. \camera{We refer readers to \Cref{fig:orbit_dec_sem} and \Cref{app:bag_interpretation} for additional details on how nodes in each subgraph are represented in \IGNs{3}}.
Next, we show \IGNs{3} can implement \revision{layers of any model $\in \Upsilon$}. 
\begin{lemma}[\IGNs{3} implement Subgraph GNN layers]\label{lemma:3IGN_implements_SubgraphNetworks}
    Let $G_1, G_2$ be two graphs in $\mathcal{G}$ and $\mathcal{N}$ a model in family $\Upsilon$ equipped with \citet{morris2019weisfeiler} message-passing base-encoders. Let $B^{(t)}_1, B^{(t)}_2$ be bags of subgraphs in the input of some intermediate layer $L$ in $\mathcal{N}$. Then there exists a stacking of \IGN{3} layers $\mathcal{M}_L$ for which $\mathcal{M}_{L}(B^{(t)}_i) = B^{(t+1)}_i = L(B^{(t)}_i)$ for $i=1,2$.
\end{lemma}
% we can rewrite the above more compactly by defining implementation on a subset of the target function domain.

Lemmas~\ref{lemma:3IGN_implements_policies} and~\ref{lemma:3IGN_implements_SubgraphNetworks} allow us to upper-bound the expressive power of all known instances of node-based Subgraph GNNs by that of \IGNs{3}:
\begin{theorem}[\camera{\IGNs{3} upper-bound node-based Subgraph GNNs}]\label{thm:3IGN_upperbounds_SubgraphNetworks}
   For any pair of non-isomorphic graphs $G_1, G_2$ in family $\mathcal{G}$ and Subgraph GNN model $\mathcal{N} \in \Upsilon$ equipped with \citet{morris2019weisfeiler} message-passing base-encoders, if there exists weights $\Theta$ such that $G_1$, $G_2$ are distinguished by instance $\mathcal{N}_\Theta$, then there exist weights $\Omega$ for a $3$-IGN instance $\mathcal{M}_{\Omega}$ such that $G_1, G_2$ are distinguished by $\mathcal{M}_{\Omega}$ as well.
\end{theorem}

% Consequences on expressive power: upperbound by 3-WL
Theorem~\ref{thm:3IGN_upperbounds_SubgraphNetworks} has profound consequences in the characterisation of the expressive power of node-based Subgraph GNNs, as we show in the following %corollary.
\begin{corollary}[\camera{\WL{3} upper-bounds node-based Subgraph GNNs}]\label{cor:upperbound}
    Let $G_1, G_2 \in \mathcal{G}$ be two non-isomorphic graphs  and $\mathcal{N}_{\Theta} \in \Upsilon$ one instance of model $\mathcal{N}$ with weights $\Theta$. If $\mathcal{N}_{\Theta}$ distinguishes $G_1, G_2$, then the \WL{3} algorithm does so as well.
\end{corollary}
\emph{Proof idea}: If there is a pair of graphs undistinguishable by \WL{3}, but for which there exists a Subgraph GNN separating them, there must exists a \IGN{3} separating these (Theorem~\ref{thm:3IGN_upperbounds_SubgraphNetworks}). This is in contradiction with the result \camera{by}~\citet{geerts2020expressive, azizian2020characterizing}\footnote{\WL{k} is equivalent to \FWL{(k-1)}, i.e. the ``Folklore'' WL test, see~\citep{morris2019weisfeiler}.}.

\section{A design space for Subgraph GNNs}\label{sec:space}

As discussed, different formulations of Subgraph GNNs differ primarily in the  specific rules for updating node representations across subgraphs. However, until now it is not clear whether existing rules exhaust all the possible equivariant options.
We devote this section to a systematic characterisation of the `layer space' of Subgraph GNNs.%, with the aim of achieving a deeper comprehension of this class of models and a principled framework to guide their future design.

In the spirit of the previous Section~\ref{sec:upperbound}, where we ``embedded'' Subgraph GNNs in \IGNs{3}, one option would be to consider all $\text{bell}(6)=203$ linear equivariant operations prescribed by this formalism. However, this choice would be problematic for three main reasons: (i) This layer space is \emph{too vast} to be conveniently explored; (ii) It includes operations involving $\mathcal{O}(n^3)$ space complexity, impractical in most applications; (iii) The linear \IGN{} basis does not directly support local message passing, a key operation in \revision{s}ubgraph methods. Following previous Subgraph GNN variants, which use $\mathcal{O}(n^2)$ storage for the representation of $n$ nodes in $n$ subgraphs, we set the desideratum of $\mathcal{O}(n^2)$ memory complexity as our main constraint, and use this restriction to reduce the design space. \revision{Precisely}, we are interested in modelling $S_n$-equivariant transformations on the subgraph-node tensor $\mathcal{X}$.

\subsection{Extended 2-IGNs} As we have already observed in Equation~\ref{eq:symmetry} in Section~\ref{sec:symmetries}, such a second order tensor $\mathcal{X}$ abides by the same symmetry structure of \IGNs{2}. We therefore gain intuition from the characterisation of linear equivariant mappings as introduced by~\citet{maron2018invariant}, and propose an extension of this formalism.

\textbf{\texttt{2-IGN} layer space.} A \IGN{2} layer $L_{\Theta}$ updates $\mathcal{X} \in \mathbb{R}^{n\times n \times d}$ as $\mathcal{X}^{(t+1)} = L_{\Theta} \big ( \mathcal{X}^{(t)} \big )$ by applying a specific transformation to on- ($x_{i}^{i}$) and off-diagonal terms ($x_{j}^{i}, i \neq j$):
\begin{align}
    x^{i,(t+1)}_{i} \!\!\!&=\! \upsilon_{\theta_1} \big ( x^{i,(t)}_{i}\!, \agg_{j} x^{j,(t)}_{j}\!, \agg_{j \neq i} x^{i,(t)}_{j}\!, \agg_{h \neq i} x^{h,(t)}_{i}\!, \agg_{h \neq j}
    x^{h,(t)}_{j} \big ) \label{eq:2IGN} \\
    x^{k,(t+1)}_{i} \!\!\!&=\! \upsilon_{\theta_2} \big ( x^{k,(t)}_{i}\!, x^{i,(t)}_{k}\!, \agg_{h \neq j} x^{h,(t)}_{j}\!, \agg_{h \neq i} x^{h,(t)}_{i}\!, \agg_{j \neq k} x^{k,(t)}_{j}\!, \agg_{j \neq i} x^{i,(t)}_{j}\!, \agg_{h \neq k} x^{h,(t)}_{k}\!, x^{k,(t)}_{k}\!, x^{i,(t)}_{i}\!, \agg_j x^{j,(t)}_{j} \big ) \nonumber
\end{align}
\noindent Here, $\agg$ indicates a permutation invariant aggregation function, $\upsilon_{\theta_1}, \upsilon_{\theta_2}$ apply a specific $d \times d'$ linear transformation to each input term and sum the outputs including bias terms.

\textbf{\ReIGN{2} layer space.} As \IGN{2} layers are linear, the authors advocate setting $\agg \equiv \sum$, performing pooling as \emph{global} summation. Here, we extend this formulation to additionally include different \emph{local} aggregation schemes. In this new extended formalism, entry $x^{k}_{i}$ represents node $i$ in subgraph $k$; accordingly, each aggregation in Equation~\ref{eq:2IGN} can be also performed locally, i.e. extending only over $i$'s neighbours, as prescribed by the connectivity of subgraph $k$ or of the original input graph. As an example, when updating entry $x^{k,(t)}_{i}$, term $\agg_{j \neq k} x^{k,(t)}_{j}$ is expanded as $\big ( \agg_{j \neq k} x^{k,(t)}_{j}, \agg_{j \sim_k i} x^{k,(t)}_{j}, \agg_{j \sim i} x^{k,(t)}_{j} \big )$, with $\sim_k$ denoting adjacency in subgraph $k$, and $\sim$ that in the original graph connectivity. Each term in the expansion is associated with a specific learnable linear transformation. We report a full list of pooling operations in \Cref{app:space}, Table~\ref{tab:expansion}. These local pooling operations allow to readily recover sparse message passing, which constitutes the main computational primitive of all popular (Subgraph) GNNs. Other characteristic Subgraph GNN operations are also recovered by this formalism: for example, $\agg_{h \neq i} x^{h,(t)}_{i}$ operates global pooling of node $i$'s representations across subgraphs, as previously introduced in~\citet{bevilacqua2022equivariant, zhao2022from}. We also note that additional, novel, operations are supported, e.g. the transpose $x^{i,(t)}_{k}$. We generally refer to this framework as \ReIGN{2}{} (``Rethought \IGN{2}'').

\textbf{\ReIGN{2} architectures.} \ReIGN{2} induces (linear) layers in the same form of Equation~\ref{eq:2IGN}, but where $\square$ terms are expanded to both local and global operations, as explained. These layers can operate on any bag generated by a node-based selection policy $\bar{\pi}$, and can be combined together in \ReIGN{2} stacks of the form $\mathcal{S}_\mathcal{R} = L^{(T)} \circ \sigma \circ L^{(T-1)} \circ \sigma \circ \mathellipsis \circ \sigma \circ L^{(1)}$, where $\sigma$'s are pointwise nonlinearities and $L$'s are \ReIGN{2} layers. This allows us to define \ReIGN{2} models as Subgraph GNNs in the form of Equation~\ref{eq:subgraph_gnn}, where $\mathcal{S}$ is a \ReIGN{2} layer stacking and $\pi$ is node-based: $\mathcal{R}_{\bar{\pi}} = \mu \circ \rho \circ \mathcal{S}_\mathcal{R} \circ \bar{\pi}$.

More generally, \ReIGN{2} induces a `layer space' for node-based Subgraph GNNs: the expanded terms in its update equations represent a pool of atomic operations that can be selected and combined to define new equivariant layers. Compared to that of \IGNs{3}, this space is of tractable size, yet \revision{it recovers previously proposed Subgraph GNNs and allows to define novel interesting variants}. 

\begin{figure}[t]
    \centering%\vspace{-1mm}
    \includegraphics[width=\linewidth]{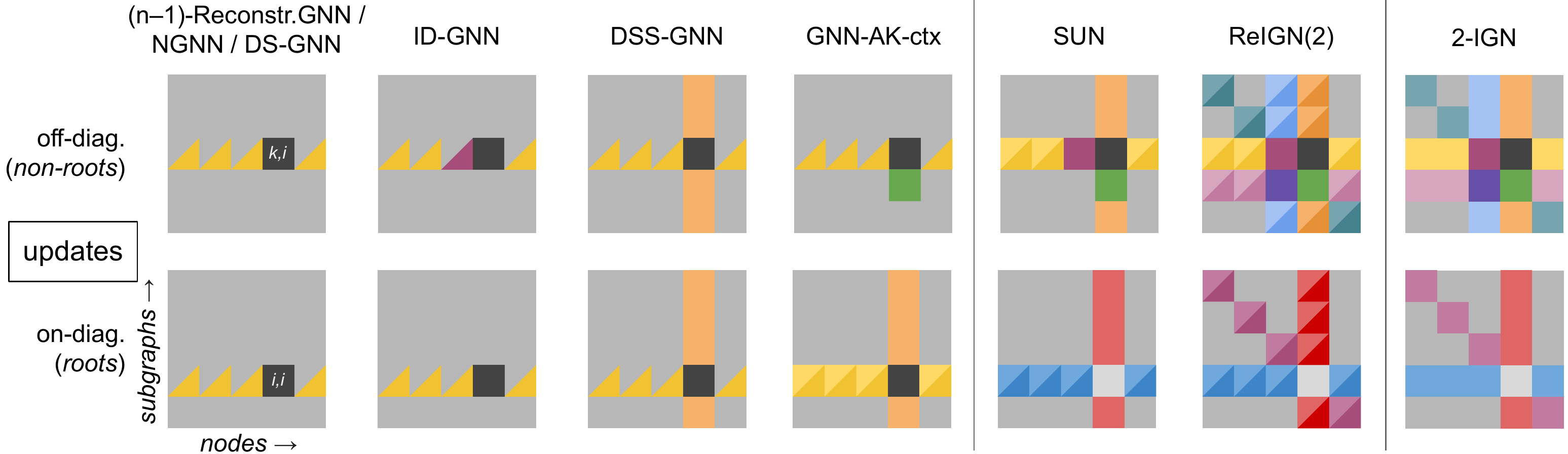}%\vspace{-2mm}
    \caption{\revision{C}omparison of aggregation and update rules in Subgraph GNNs, illustrated on an $n \times n$ matrix \revision{($n$ subgraphs with $n$ nodes). Top row: off-diagonal updates; bottom row: diagonal (root node) updates}. Each colour represents a different parameter. Full squares\revision{: global sum pooling; triangles: local pooling; two triangles: both local and global pooling}. See \Cref{app:figure} for more details.}%\vspace{-5mm}
    \label{fig:ops}
\end{figure}

\textbf{Recovering previous Subgraph GNNs.} The following result states that the \revision{\ReIGN{2}} generalises all known \revision{subgraph methods in $\Upsilon$}, as their layers are captured by a \ReIGN{2} stacking.
\begin{theorem}[\camera{\ReIGN{2} implements node-based Subgraph GNNs}]\label{thm:ReIGN_implements_SubgraphNetworks}
    Let $\mathcal{N}$ be a model in family $\Upsilon$ equipped with \citet{morris2019weisfeiler} message-passing base-encoders. For any instance $\mathcal{N}_\Theta$, there exists \ReIGN{2} instance $\mathcal{R}_\Omega$ such that $\mathcal{R}_\Omega \cong \mathcal{N}_\Theta$.
\end{theorem}

This shows that known methods are generalised without resorting to the $\mathcal{O}(n^3)$ computational complexity of \IGNs{3}. \camera{\Cref{fig:ops} illustrates the aggregation and sharing rules used by previous Subgraph GNNs to update root and non-root nodes, and compare them with those of \ReIGN{2} and \IGNs{2}. We visualise these on the subgraph-node sub-tensor gathering node representations across subgraphs; here, root nodes occupy the main diagonal, non-root nodes all the remaining off-diagonal entries. As for to the \IGN{2} Equations~\ref{eq:2IGN}, the elements in these two partitions may be updated differently, so we depict them separately in, respectively, the bottom and top rows. In each depiction we colour elements depending on the set of weights parameterising their contribution in the update process, with two main specifications: (i) Elements sharing the same colour are pooled together; (ii) Triangles indicate such pooling is performed locally based on the subgraph connectivity at hand (two triangles indicate both local and global pooling ops are performed). E.g., note how DS-GNN equivalently updates the representations of root and non-root nodes via the same (local) message-passing layer (triangles, yellow, leftmost picture). By illustrating how \ReIGN{2} generalises previous node-based methods, this figure is to be interpreted as visual support for the Proof of \Cref{thm:ReIGN_implements_SubgraphNetworks} (see \Cref{app:space}). Additional details and discussions on \Cref{fig:ops} are found in~\Cref{app:figure}.}

Notably, as methods in $\Upsilon$ have been shown to be strictly stronger than \WL{2}~\citep{bevilacqua2022equivariant, cotta2021reconstruction, zhao2022from, zhang2021nested, you2021identity}, \Cref{thm:ReIGN_implements_SubgraphNetworks} implies the same lower bound for \ReIGN{2}. Nevertheless, when employing policies in $\Pi$ and \IGN{3}-computable invariant pooling functions $\rho$ (as those used by models in $\Upsilon$), \ReIGN{2}s are upper-bounded by \IGNs{3}:
\begin{proposition}[\camera{\IGNs{3} implement \ReIGN{2}}]\label{prop:3IGN_implements_ReIGN}
   For any pair of non-isomorphic graphs $G_1, G_2$ in family $\mathcal{G}$, if there exist policy $\bar{\pi} \in \Pi$, parameters $\Theta$ and \IGN{3}-computable invariant pooling function $\rho$ such that the \ReIGN{2} instance $\mathcal{R}_{\rho,\Theta,\bar{\pi}}$ distinguishes $G_1$, $G_2$, then there exist weights $\Omega$ for a \IGN{3} instance $\mathcal{M}_{\Omega}$ such that $G_1, G_2$ are distinguished by $\mathcal{M}_{\Omega}$ as well.
\end{proposition}
This proposition entails an upper-bound on the expressive power of \ReIGN{2}.
\begin{corollary}[\camera{\WL{3} upper-bounds \ReIGN{2}}]\label{cor:ReIGN_upperbound}
    The expressive power of a \ReIGN{2} model with policy $\pi \in \Pi$ and \IGN{3}-computable invariant pooling function $\rho$ is upper-bounded by \WL{3}.
\end{corollary}

We note that there may be layers equivariant to $S_n$ over $\mathbb{R}^{n^2}$ not captured by \ReIGN{2}. Yet, previously proposed Subgraph GNN layers do not exhaust the \ReIGN{2} design space, which remains largely unexplored. One, amongst possible novel constructions, is introduced next.\label{sec:reign}

% https://open.spotify.com/playlist/1nAh5gWnxeKdPGtwV0RHZ5?si=0c8d1f51dc734198
\subsection{A unifying architecture: Subgraph Union Networks} 

We now show how the \ReIGN{2} layer space can guide the design of novel, expressive, Subgraph GNNs. Our present endeavour is to conceive a computationally tractable architecture subsuming known node-based models: in virtue of this latter desideratum, we will dub this architecture ``Subgraph Union Network'' (\SUN{}). To design the base equivariant layer for \SUN{}, we select and combine specific aggregation terms suggested by the \ReIGN{2} framework:
\begin{align}
    x^{i,(t+1)}_{i} &= \sigma \Big( \upsilon_{\theta_1} \big( x^{i,(t)}_{i}, \sum_{j \sim_i i} x^{i,(t)}_{j}, \sum_{j} x^{i,(t)}_{j}, \sum_{h} x^{h,(t)}_{i}, \sum_{j \sim i} \sum_{h} x^{h,(t)}_{j} \big) \Big) \label{eq:sun_layer_on} \\
    x^{k,(t+1)}_{i} &= \sigma \Big( \upsilon_{\theta_2} \big( x^{k,(t)}_{i}, \sum_{j \sim_k i} x^{k,(t)}_{j}, x^{i,(t)}_{i}, x^{k,(t)}_{k}, \sum_{j} x^{k,(t)}_{j}, \sum_{h} x^{h,(t)}_{i}, \sum_{j \sim i} \sum_{h} x^{h,(t)}_{j} \big) \Big) \label{eq:sun_layer_off}
\end{align}
\noindent where $\upsilon$'s sum their inputs after applying a specific linear transformations to each term. One of the novel features of \SUN{} is that roots are transformed by a \emph{different set of parameters} ($\theta_1$) than the other nodes \footnote{As a result, the architecture can mark root nodes, for example.} ($\theta_2$, see Figure \ref{def:subgraph_gnns}). In practice, the first and last two terms in each one of~\Cref{eq:sun_layer_on,eq:sun_layer_off} can be processed by maximally expressive MPNNs \citep{morris2019weisfeiler,xu2019how}, the remaining terms by MLPs. We test these variants in our experiments, with their formulations in~\Cref{app:experiments}. \SUN{} remains an instantiation of the \ReIGN{2} framework:
\begin{proposition}[A \ReIGN{2} stacking implements \SUN{} layers]\label{prop:SUN_is_in_ReIGN}
    For any \SUN{} layer $L$ defined according to Equations~\ref{eq:sun_layer_on} and~\ref{eq:sun_layer_off}, there exists a \ReIGN{2} layer stacking $\mathcal{S}_L$, such that $\mathcal{S}_L \cong L$.
\end{proposition}

Finally, we show that a stacking of \SUN{} layers can implement any layer of known node-based Subgraph Networks, making this model a principled generalisation thereof.
\begin{proposition}[A \SUN{} stacking implements known Subgraph GNN layers]\label{prop:SUN_implements_SubgraphNetworks}
    Let $\mathcal{N}$ be a model in family $\Upsilon$ employing \citet{morris2019weisfeiler} as a message-passing base-encoder. Then, for any layer $L$ in $\mathcal{N}$, there exists a stacking of \SUN{} layers $\mathcal{S}_L$ such that $\mathcal{S}_L \cong L$.
\end{proposition}

\textbf{Beyond \SUN.} As it can be seen in Figure~\ref{fig:ops}, \SUN{} does not use all possible operations in the \ReIGN{2} framework. Notably, two interesting operations that are not a part of \SUN{} are: (i) The `transpose': $x^{k}_{i} = \upsilon_{\theta}(x^{i}_{k})$, which shares information between the $i$-th node in the $k$-th subgraph and the $k$-th node in the $i$-th subgraph; (ii) Local vertical pooling $x^{k}_{i}=\upsilon_{\theta}(\sum_{h \sim i} x^{h}_{i})$. The exploration of these and other operations is left to future work.\label{sec:sun}

\section{Experiments}\label{sec:exp}

\begin{table}[t]
\scriptsize
\caption{Test mean MAE on the Counting Substructures and ZINC-12k datasets. All Subgraph GNNs employ a GIN base-encoder. $^\dagger$This version of \textsc{GNN-AK+} does not follow the standard evaluation procedure.}
\label{tab:count-zinc}
\begin{minipage}{0.63\textwidth}
    % \scriptsize
    \begin{tabular}{l cccc  %ccc
        }
        \toprule
         \multirow{2}{*}{Method} &  \multicolumn{4}{c}{Counting Substructures (MAE $\downarrow$)}
         \\
         \cmidrule(l{2pt}r{2pt}){2-5}
                                      & Triangle & Tailed Tri. & Star & 4-Cycle
        \\ 
        \midrule  
         \textsc{GCN}~\citep{kipf2016semi}     & 0.4186 & 0.3248 & 0.1798 & 0.2822 
         \\
         \textsc{GIN}~\citep{xu2019how}       & 0.3569 & 0.2373 & 0.0224 & 0.2185 
         \\
         \textsc{PNA}~\citep{corso2020pna}        & 0.3532 & 0.2648 & 0.1278 & 0.2430 
         \\
         \textsc{PPGN}~\citep{maron2019provably}        & 0.0089 &0.0096  &0.0148  &{\bf 0.0090}  
         \\
         \midrule
        \textsc{GNN-AK}~\citep{zhao2022from} &   0.0934 & 0.0751 & 0.0168 & 0.0726 
        \\
        \textsc{GNN-AK-ctx}~\citep{zhao2022from} &  0.0885 &  0.0696 & 0.0162 &  0.0668 
        \\
        \textsc{GNN-AK$+$}~\citep{zhao2022from}  & 0.0123 & 0.0112 & 0.0150 & 0.0126 
        \\
        \midrule
        {\bf SUN (EGO)} &  0.0092 & 0.0105 & 0.0064 & 0.0140 
        \\
        {\bf SUN (EGO+)} & {\bf 0.0079} & {\bf0.0080} & {\bf 0.0064}  & 0.0105 
        \\
        \bottomrule         
    \end{tabular}
\end{minipage}%
\begin{minipage}{0.37\textwidth}
    \begin{tabular}{l l}
    \toprule
        Method & \textsc{ZINC (MAE $\downarrow$)} \\
        \midrule  
        \textsc{GCN}~\citep{kipf2016semi}        & 0.321 $\pm$ 0.009\\
        \textsc{GIN}~\citep{xu2019how}        & 0.163 $\pm$ 0.004\\
        \textsc{PNA}~\citep{corso2020pna} & 0.133 $\pm$ 0.011 \\
        \textsc{GSN}~\citep{bouritsas2022improving}        & 0.101 $\pm$ 0.010\\
        \textsc{CIN}~\citep{bodnar2021weisfeilerB} & {\bf 0.079} $\pm$ 0.006 \\
        \midrule
        \textsc{NGNN}~\citep{zhang2021nested} & 0.111 $\pm$ 0.003 \\
        \textsc{DS-GNN (EGO)}~\citep{bevilacqua2022equivariant} &  0.115 $\pm$ 0.004 \\
        \textsc{DS-GNN (EGO+)}~\citep{bevilacqua2022equivariant} & 0.105 $\pm$ 0.003 \\
        \textsc{DSS-GNN (EGO)}~\citep{bevilacqua2022equivariant} & 0.099 $\pm$ 0.003 \\
        \textsc{DSS-GNN (EGO+)}~\citep{bevilacqua2022equivariant} & 0.097 $\pm$ 0.006 \\
        \textsc{GNN-AK}~\citep{zhao2022from} & 0.105 $\pm$ 0.010 \\
        \textsc{GNN-AK-ctx}~\citep{zhao2022from} & 0.093 $\pm$ 0.002\\
        \textsc{GNN-AK+}~\citep{zhao2022from}$^\dagger$ & 0.086 $\pm$ ???\\
        \textsc{GNN-AK+}~\citep{zhao2022from} & 0.091 $\pm$ 0.011 \\
        \midrule
        \textsc{\bf SUN (EGO)}  & 0.083 $\pm$ 0.003\\
         \textsc{\bf SUN (EGO+)}  & 0.084 $\pm$ 0.002\\
        \bottomrule
    \end{tabular}
\end{minipage}
\end{table}

We experimentally validate the effectiveness of one \ReIGN{2} instantiation, comparing \SUN{} to previously proposed Subgraph GNNs\footnote{For GNN-AK variants~\citep{zhao2022from}, we run the code provided by the authors, for which the `context' and `subgraph' embeddings sum only over ego-network nodes.}. We seek to verify whether its theoretical representational power practically enables superior accuracy in expressiveness tasks and real-world benchmarks. Concurrently, we pay attention to the \emph{generalisation ability} of models in comparison. \SUN{} layers are less constrained in their weight sharing pattern, resulting in a more complex model. As this is traditionally associated with inferior generalisation abilities in low data regimes, we deem it important to additionally assess this aspect. %Here, we report our main results and refer readers to \Cref{app:experiments} for additional details and results.
Our code is also available.\footnote{\small \url{https://github.com/beabevi/SUN}}

\textbf{Synthetic.} Counting substructures and regressing graph topological features are notoriously hard tasks for GNNs~\citep{zhengdao2020can,dwivedi2021graph,corso2020pna}. We test the representational ability of \SUN{} on common benchmarks of this kind~\citep{zhengdao2020can, corso2020pna}. \Cref{tab:count-zinc} reports results on the substructure counting suite, on which \SUN{} attains state-of-the-art results in $3$ out of $4$ tasks. Additional results on the regression of global, structural properties are reported in Appendix~\ref{app:experiments}.

\textbf{Real-world.} On the molecular ZINC-12k benchmark (constrained solubility regression)~\citep{ZINCdataset, gomez2018auto,dwivedi2020benchmarking}, \SUN{} exhibits best performance amongst all domain-agnostic GNNs under the $500$k parameter budget, including other Subgraph GNNs (see~\Cref{tab:count-zinc}). A similar trend is observed on the large-scale Molhiv dataset from the OGB~\citep{hu2020open} (inhibition of HIV replication). Results are in~\Cref{tab:ogbg-hiv-baselines}. Remarkably, on both datasets, \SUN{} either outperforms or approaches HIMP~\citep{Fey/etal/2020}, GSN~\citep{bouritsas2022improving} and CIN~\citep{bodnar2021weisfeilerB}, GNNs which explicitly model rings. We experiment on smaller-scale TUDatasets~\citep{morris2020tudataset} in Appendix~\ref{app:experiments}, where we also compare selection policies.

\begin{wraptable}[14]{r}{0.355\textwidth}
    %\vspace{-5pt}
    %   \vspace{-10pt}
  \caption{Test results for OGB dataset. GIN base-encoder for each Subgraph GNN.
  }
  \label{tab:ogbg-hiv-baselines}
%   \vspace{-5pt}
%  \tiny
%  \renewcommand{\arraystretch}{1.1}
      \tiny
  \begin{tabular}{l|c}
      \toprule
      & \textsc{ogbg-molhiv} \\
        \textbf{Method} & \textbf{ROC-AUC (\%)} \\
      \midrule
        \textsc{GCN}~\citep{kipf2016semi} & 76.06$\pm$0.97  \\
        \textsc{GIN}~\citep{xu2019how} & 75.58$\pm$1.40 \\
        \textsc{PNA}~\citep{corso2020pna} & 79.05$\pm$1.32 \\
        \textsc{DGN}~\citep{beaini2021directional} & 79.70$\pm$0.97 \\
        \textsc{HIMP}~\citep{Fey/etal/2020} & 78.80$\pm$0.82 \\
        \textsc{GSN}~\citep{bouritsas2022improving} & 80.39$\pm$0.90 \\
        \textsc{CIN}~\citep{bodnar2021weisfeilerB} & {\bf80.94}$\pm$0.57\\
        \midrule
        \textsc{Reconstr.GNN}~\citep{cotta2021reconstruction} & 76.32$\pm$1.40 \\
        \textsc{DS-GNN (EGO+)}~\citep{bevilacqua2022equivariant} & 77.40$\pm$2.19  \\
        \textsc{DSS-GNN (EGO+)}~\citep{bevilacqua2022equivariant} & 76.78$\pm$1.66  \\
        \textsc{GNN-AK+}~\citep{zhao2022from} & 79.61$\pm$1.19 \\
        \midrule
        \textsc{\bf SUN (EGO+)} & 80.03$\pm$0.55\\
      \bottomrule
    %\vspace{-15pt}
    \end{tabular}
\end{wraptable}

\begin{figure}[t]
    \centering
    \begin{subfigure}{0.32\textwidth}
        \centering
        \includegraphics[width=\textwidth]{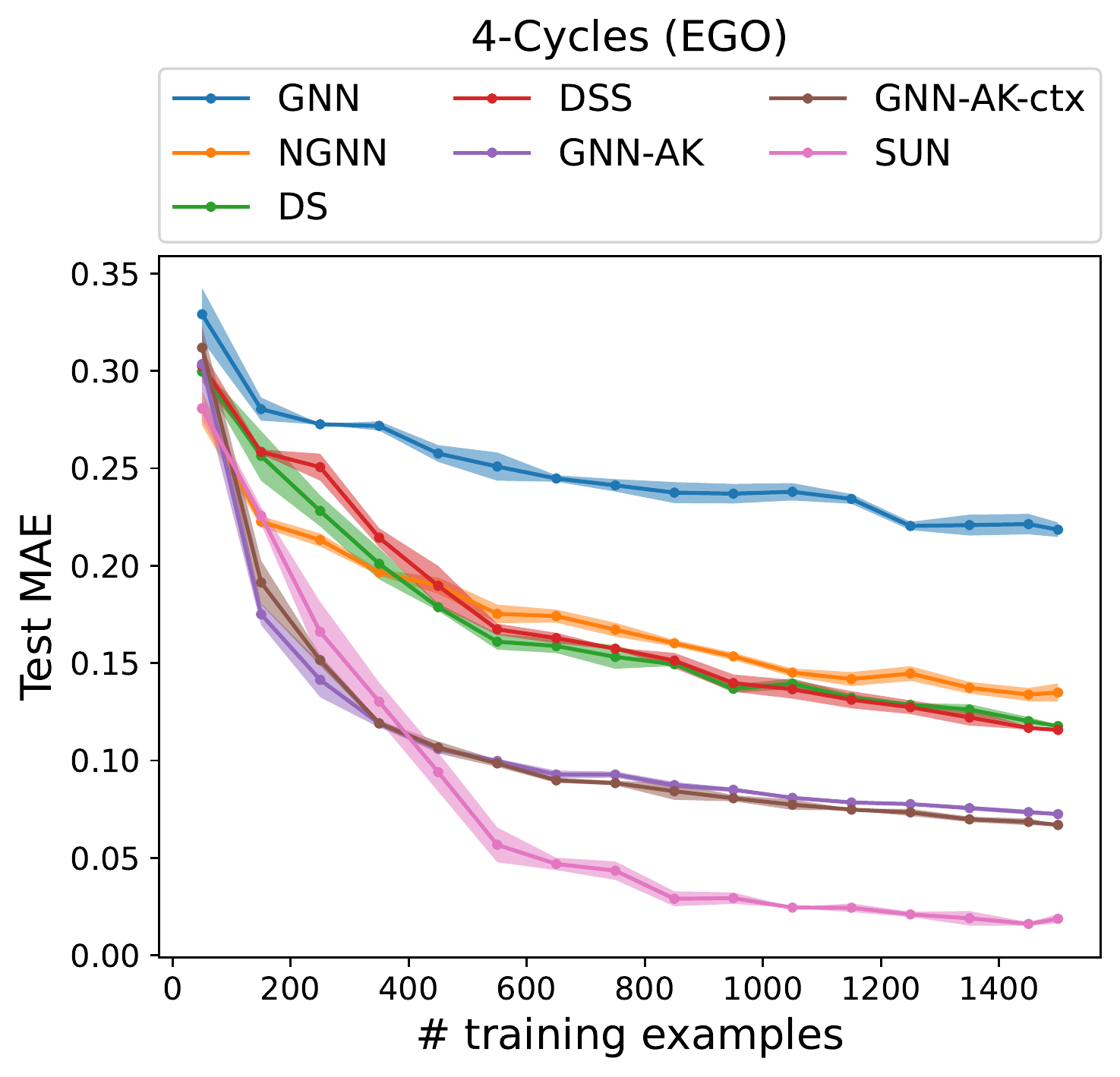} % first figure itself
        \caption{}
        \label{fig:4cycles-ego}
    \end{subfigure}
    \begin{subfigure}{0.32\textwidth}
        \centering
        \includegraphics[width=\textwidth]{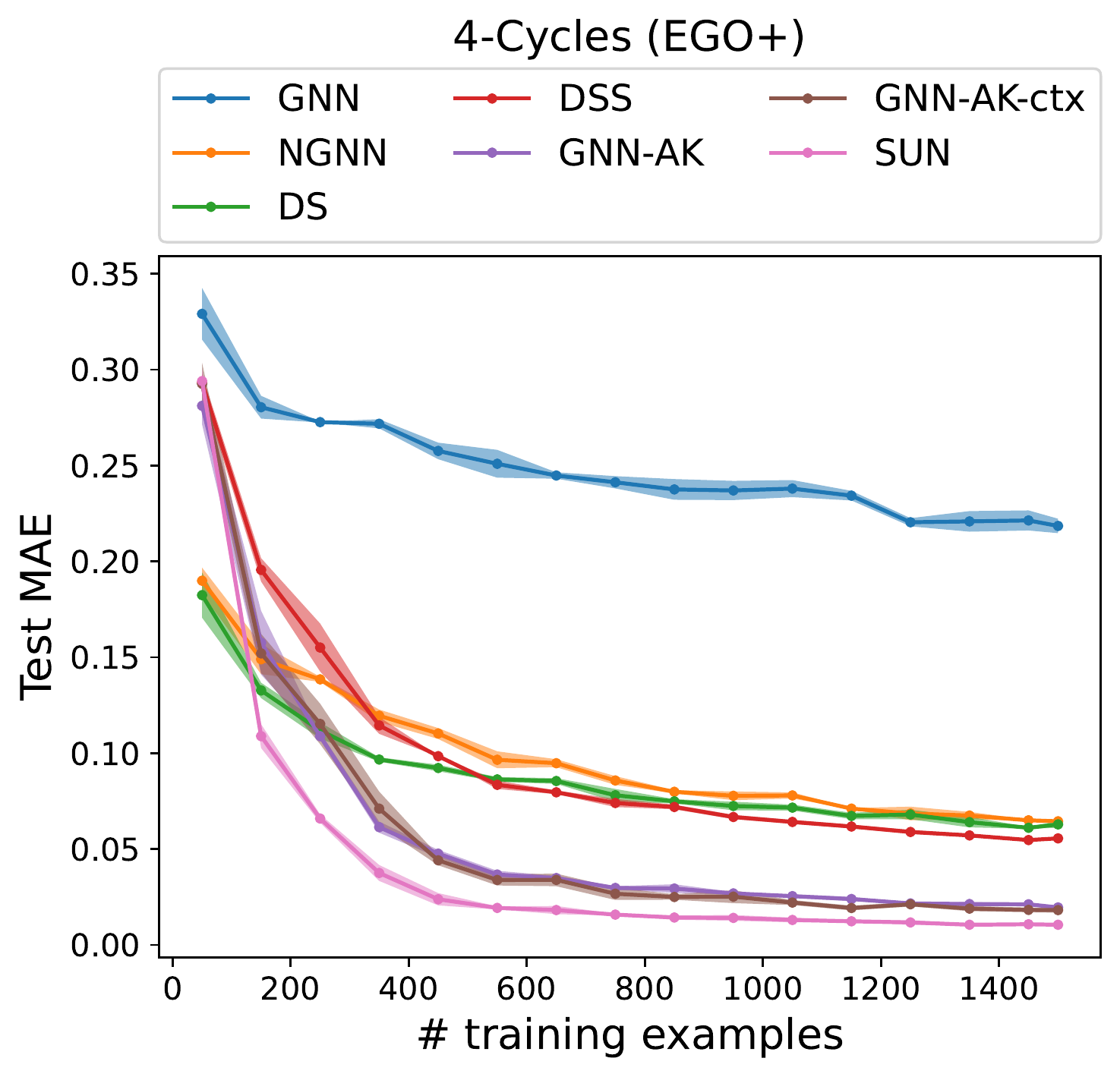} % second figure itself
        \caption{}
        \label{fig:4cycles-ego-plus}
    \end{subfigure}
    \begin{subfigure}{0.32\textwidth}
        \centering
        \includegraphics[width=\textwidth]{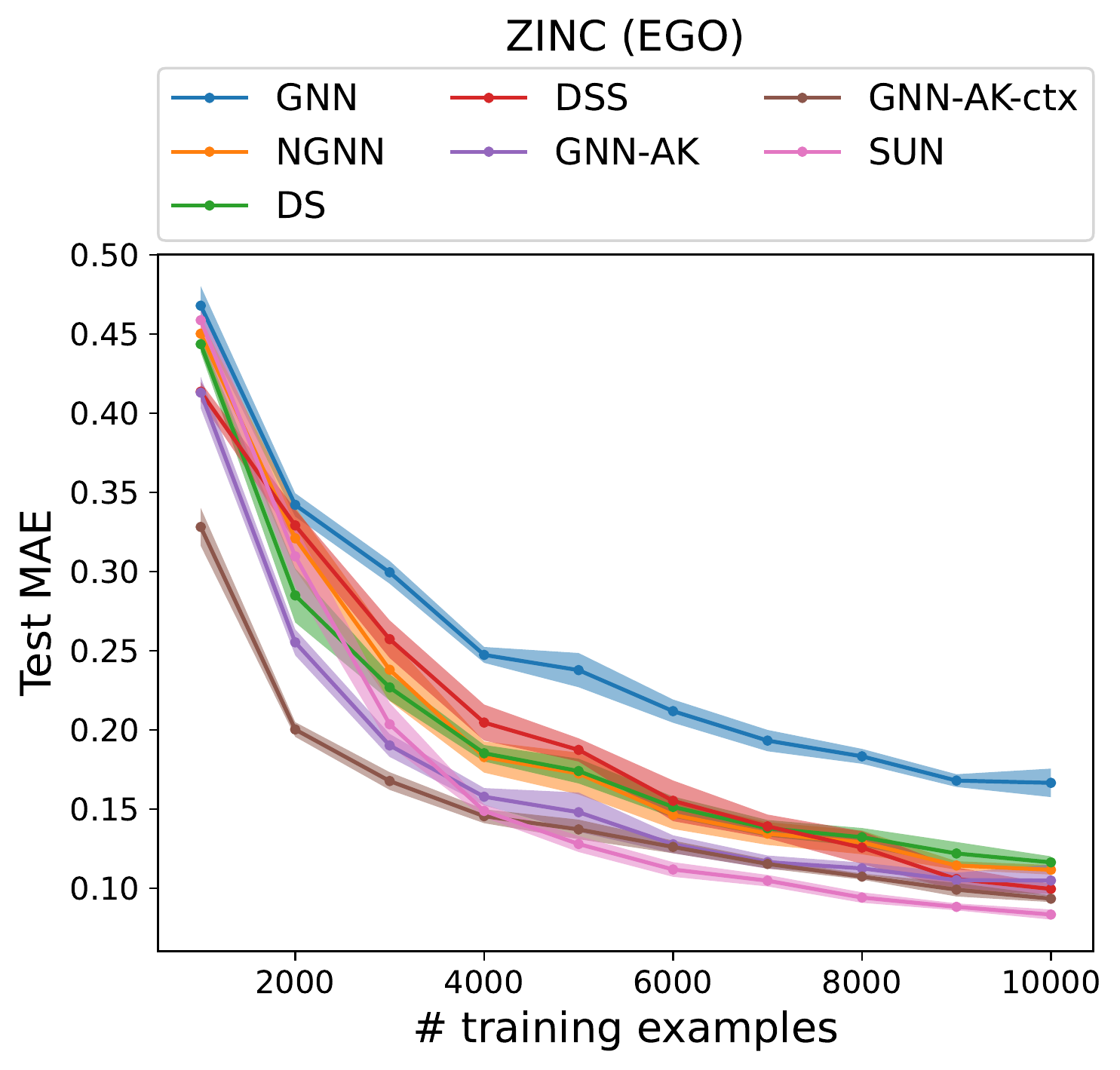} % third figure itself
        \caption{}
        \label{fig:zinc}
    \end{subfigure}
    \caption{Generalisation capabilities of Subgraph GNNs in the counting prediction task (\Cref{fig:4cycles-ego,fig:4cycles-ego-plus}) and in the ZINC-12k dataset (\Cref{fig:zinc}).}
\end{figure}

\textbf{Generalisation from limited data.} In this set of experiments we compare the test performance of Subgraph GNNs when trained on increasing fractions of the available training data. Each architecture is selected by tuning the hyperparameters with the entire training and validation sets. We run this experiment on the $4$-cycle counting task and the real-world ZINC-12k. We illustrate results in~\Cref{fig:4cycles-ego,fig:4cycles-ego-plus,fig:zinc}. Except for a short initial phase in the EGO policy, \SUN{} generalises better than other Subgraph GNNs on cycle-counting. On ZINC-12k, \SUN{} outruns DSS-, DS-GNN and GNN-AK variants from, respectively, $20, 30$ and $40\%$ of the samples.
These results demonstrate that \SUN{}'s expressiveness is not at the expense of sample efficiency, suggesting that its modelled symmetries guarantee strong representational power while retaining important inductive biases for learning on graphs.

\section{Conclusions}
Our work unifies, extends, and analyses the emerging class of Subgraph GNNs. Notably, we demonstrated that the expressive power of these methods is bounded by \WL{3}. 
%We note that for many applications such as molecule property prediction, \WL{3} might be expressive enough: our contribution should not be understood as a negative result, rather as an encouragement towards a more systematic study of models whose expressivity lies between \texttt{2-} and \WL{3}. 
Towards a systematic study of models whose expressivity lies between \texttt{1-} and \WL{3}, we proposed a new family of layers for the class of Subgraph GNNs and, unlike most previous works on the expressive power of GNNs, we also investigated the generalisation abilities of these models, for which \SUN{} shows considerable improvement. Appendix \ref{app:future} lists several directions for future work, including an extension of our work to higher-order node-based policies.

\textbf{Societal impact.} We do not envision any negative, immediate societal impact originating from our theoretical results, which represent most of our contribution. Experimentally, our model has shown promising results on molecular property prediction tasks and strong generalisation ability in low-data regimes. This leads us to believe our work may contribute to positively impactful pharmaceutical research, such as drug discovery~\citep{drug_discovery, AlRaPaPa2017}.\label{sec:concl}

\begin{ack}
The authors are grateful to Joshua Southern, Davide Eynard, Maria Gorinova, Guadalupe Gonzalez, Katarzyna Janocha for valuable feedback on early versions of the manuscript. They would like to thank Bruno Ribeiro and Or Litany for helpful discussions, Giorgos Bouritsas for constructive conversations about the generalisation experiments and, in particular, Marco Ciccone for the precious exchange on sharpness-aware optimisation and Neapolitan pizza.
MB is supported in part by ERC Consolidator grant no 724228 (LEMAN). No competing interests are declared.

% Use unnumbered first level headings for the acknowledgments. All acknowledgments
% More information about this disclosure can be found at: \url{https://neurips.cc/Conferences/2022/PaperInformation/FundingDisclosure}.

% Do {\bf not} include this section in the anonymized submission, only in the final paper. You can use the \texttt{ack} environment provided in the style file to autmoatically hide this section in the anonymized submission.
\end{ack}
% \newpage

\bibliographystyle{plainnat}
\bibliography{references,app_references}

\begin{thebibliography}{61}
\providecommand{\natexlab}[1]{#1}
\providecommand{\url}[1]{\texttt{#1}}
\expandafter\ifx\csname urlstyle\endcsname\relax
  \providecommand{\doi}[1]{doi: #1}\else
  \providecommand{\doi}{doi: \begingroup \urlstyle{rm}\Url}\fi

\bibitem[Abboud et~al.(2020)Abboud, Ceylan, Grohe, and
  Lukasiewicz]{abboud2020surprising}
Ralph Abboud, \.{I}smail~\.{I}lkan Ceylan, Martin Grohe, and Thomas
  Lukasiewicz.
\newblock The surprising power of graph neural networks with random node
  initialization.
\newblock In \emph{Proceedings of the Thirtieth International Joint Conference
  on Artificial Intelligence ({IJCAI})}, 2020.

\bibitem[Albooyeh et~al.(2019)Albooyeh, Bertolini, and
  Ravanbakhsh]{Albooyeh2019incidence}
Marjan Albooyeh, Daniele Bertolini, and Siamak Ravanbakhsh.
\newblock Incidence networks for geometric deep learning.
\newblock \emph{arXiv preprint arXiv:1905.11460}, 2019.

\bibitem[Altae-Tran et~al.(2017)Altae-Tran, Ramsundar, Pappu, and
  Pande]{AlRaPaPa2017}
Han Altae-Tran, Bharath Ramsundar, Aneesh~S Pappu, and Vijay Pande.
\newblock Low data drug discovery with one-shot learning.
\newblock \emph{ACS Central Science}, 3\penalty0 (4):\penalty0 283--293, 2017.

\bibitem[Atwood and Towsley(2016)]{DCNN_2016}
James Atwood and Don Towsley.
\newblock Diffusion-convolutional neural networks.
\newblock In \emph{Advances in Neural Information Processing Systems},
  volume~29, 2016.

\bibitem[Azizian and Lelarge(2021)]{azizian2020characterizing}
Wa{\"\i}ss Azizian and Marc Lelarge.
\newblock Expressive power of invariant and equivariant graph neural networks.
\newblock In \emph{International Conference on Learning Representations}, 2021.

\bibitem[Beaini et~al.(2021)Beaini, Passaro, Létourneau, Hamilton, Corso, and
  Li\`{o}]{beaini2021directional}
Dominique Beaini, Saro Passaro, Vincent Létourneau, William~L. Hamilton,
  Gabriele Corso, and Pietro Li\`{o}.
\newblock Directional graph networks.
\newblock In \emph{International Conference on Machine Learning}, 2021.

\bibitem[Bevilacqua et~al.(2022)Bevilacqua, Frasca, Lim, Srinivasan, Cai,
  Balamurugan, Bronstein, and Maron]{bevilacqua2022equivariant}
Beatrice Bevilacqua, Fabrizio Frasca, Derek Lim, Balasubramaniam Srinivasan,
  Chen Cai, Gopinath Balamurugan, Michael~M Bronstein, and Haggai Maron.
\newblock Equivariant subgraph aggregation networks.
\newblock In \emph{International Conference on Learning Representations}, 2022.

\bibitem[Biewald(2020)]{wandb}
Lukas Biewald.
\newblock Experiment tracking with weights and biases, 2020.
\newblock Software available from wandb.com.

\bibitem[Bodnar et~al.(2021{\natexlab{a}})Bodnar, Frasca, Otter, Wang, Li\`{o},
  Mont\'{u}far, and Bronstein]{bodnar2021weisfeilerB}
Cristian Bodnar, Fabrizio Frasca, Nina Otter, Yuguang Wang, Pietro Li\`{o},
  Guido~F Mont\'{u}far, and Michael Bronstein.
\newblock Weisfeiler and lehman go cellular: Cw networks.
\newblock In \emph{Advances in Neural Information Processing Systems},
  volume~34, 2021{\natexlab{a}}.

\bibitem[Bodnar et~al.(2021{\natexlab{b}})Bodnar, Frasca, Wang, Otter,
  Mont\'{u}far, Li\`{o}, and Bronstein]{bodnar2021weisfeilerA}
Cristian Bodnar, Fabrizio Frasca, Yuguang Wang, Nina Otter, Guido~F
  Mont\'{u}far, Pietro Li\`{o}, and Michael Bronstein.
\newblock Weisfeiler and lehman go topological: Message passing simplicial
  networks.
\newblock In \emph{International Conference on Machine Learning},
  2021{\natexlab{b}}.

\bibitem[Bouritsas et~al.(2022)Bouritsas, Frasca, Zafeiriou, and
  Bronstein]{bouritsas2022improving}
Giorgos Bouritsas, Fabrizio Frasca, Stefanos~P Zafeiriou, and Michael
  Bronstein.
\newblock Improving graph neural network expressivity via subgraph isomorphism
  counting.
\newblock \emph{IEEE Transactions on Pattern Analysis and Machine
  Intelligence}, 2022.

\bibitem[Chen et~al.(2020)Chen, Chen, Villar, and Bruna]{zhengdao2020can}
Zhengdao Chen, Lei Chen, Soledad Villar, and Joan Bruna.
\newblock Can graph neural networks count substructures?
\newblock In \emph{Advances in Neural Information Processing Systems},
  volume~33, 2020.

\bibitem[Corso et~al.(2020)Corso, Cavalleri, Beaini, Li\`{o}, and
  Veli\v{c}kovi\'{c}]{corso2020pna}
Gabriele Corso, Luca Cavalleri, Dominique Beaini, Pietro Li\`{o}, and Petar
  Veli\v{c}kovi\'{c}.
\newblock Principal neighbourhood aggregation for graph nets.
\newblock In \emph{Advances in Neural Information Processing Systems},
  volume~33, 2020.

\bibitem[Cotta et~al.(2021)Cotta, Morris, and Ribeiro]{cotta2021reconstruction}
Leonardo Cotta, Christopher Morris, and Bruno Ribeiro.
\newblock Reconstruction for powerful graph representations.
\newblock In \emph{Advances in Neural Information Processing Systems},
  volume~34, 2021.

\bibitem[de~Haan et~al.(2020)de~Haan, Cohen, and Welling]{de2020natural}
Pim de~Haan, Taco~S Cohen, and Max Welling.
\newblock Natural graph networks.
\newblock In \emph{Advances in Neural Information Processing Systems},
  volume~33, 2020.

\bibitem[Dwivedi et~al.(2020)Dwivedi, Joshi, Laurent, Bengio, and
  Bresson]{dwivedi2020benchmarking}
Vijay~Prakash Dwivedi, Chaitanya~K Joshi, Thomas Laurent, Yoshua Bengio, and
  Xavier Bresson.
\newblock Benchmarking graph neural networks.
\newblock \emph{arXiv preprint arXiv:2003.00982}, 2020.

\bibitem[Dwivedi et~al.(2022)Dwivedi, Luu, Laurent, Bengio, and
  Bresson]{dwivedi2021graph}
Vijay~Prakash Dwivedi, Anh~Tuan Luu, Thomas Laurent, Yoshua Bengio, and Xavier
  Bresson.
\newblock Graph neural networks with learnable structural and positional
  representations.
\newblock In \emph{International Conference on Learning Representations}, 2022.

\bibitem[Fey and Lenssen(2019)]{fey2019fast}
Matthias Fey and Jan~Eric Lenssen.
\newblock Fast graph representation learning with pytorch geometric.
\newblock \emph{arXiv preprint arXiv:1903.02428}, 2019.

\bibitem[Fey et~al.(2020)Fey, Yuen, and Weichert]{Fey/etal/2020}
Matthias Fey, Jan-Gin Yuen, and Frank Weichert.
\newblock Hierarchical inter-message passing for learning on molecular graphs.
\newblock In \emph{ICML Graph Representation Learning and Beyond (GRL+)
  Workhop}, 2020.

\bibitem[Gaudelet et~al.(2021)Gaudelet, Day, Jamasb, Soman, Regep, Liu, Hayter,
  Vickers, Roberts, Tang, Roblin, Blundell, Bronstein, and
  Taylor-King]{drug_discovery}
Thomas Gaudelet, Ben Day, Arian~R Jamasb, Jyothish Soman, Cristian Regep,
  Gertrude Liu, Jeremy B~R Hayter, Richard Vickers, Charles Roberts, Jian Tang,
  David Roblin, Tom~L Blundell, Michael~M Bronstein, and Jake~P Taylor-King.
\newblock {Utilizing graph machine learning within drug discovery and
  development}.
\newblock \emph{Briefings in Bioinformatics}, 05 2021.
\newblock ISSN 1477-4054.

\bibitem[Geerts(2020)]{geerts2020expressive}
Floris Geerts.
\newblock The expressive power of kth-order invariant graph networks.
\newblock \emph{arXiv preprint arXiv:2007.12035}, 2020.

\bibitem[Gómez-Bombarelli et~al.(2018)Gómez-Bombarelli, Wei, Duvenaud,
  Hernández-Lobato, Sánchez-Lengeling, Sheberla, Aguilera-Iparraguirre,
  Hirzel, Adams, and Aspuru-Guzik]{gomez2018auto}
Rafael Gómez-Bombarelli, Jennifer~N. Wei, David Duvenaud, José~Miguel
  Hernández-Lobato, Benjamín Sánchez-Lengeling, Dennis Sheberla, Jorge
  Aguilera-Iparraguirre, Timothy~D. Hirzel, Ryan~P. Adams, and Alán
  Aspuru-Guzik.
\newblock Automatic chemical design using a data-driven continuous
  representation of molecules.
\newblock \emph{ACS Central Science}, 4\penalty0 (2):\penalty0 268–276, Jan
  2018.
\newblock ISSN 2374-7951.
\newblock \doi{10.1021/acscentsci.7b00572}.

\bibitem[Hu et~al.(2020)Hu, Fey, Zitnik, Dong, Ren, Liu, Catasta, and
  Leskovec]{hu2020open}
Weihua Hu, Matthias Fey, Marinka Zitnik, Yuxiao Dong, Hongyu Ren, Bowen Liu,
  Michele Catasta, and Jure Leskovec.
\newblock Open graph benchmark: Datasets for machine learning on graphs.
\newblock In \emph{Advances in Neural Information Processing Systems},
  volume~33, 2020.

\bibitem[Hy et~al.(2019)Hy, Trivedi, Pan, Anderson, and
  Kondor]{hy2019covariant}
Truong~Son Hy, Shubhendu Trivedi, Horace Pan, Brandon~M Anderson, and Risi
  Kondor.
\newblock Covariant compositional networks for learning graphs.
\newblock \emph{Anchorage ’19: 15th International Workshop on Mining and
  Learning with Graphs}, 2019.

\bibitem[Kelly(1957)]{kelly1957a}
Paul~J. Kelly.
\newblock A congruence theorem for trees.
\newblock \emph{Pacific Journal of Mathematics}, 7\penalty0 (1):\penalty0
  961--968, 1957.

\bibitem[Keriven and Peyr{\'e}(2019)]{keriven2019universal}
Nicolas Keriven and Gabriel Peyr{\'e}.
\newblock Universal invariant and equivariant graph neural networks.
\newblock In \emph{Advances in Neural Information Processing Systems},
  volume~32, 2019.

\bibitem[Kipf and Welling(2017)]{kipf2016semi}
Thomas~N Kipf and Max Welling.
\newblock Semi-supervised classification with graph convolutional networks.
\newblock In \emph{International Conference on Learning Representations}, 2017.

\bibitem[Kreuzer et~al.(2021)Kreuzer, Beaini, Hamilton, L{\'e}tourneau, and
  Tossou]{kreuzer2021rethinking}
Devin Kreuzer, Dominique Beaini, Will Hamilton, Vincent L{\'e}tourneau, and
  Prudencio Tossou.
\newblock Rethinking graph transformers with spectral attention.
\newblock In \emph{Advances in Neural Information Processing Systems},
  volume~34, 2021.

\bibitem[Kwon et~al.(2021)Kwon, Kim, Park, and Choi]{kwon2021asam}
Jungmin Kwon, Jeongseop Kim, Hyunseo Park, and In~Kwon Choi.
\newblock Asam: Adaptive sharpness-aware minimization for scale-invariant
  learning of deep neural networks.
\newblock In \emph{International Conference on Machine Learning}, 2021.

\bibitem[Li and Leskovec(2022)]{Li2022}
Pan Li and Jure Leskovec.
\newblock The expressive power of graph neural networks.
\newblock In Lingfei Wu, Peng Cui, Jian Pei, and Liang Zhao, editors,
  \emph{Graph Neural Networks: Foundations, Frontiers, and Applications}, pages
  63--98. Springer Singapore, Singapore, 2022.

\bibitem[Lim et~al.(2022)Lim, Robinson, Zhao, Smidt, Sra, Maron, and
  Jegelka]{lim2022sign}
Derek Lim, Joshua~David Robinson, Lingxiao Zhao, Tess Smidt, Suvrit Sra, Haggai
  Maron, and Stefanie Jegelka.
\newblock Sign and basis invariant networks for spectral graph representation
  learning.
\newblock In \emph{ICLR 2022 Workshop on Geometrical and Topological
  Representation Learning}, 2022.

\bibitem[Maron et~al.(2019{\natexlab{a}})Maron, Ben-Hamu, Serviansky, and
  Lipman]{maron2019provably}
Haggai Maron, Heli Ben-Hamu, Hadar Serviansky, and Yaron Lipman.
\newblock Provably powerful graph networks.
\newblock In \emph{Advances in Neural Information Processing Systems},
  volume~32, 2019{\natexlab{a}}.

\bibitem[Maron et~al.(2019{\natexlab{b}})Maron, Ben-Hamu, Shamir, and
  Lipman]{maron2018invariant}
Haggai Maron, Heli Ben-Hamu, Nadav Shamir, and Yaron Lipman.
\newblock Invariant and equivariant graph networks.
\newblock In \emph{International Conference on Learning Representations},
  2019{\natexlab{b}}.

\bibitem[Maron et~al.(2019{\natexlab{c}})Maron, Fetaya, Segol, and
  Lipman]{maron2019universality}
Haggai Maron, Ethan Fetaya, Nimrod Segol, and Yaron Lipman.
\newblock On the universality of invariant networks.
\newblock In \emph{International Conference on Machine Learning},
  2019{\natexlab{c}}.

\bibitem[Maron et~al.(2020)Maron, Litany, Chechik, and
  Fetaya]{maron2020learning}
Haggai Maron, Or~Litany, Gal Chechik, and Ethan Fetaya.
\newblock On learning sets of symmetric elements.
\newblock In \emph{International Conference on Machine Learning}, 2020.

\bibitem[Morris et~al.(2019)Morris, Ritzert, Fey, Hamilton, Lenssen, Rattan,
  and Grohe]{morris2019weisfeiler}
Christopher Morris, Martin Ritzert, Matthias Fey, William~L Hamilton, Jan~Eric
  Lenssen, Gaurav Rattan, and Martin Grohe.
\newblock Weisfeiler and leman go neural: Higher-order graph neural networks.
\newblock In \emph{Proceedings of the AAAI conference on artificial
  intelligence}, volume~33, 2019.

\bibitem[Morris et~al.(2020{\natexlab{a}})Morris, Kriege, Bause, Kersting,
  Mutzel, and Neumann]{morris2020tudataset}
Christopher Morris, Nils~M Kriege, Franka Bause, Kristian Kersting, Petra
  Mutzel, and Marion Neumann.
\newblock {TUDataset: A collection of benchmark datasets for learning with
  graphs}.
\newblock In \emph{ICML Graph Representation Learning and Beyond (GRL+)
  Workhop}, 2020{\natexlab{a}}.

\bibitem[Morris et~al.(2020{\natexlab{b}})Morris, Rattan, and
  Mutzel]{morris2020weisfeiler}
Christopher Morris, Gaurav Rattan, and Petra Mutzel.
\newblock Weisfeiler and leman go sparse: Towards scalable higher-order graph
  embeddings.
\newblock In \emph{Advances in Neural Information Processing Systems},
  volume~33, 2020{\natexlab{b}}.

\bibitem[Morris et~al.(2021)Morris, Lipman, Maron, Rieck, Kriege, Grohe, Fey,
  and Borgwardt]{morris2021weisfeiler}
Christopher Morris, Yaron Lipman, Haggai Maron, Bastian Rieck, Nils~M Kriege,
  Martin Grohe, Matthias Fey, and Karsten Borgwardt.
\newblock Weisfeiler and leman go machine learning: The story so far.
\newblock \emph{arXiv preprint arXiv:2112.09992}, 2021.

\bibitem[Morris et~al.(2022)Morris, Rattan, Kiefer, and
  Ravanbakhsh]{morris2022speqnets}
Christopher Morris, Gaurav Rattan, Sandra Kiefer, and Siamak Ravanbakhsh.
\newblock Speqnets: Sparsity-aware permutation-equivariant graph networks.
\newblock In \emph{ICLR 2022 Workshop on Geometrical and Topological
  Representation Learning}, 2022.

\bibitem[Niepert et~al.(2021)Niepert, Minervini, and
  Franceschi]{niepert2021implicit}
Mathias Niepert, Pasquale Minervini, and Luca Franceschi.
\newblock Implicit mle: Backpropagating through discrete exponential family
  distributions.
\newblock In \emph{Advances in Neural Information Processing Systems},
  volume~34, 2021.

\bibitem[Papp and Wattenhofer(2022)]{papp2022theoretical}
P{\'a}l~Andr{\'a}s Papp and Roger Wattenhofer.
\newblock A theoretical comparison of graph neural network extensions.
\newblock \emph{arXiv preprint arXiv:2201.12884}, 2022.

\bibitem[Papp et~al.(2021)Papp, Martinkus, Faber, and
  Wattenhofer]{papp2021dropgnn}
P{\'a}l~Andr{\'a}s Papp, Karolis Martinkus, Lukas Faber, and Roger Wattenhofer.
\newblock Dropgnn: Random dropouts increase the expressiveness of graph neural
  networks.
\newblock In \emph{Advances in Neural Information Processing Systems}, 2021.

\bibitem[Paszke et~al.(2019)Paszke, Gross, Massa, Lerer, Bradbury, Chanan,
  Killeen, Lin, Gimelshein, Antiga, Desmaison, Kopf, Yang, DeVito, Raison,
  Tejani, Chilamkurthy, Steiner, Fang, Bai, and Chintala]{pytorch}
Adam Paszke, Sam Gross, Francisco Massa, Adam Lerer, James Bradbury, Gregory
  Chanan, Trevor Killeen, Zeming Lin, Natalia Gimelshein, Luca Antiga, Alban
  Desmaison, Andreas Kopf, Edward Yang, Zachary DeVito, Martin Raison, Alykhan
  Tejani, Sasank Chilamkurthy, Benoit Steiner, Lu~Fang, Junjie Bai, and Soumith
  Chintala.
\newblock Pytorch: An imperative style, high-performance deep learning library.
\newblock In \emph{Advances in Neural Information Processing Systems},
  volume~32, 2019.

\bibitem[Puny et~al.(2020)Puny, Ben-Hamu, and Lipman]{puny2020global}
Omri Puny, Heli Ben-Hamu, and Yaron Lipman.
\newblock Global attention improves graph networks generalization.
\newblock \emph{arXiv preprint arXiv:2006.07846}, 2020.

\bibitem[Qian et~al.(2022)Qian, Rattan, Geerts, Morris, and
  Niepert]{Qian2022osan}
Chendi Qian, Gaurav Rattan, Floris Geerts, Christopher Morris, and Mathias
  Niepert.
\newblock Ordered subgraph aggregation networks.
\newblock \emph{arXiv preprint}, 2022.

\bibitem[Ravanbakhsh(2020)]{ravanbakhsh2020universal}
Siamak Ravanbakhsh.
\newblock Universal equivariant multilayer perceptrons.
\newblock In \emph{International Conference on Machine Learning}, 2020.

\bibitem[Rong et~al.(2019)Rong, Huang, Xu, and Huang]{rong2019dropedge}
Yu~Rong, Wenbing Huang, Tingyang Xu, and Junzhou Huang.
\newblock Dropedge: Towards deep graph convolutional networks on node
  classification.
\newblock In \emph{International Conference on Learning Representations}, 2019.

\bibitem[Sato(2020)]{sato2020survey}
Ryoma Sato.
\newblock A survey on the expressive power of graph neural networks.
\newblock \emph{arXiv preprint arXiv:2003.04078}, 2020.

\bibitem[Sterling and Irwin(2015)]{ZINCdataset}
Teague Sterling and John~J. Irwin.
\newblock {ZINC 15} -- ligand discovery for everyone.
\newblock \emph{Journal of Chemical Information and Modeling}, 55\penalty0
  (11):\penalty0 2324--2337, 11 2015.
\newblock \doi{10.1021/acs.jcim.5b00559}.

\bibitem[Thiede et~al.(2021)Thiede, Zhou, and Kondor]{thiede2021autobahn}
Erik Thiede, Wenda Zhou, and Risi Kondor.
\newblock Autobahn: Automorphism-based graph neural nets.
\newblock In \emph{Advances in Neural Information Processing Systems},
  volume~34, 2021.

\bibitem[Ulam(1960)]{ulam1960a}
Stanislaw~M. Ulam.
\newblock \emph{A collection of mathematical problems}, volume~8.
\newblock Interscience Publishers, 1960.

\bibitem[Vignac et~al.(2020)Vignac, Loukas, and Frossard]{vignac2020building}
Cl\'{e}ment Vignac, Andreas Loukas, and Pascal Frossard.
\newblock Building powerful and equivariant graph neural networks with
  structural message-passing.
\newblock In \emph{Advances in Neural Information Processing Systems},
  volume~33, 2020.

\bibitem[Weisfeiler and Leman(1968)]{weisfeiler1968reduction}
Boris Weisfeiler and Andrei Leman.
\newblock The reduction of a graph to canonical form and the algebra which
  appears therein.
\newblock \emph{NTI, Series}, 2\penalty0 (9):\penalty0 12--16, 1968.

\bibitem[Xu et~al.(2019)Xu, Hu, Leskovec, and Jegelka]{xu2019how}
Keyulu Xu, Weihua Hu, Jure Leskovec, and Stefanie Jegelka.
\newblock How powerful are graph neural networks?
\newblock In \emph{International Conference on Learning Representations}, 2019.

\bibitem[You et~al.(2021)You, Gomes-Selman, Ying, and
  Leskovec]{you2021identity}
Jiaxuan You, Jonathan Gomes-Selman, Rex Ying, and Jure Leskovec.
\newblock Identity-aware graph neural networks.
\newblock \emph{AAAI Conference on Artificial Intelligence (AAAI)}, 2021.

\bibitem[Yun et~al.(2019)Yun, Sra, and Jadbabaie]{yun2019small}
Chulhee Yun, Suvrit Sra, and Ali Jadbabaie.
\newblock Small relu networks are powerful memorizers: a tight analysis of
  memorization capacity.
\newblock In \emph{Advances in Neural Information Processing Systems},
  volume~32, 2019.

\bibitem[Zaheer et~al.(2017)Zaheer, Kottur, Ravanbakhsh, Poczos, Salakhutdinov,
  and Smola]{zaheer2017deep}
Manzil Zaheer, Satwik Kottur, Siamak Ravanbakhsh, Barnabas Poczos, Russ~R
  Salakhutdinov, and Alexander~J Smola.
\newblock Deep sets.
\newblock In \emph{Advances in Neural Information Processing Systems},
  volume~30, 2017.

\bibitem[Zhang and Li(2021)]{zhang2021nested}
Muhan Zhang and Pan Li.
\newblock Nested graph neural networks.
\newblock In \emph{Advances in Neural Information Processing Systems},
  volume~34, 2021.

\bibitem[Zhang et~al.(2018)Zhang, Cui, Neumann, and Chen]{zhang2018end}
Muhan Zhang, Zhicheng Cui, Marion Neumann, and Yixin Chen.
\newblock An end-to-end deep learning architecture for graph classification.
\newblock \emph{Proceedings of the AAAI Conference on Artificial Intelligence},
  2018.

\bibitem[Zhao et~al.(2022)Zhao, Jin, Akoglu, and Shah]{zhao2022from}
Lingxiao Zhao, Wei Jin, Leman Akoglu, and Neil Shah.
\newblock From stars to subgraphs: Uplifting any {GNN} with local structure
  awareness.
\newblock In \emph{International Conference on Learning Representations}, 2022.

\end{thebibliography}

\newpage

%%%%%%%%%%%%%%%%%%%%%%%%%%%%%%%%%%%%%%%%%%%%%%%%%%%%%%%%%%%%
\section*{Checklist}

%%% BEGIN INSTRUCTIONS %%%
% The checklist follows the references.  Please
% read the checklist guidelines carefully for information on how to answer these
% questions.  For each question, change the default \answerTODO{} to \answerYes{},
% \answerNo{}, or \answerNA{}.  You are strongly encouraged to include a {\bf
% justification to your answer}, either by referencing the appropriate section of
% your paper or providing a brief inline description.  For example:
% \begin{itemize}
%   \item Did you include the license to the code and datasets? \answerYes{See Section.}
%   \item Did you include the license to the code and datasets? \answerNo{The code and the data are proprietary.}
%   \item Did you include the license to the code and datasets? \answerNA{}
% \end{itemize}
% Please do not modify the questions and only use the provided macros for your
% answers.  Note that the Checklist section does not count towards the page
% limit.  In your paper, please delete this instructions block and only keep the
% Checklist section heading above along with the questions/answers below.
%%% END INSTRUCTIONS %%%

\begin{enumerate}

\item For all authors...
\begin{enumerate}
  \item Do the main claims made in the abstract and introduction accurately reflect the paper's contributions and scope?
    \answerYes{}
  \item Did you describe the limitations of your work?
    \answerYes{We discussed limitations of several previous works, as well as our own model, throughout the paper as our main contribution.}
  \item Did you discuss any potential negative societal impacts of your work?
    \answerYes{See Section~\ref{sec:concl}.}
  \item Have you read the ethics review guidelines and ensured that your paper conforms to them?
    \answerYes{}
\end{enumerate}

\item If you are including theoretical results...
\begin{enumerate}
  \item Did you state the full set of assumptions of all theoretical results?
    \answerYes{}
        \item Did you include complete proofs of all theoretical results?
    \answerYes{See~\Cref{app:upperbound,app:space}.}
\end{enumerate}

\item If you ran experiments...
\begin{enumerate}
  \item Did you include the code, data, and instructions needed to reproduce the main experimental results (either in the supplemental material or as a URL)?
    \answerYes{See~\Cref{sec:exp} and \Cref{app:experiments}.}
  \item Did you specify all the training details (e.g., data splits, hyperparameters, how they were chosen)?
    \answerYes{See~\Cref{app:experiments}.}
        \item Did you report error bars (e.g., with respect to the random seed after running experiments multiple times)?
    \answerYes{We report the standard deviation computed over multiple seeds for experiments on ZINC12k (\Cref{tab:count-zinc}), ogbg-molhiv (\Cref{tab:ogbg-hiv-baselines}) and on all generalisation experiments (\Cref{fig:4cycles-ego,fig:4cycles-ego-plus,fig:zinc}). We report the standard deviation for the ``Counting Substructures'' experiments (\Cref{tab:count-zinc}) in~\Cref{app:experiments}.}
        \item Did you include the total amount of compute and the type of resources used (e.g., type of GPUs, internal cluster, or cloud provider)?
    \answerYes{See~\Cref{app:experiments}. }
\end{enumerate}

\item If you are using existing assets (e.g., code, data, models) or curating/releasing new assets...
\begin{enumerate}
  \item If your work uses existing assets, did you cite the creators?
    \answerYes{}
  \item Did you mention the license of the assets?
    \answerYes{See~\Cref{app:experiments}.}
  \item Did you include any new assets either in the supplemental material or as a URL?
    \answerNA{}
  \item Did you discuss whether and how consent was obtained from people whose data you're using/curating?
    \answerNA{}
  \item Did you discuss whether the data you are using/curating contains personally identifiable information or offensive content?
    \answerNA{}
\end{enumerate}

\item If you used crowdsourcing or conducted research with human subjects...
\begin{enumerate}
  \item Did you include the full text of instructions given to participants and screenshots, if applicable?
    \answerNA{}
  \item Did you describe any potential participant risks, with links to Institutional Review Board (IRB) approvals, if applicable?
    \answerNA{}
  \item Did you include the estimated hourly wage paid to participants and the total amount spent on participant compensation?
    \answerNA{}
\end{enumerate}

\end{enumerate}

%%%%%%%%%%%%%%%%%%%%%%%%%%%%%%%%%%%%%%%%%%%%%%%%%%%%%%%%%%%%

\newpage

\appendix
\begin{Large}
    \begin{center}
        Supplementary Materials for \\
        \textbf{Understanding and Extending Subgraph GNNs \\
        by Rethinking Their Symmetries}
    \end{center}
\end{Large}
\section{Subgraph GNNs}\label{app:subgraphs}
\subsection{Review of Subgraph GNN architectures}

Here we review a series of previously proposed Subgraph GNNs, showing how the proposed architectures are captured by the formulation of Equation~\ref{eq:subgraph_gnn}. We report this here for convenience:
\begin{equation*}
    (A, X) \mapsto \big ( \mu \circ \rho \circ \mathcal{S} \circ \pi \big ) ( A, X ).
\end{equation*}
\noindent {\bf \camera{$(n-k)$-}Reconstruction GNN} \camera{by} \citet{cotta2021reconstruction} is the simplest Subgraph GNN: it applies a Siamese MPNN base-encoder $\gamma$ to \camera{$k$-}node-deleted subgraphs of the original graph and then processes the resulting representations with a set network. \camera{When $k=1$, these models are \emph{node-based} Subgraph GNNs.} More formally, $\pi = \pi_{\mathrm{ND}}$, a DeepSets network~\citep{zaheer2017deep} implements $\mu \circ \rho$, and $\mathcal{S}$ is realised with layers of the form:
\begin{equation}\label{eq:DSGNN}
    X^{i,(t+1)} = \gamma_t(A^i,X^{i,(t)})
\end{equation}

\noindent {\bf Equivariant Subgraph Aggregation Network (ESAN)} \camera{by} \citet{bevilacqua2022equivariant} extends Reconstruction GNNs in two main ways: First, introducing subgraph selection policies that allow for more general sets of subgraphs such as edge-deleted policies. Second, performing an in-depth equivariance analysis which advocates the use of the DSS layer structure introduced by \citet{maron2020learning}. This choice gives rise to a more expressive architecture that shares information between subgraphs. Formally, in ESAN, $\mathcal{S}$ is defined as a sequence of equivariant layers which process subgraphs as well as the aggregated graph $G_{\mathrm{agg}} = \big( A^{\mathrm{agg}}, X^{\mathrm{agg}} \big ) = \sum_{G^i\in B_G} G^i$. Each layer in $\mathcal{S}$ is of the following form: 
\begin{equation}\label{eq:DSSGNN}
    X^{i,(t+1)} = \gamma^{0}_{t}(A^i,X^{i, (t)}) + \gamma^{1}_{t}(A^{\mathrm{agg}},X^{\mathrm{agg}, (t)})
\end{equation}
\noindent with $\gamma^{0}_{t}, \gamma^{1}_{t}$ being two \emph{distinct} MPNN base-encoders. The above architecture is referred to as DSS-GNN. \citet{bevilacqua2022equivariant} also explore disabling component $\gamma^{1}$ and term this simplified model DS-GNN (which reduces to a Reconstruction GNN of \citet{cotta2021reconstruction} \camera{under node-deletion policies}\footnote{For this reason, we will only consider DS-GNN in the following proofs.}). In the same work, the considered policies are edge-covering~\citep[Definition~7]{bevilacqua2022equivariant}, that is, each edge in the original connectivity appears in the connectivity of at least one subgraph. In view of this observation, the authors consider and implement a simplified version of DSS-GNN, whereby $\gamma^1$'s operate on the original connectivity $A$, rather than $A^{\mathrm{agg}}$, that is:
\begin{align}\label{eq:DSSGNN_orig}
    X^{i,(t+1)} = \gamma^{0}_{t}(A^i,X^{i, (t)}) + \gamma^{1}_{t}(A, X^{\mathrm{agg}, (t)}).
\end{align}
\camera{Both DS- and DSS-GNN are \emph{node-based} Subgraph GNNs when equipped with policies in $\Pi$. Also, these policies} are clearly edge-covering and, in this work, we will consider DSS-GNN as defined by Equation~\ref{eq:DSSGNN_orig}.

\noindent {\bf GNN as Kernel (GNN-AK).} \citet{zhao2022from} employs an ego-network policy ($\pi=\pi_{\mathrm{EGO}}$), while each layer in $\mathcal{S}$ is structured as $ A \circ S$, where $S$ is a stacking of layers in the form of Equation~\ref{eq:DSGNN} and $A$ is an aggregation/pooling block in the form:
\begin{equation}\label{eq:GNNAK}
    x^{i,(t+1)}_j = \phi \big( h^{j,(t)}_j, \sum_\ell h^{j,(t)}_\ell  \big)
\end{equation}
\noindent with $\phi$ either concatenation or summation, $h_j^{i,(t)} = \big ( \gamma_{t}(A^i, X^{i,(t)}) \big )_j^\top$, for MPNN $\gamma_{t}$.
The authors introduce an additional variant (GNN-AK-ctx in the following) which also pools information from nodes in other subgraphs:
\begin{equation}\label{eq:GNNAK-ctx}
    x^{i, (t+1)}_j = \phi \big( h^{j,(t)}_j, \sum_\ell h^{j,(t)}_\ell, \sum_\ell h^{\ell,(t)}_j \big).
\end{equation}
In this paper we consider a more general case of global summation in Equations \eqref{eq:GNNAK}-\eqref{eq:GNNAK-ctx} \footnote{The original paper considers summation over each ego network which is specific to a particular policy. Such summations can be dealt with by adding masking node features.}.

\noindent {\bf Nested GNN (NGNN).} \citet{zhang2021nested} also uses $\pi = \pi_{\mathrm{EGO}}$ and applies an independent MPNN to each ego-network, effectively structuring $\mathcal{S}$ as a stack of layers in the form of Equation~\ref{eq:DSGNN}. This architecture differs in the way block $\rho$ is realised, namely by pooling the obtained local representations and running an additional MPNN $\gamma_\rho$ on the original graph:
\begin{equation}\label{eq:NestedGNN}
    x^{(\rho)}_j = \sum_\ell x^{j,(T)}_\ell \thinspace, \qquad x_G = \sum_j \big( \gamma_{\rho} (A, X^{(\rho)}) \big)_j
\end{equation}

\noindent {\bf ID-GNN.} \citet{you2021identity} proposes distinguishing messages propagated by ego-network roots. This architecture uses $\pi = \pi_{\mathrm{EGO+}(T)}$ policy and $\mathcal{S}$ as a $T$-layer stacking performing independent \emph{heterogeneous} message-passing on each subgraph:
\begin{equation}\label{eq:IDGNN}
    x^{i, (t+1)}_j = \upsilon_t \big( x^{i,(t)}_j, \sum_{\ell \sim_i j, \ell \neq i} \mu_{0,t}(x^{i,(t)}_\ell) + \mathds{1}_{[i \sim_i j]} \cdot \mu_{1,t}(x^{i,(t)}_i) \big)
\end{equation}
\noindent where $\mathds{1}_{[i \sim_i j]}$ if $i \sim_i j$, $0$ otherwise, and $\sim_i$ denotes the connectivity induced by $A^i$. \camera{
GNN-AK, GNN-AK-ctx, NGNN and ID-GNN are all, intrinsically, \emph{node-based} Subgraph GNNs.}

\camera{Interestingly, we notice that the contemporary work by \citet{papp2022theoretical} suggested using a node marking policy as a more powerful alternative to node deletion. Finally, we note that the model by \citet{vignac2020building} may potentially be considered as a Subgraph GNN as well.}

\revision{
    \subsection{The computational complexity of Subgraph GNNs}
    
    Other than proposing Subgraph GNN architectures, the works by~\citet{bevilacqua2022equivariant, zhang2021nested, zhao2022from} also study their computational complexity. In particular, \citet{bevilacqua2022equivariant} describe both the space and time complexity of a subgraph method equipped with generic subgraph selection policy and an MPNN as a base encoder. Given the inherent locality of traditional message-passing, the authors derive asymptotic bounds accounting for the sparsity of input graphs. Let $n, d$ refer to, respectively, the number of nodes and \emph{maximum node degree} of an input graph generating a subgraph bag of size $b$. The forward-pass asymptotic time complexity amounts to $\mathcal{O}(b \cdot n \cdot d)$, while the memory complexity to $\mathcal{O}(b \cdot (n + n \cdot d))$. For a node-based selection policy, $b = n$ so these become, respectively, $\mathcal{O}(n^2 \cdot d)$ and $\mathcal{O}(n \cdot (n + n d))$. The authors stress the explicit dependence on $d$, which is, on the contrary, lacking in \IGNs{3}. As we show in \Cref{app:upperbound}, \IGNs{3} subsume node-based subgraph methods, but at the cost of a cubic time and space complexity ($\mathcal{O}(n^3)$)~\citep{bevilacqua2022equivariant}. Amongst others, this is one reason why Subgraph GNNs may be preferable when applied to sparse, real world graphs (where we typically have $d \ll n$).
    
    As we have noted in the above, more sophisticated Subgraph GNNs layers may feature ``global'' pooling terms other than local message-passing: see, e.g., term $X^\mathrm{agg}$ for DSS-GNN~\citep{bevilacqua2022equivariant} in \Cref{eq:DSSGNN,eq:DSSGNN_orig} or the ``subgraph'' and ``context'' encodings in the GNN-AK-ctx model~\citep{zhao2022from} (second and third term in the summation of \Cref{eq:GNNAK-ctx}). In principle, each of these operations require a squared asymptotic computational complexity ($\mathcal{O}(n^2)$). However, we note that these terms are shared in the update equations of nodes / subgraphs: in practice, it is only sufficient to perform the computation once. In this case, the asymptotic time complexity would amount to $\mathcal{O}(n^2 \cdot d + n^2)$, i.e., still $\mathcal{O}(n^2 \cdot d)$. Therefore, these Subgraph GNNs retain the same asymptotic complexity described above.

    Our proposed \SUN{} layers involve the same ``local'' message-passing and ``global'' pooling operations: the above analysis is directly applicable, yielding the same asymptotic bounds.
 
    We conclude this section by noting that, for some specific selection policies, these bounds can be tightened. In particular, let us consider ego-networks and refer to $c$ as the maximum ego-network size. As observed by~\citet{zhang2021nested}, the time complexity of a Subgraph GNN equipped with such policy becomes $\mathcal{O}(n \cdot c \cdot d)$. Importantly, When ego-networks are of limited depth, the size of the subgraphs may be significantly smaller than that of the input graph; in other words $c \ll n$, reducing the overall forward-pass complexity.
}

\section{Proofs for Section \ref{sec:upperbound} -- Subgraph GNNs and 3-IGNs}\label{app:upperbound}
\subsection{3-IGNs as computational models}

Before proving the results in Section~\ref{sec:upperbound}, we first describe here a list of simple operations that are computable by \IGNs{3}. These opearations are to be intended as `computational primitives' that can then be invoked and reused together in a way to program these models to implement more complex functions. We believe this effort not only serves our need to define those atomic operations required to simulate Subgraph GNNs, but also, it points out to an interpretation of \IGNs{} as (abstract) comprehensive computational models beyond deep learning on hypergraphs. We start by describing the objects on which \IGNs{3} operate, and how they can be interpreted as bags of subgraphs.

\subsubsection{The 3-IGN computational data structure}\label{app:3IGN_data_struct}

The main object on which a \IGN{3} operates is a `cubed' tensor in $\mathbb{R}^{n^3 \times d}$, typically referred to as $\mathcal{Y}$ in the following. An $S_n$ permutation group act on the first three dimensions of this tensor as:
\begin{align*}
    \big( \sigma \cdot \mathcal{Y} \big)_{ijkl} &= \mathcal{Y}_{\sigma^{-1}(i)\sigma^{-1}(j)\sigma^{-1}(k)l} \quad \forall \sigma \in S_n
\end{align*}
whereas the last dimension ($l$ above) hosts $d$ `channels' not subject to the permutation group.

The action of the permutation group on $[n]^k$ decomposes it into orbits, that is equivalence classes associated with the relation $\sim_S$ defined as:
\begin{equation*}
    \forall x, y \in [n]^k, x \sim_S y \iff \exists \sigma \in S_n : \sigma(x) = y
\end{equation*}

Orbits induce a partitioning of $[n]^k$. In particular, for $k=3$ we have:
\begin{align}
    [n]^3 &= \morbit{iii} \sqcup \morbit{iij} \sqcup \morbit{iji} \sqcup \morbit{ijj} \sqcup \morbit{ijk} \label{eq:orbits} \\ 
    \morbit{iii} &= \big \{ (i,i,i) \thinspace | \thinspace i \in [n] \big \}, \nonumber \\
    \morbit{iij} &= \big \{ (i,i,j) \thinspace | \thinspace i,j \in [n], i \neq j \big \}, \nonumber \\
    \morbit{iji} &= \big \{ (i,j,i) \thinspace | \thinspace i,j \in [n], i \neq j \big \}, \nonumber \\
    \morbit{ijj} &= \big \{ (i,j,j) \thinspace | \thinspace i,j \in [n], i \neq j \big \}, \nonumber \\
    \morbit{ijk} &= \big \{ (i,j,k) \thinspace | \thinspace i,j \in [n], i \neq j \neq k \big \} \nonumber 
\end{align}

\begin{figure}[t]
    \centering%\vspace{-1mm}
    \includegraphics[width=0.75\linewidth]{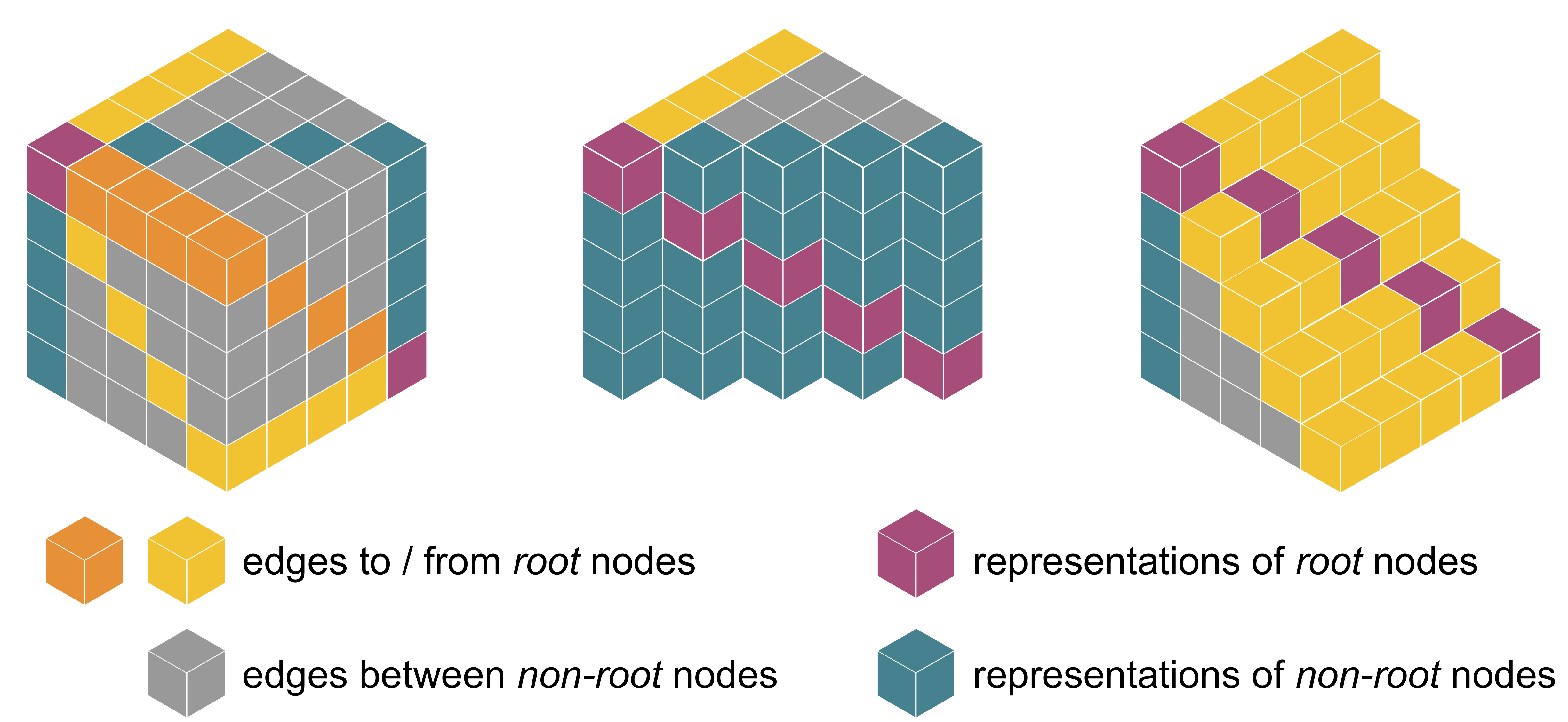}%\vspace{-2mm}
    \caption{Depiction of cubed tensor $\mathcal{Y}$, its orbit-induced partitioning and the related semantics when $\mathcal{Y}$ is interpreted as a bag of node-based subgraphs, $n = 5$. Elements in the same partition are depicted with the same colour. Left: the whole tensor. Middle and right: sections displaying orbit representations $X_{iii}, X_{ijj}, X_{iij}$ in their entirety.}\vspace{-3mm}
    \label{fig:orbit_dec}
\end{figure}

As studied in~\citet{Albooyeh2019incidence}\footnote{The authors study the more general case of incidence tensors of any order, for which our three-way cubed tensor is a special case.}, each of these orbits indexes a specific face-vector, that is a (vectorised) sub-tensor of $\mathcal{Y}$ with a certain number of free index variables, which determines its `size'. Importantly, this entails that the partitioning in Equation~\ref{eq:orbits} induces the same partitioning on $\mathcal{Y}$, so that we can interpret the cubed tensor $\mathcal{Y}$ as a disjoint union of following face-vectors\footnote{The use of symbol `$\cong$', referring to an isomorphism, follows~\citet{Albooyeh2019incidence}.}:
\begin{align}
    \mathcal{Y} &\cong X_{iii} \sqcup X_{iij} \sqcup X_{iji} \sqcup X_{ijj} \sqcup X_{ijk} \label{eq:partitioning} \\ 
    &\qquad X_{iii} = \big(\mathcal{Y}\big)_{\morbit{iii}} \qquad \mathrm{(size\ 1)},\nonumber \\
    &\qquad X_{iij} = \big(\mathcal{Y}\big)_{\morbit{iij}} \qquad \mathrm{(size\ 2)},\nonumber \\
    &\qquad X_{iji} = \big(\mathcal{Y}\big)_{\morbit{iji}} \qquad \mathrm{(size\ 2)},\nonumber \\
    &\qquad X_{ijj} = \big(\mathcal{Y}\big)_{\morbit{ijj}} \qquad \mathrm{(size\ 2)},\nonumber \\
    &\qquad X_{ijk} = \big(\mathcal{Y}\big)_{\morbit{ijk}} \qquad \mathrm{(size\ 3)} \nonumber 
\end{align}
\noindent or, more compactly, as $\mathcal{Y} \cong \bigsqcup_{\omega} X_{\omega}, \omega \in \Omega = \{ iii, iij, iji, ijj, ijk \}$. 
According to this notation, we consider subscripts in $\Omega$ for $X$ as indexing variables for $\mathcal{Y}$. Importantly, since they directly reflect the set of indexes in \orbit{iii}, \orbit{iij}, \orbit{iji}, \orbit{ijj}, \orbit{ijk}, when subscripting $X$, $i,j,k$ are always distinct amongst each other. At the same time, as it can be observed from the definition above, we keep duplicate indexing variables in the subscripts of $X$'s to highlight the characteristic equality pattern of the corresponding orbit. As an example, element $\big(\mathcal{Y}\big)_{1,1,1} \in \mathbb{R}^d$ uniquely belongs to face-vector $X_{iii}$, elements $\big(\mathcal{Y}\big)_{1,1,2}, \big(\mathcal{Y}\big)_{1,2,1}, \big(\mathcal{Y}\big)_{1,2,2}$ to, respectively, face-vectors $X_{iij}, X_{iji}, X_{ijj}$, $\big(\mathcal{Y}\big)_{1,2,3}$ to $X_{ijk}$. Since each of these face-vectors is uniquely associated with a particular orbit, we will more intuitively refer to them as `orbit representations' in the following. 

From the considerations above, it is interesting to notice that orbit representations have a precise, characteristic collocation within the cubed tensor $\mathcal{Y}$, directly induced by the equality patterns of the orbits they are indexed by. In particular, $X_{iii}$ corresponds to elements on the main diagonal of the cube, $X_{iij}$, $X_{iji}$, $X_{ijj}$ to its three diagonal planes (with main-diagonal excluded), while $X_{ijk}$ corresponds to the the overall adjacency cube (with the main diagonal and the three diagonal planes excluded). The described partitioning is visually depicted in~\citet[Figure in page 6]{Albooyeh2019incidence} and \Cref{fig:orbit_dec} (\camera{corresponding to \Cref{fig:orbit_dec_sem} in the main paper}).

\subsubsection{Bag-of-subgraphs interpretation}\label{app:bag_interpretation}

An important observation underpinning the majority of our results is that the described cubed tensor $\mathcal{Y}$ can represent bags of node-based subgraphs --- this is in contrast with standard interpretations whereby this tensor represents a $3$-hypergraph~\citep{maron2018invariant,maron2019provably}. In particular, as depicted in Figure~\ref{fig:symmetries}, we arrange subgraphs on the first axis, whereas the second and third axes index nodes.

Accordingly, the adjacency matrix and node representations for subgraph $\bar{i}$ are in sub-tensor $\mathcal{Z}_{\bar{i}} = (\mathcal{Y})_{\bar{i},j,k,l} \in \mathbb{R}^{n^2 \times d}$, with $j,k = 1, \mathellipsis, n, l =1, \mathellipsis, d$. Here, coherently with~\citep{maron2018invariant}, we assume node representations are stored in the on-diagonal entries ($j=k$) of $\mathcal{Z}_{\bar{i}}$, while off-diagonal terms ($j \neq k$) host edges, i.e. connectivity information.
% \fab{maybe this is a good place to say about $d$ and the padding required for edges}.

We note that this interpretation of $\mathcal{Y}$ assigns meaningful semantics to orbit representations (\camera{see \Cref{fig:orbit_dec} for a visual illustration}):
\begin{itemize}
    \item $X_{iii}$ stores representations for root nodes;
    \item $X_{ijj}$ stores representations for non-root nodes;
    \item $X_{iij}$ stores connections incoming into root nodes;
    \item $X_{iji}$ stores connections outgoing from root nodes;
    \item $X_{ijk}$ stores connections between non-root nodes.
\end{itemize}
To come back to the examples above, and consistently with the aforementioned semantics, $\big(\mathcal{Y}\big)_{1,1,1}$ represents subgraph $1$'s root node (that is node $1$); $\big(\mathcal{Y}\big)_{1,1,2}, \big(\mathcal{Y}\big)_{1,2,1}$ the connectivity between such root node and node $2$ in the same subgraph; $\big(\mathcal{Y}\big)_{1,2,2}$ represents node $2$ in subgraph $1$; $\big(\mathcal{Y}\big)_{1,2,3}$ the connectivity from node $3$ to node $2$ in subgraph $1$.

It is important to note how this interpretation induces a correspondence between the \IGN{3} tensor $\mathcal{Y}$ and tensors $\mathcal{A}, \mathcal{X}$ introduced in the main paper,~\Cref{sec:symmetries}. $\mathcal{X}$ gathers node features across subgraphs and is therefore in correspondence with $X_{iii} \sqcup X_{ijj}$; $\mathcal{A}$ hosts subgraph connectivity information and is thus in correspondence with $X_{iij} \sqcup X_{iji} \sqcup X_{ijk}$.

As a last note, this interpretation already preludes the more general, less constrained weight sharing pattern advocated by the \ReIGN{2} framework, which prescribes, for example, parameters specific to root- and non-root-updates. See Figure~\ref{fig:ops}. This will become more clear in the following (see Equation~\ref{eq:orbit_update_equation}).

\subsubsection{Updating orbit representations}

A \IGN{3} architecture has the following form:
\begin{equation}\label{eq:3IGN_model}
    (A, X) \mapsto \big( m \circ h \circ \mathcal{I} \big)(A, X) \thinspace , \qquad \mathcal{I} = L^{(T)} \circ \sigma \mathellipsis \circ \sigma \circ L^{(1)}
\end{equation}
\noindent where $m$ is an MLP, $h$ is an invariant linear `pooling' layer and $L$'s are equivariant linear \IGN{k}-layers, with \texttt{k} $\leq 3$. From here onwards we will assume $\sigma$'s are ReLU non-linearities.

\citet{Albooyeh2019incidence} show that, effectively, a linear \IGN{3}-layer $L$ updates each orbit representation such that output $X_{\omega'}$ is obtained as the sum of linear equivariant transformations of input orbit representations $X_{\omega}$\footnote{Even if omitted, each face-vector update equation includes a bias term.}:
\begin{align}
    \mathcal{Y}^{(t+1)} &= L(\mathcal{Y}^{t}) \cong \bigsqcup_{\omega' \in \Omega} X^{(t+1)}_{\omega'} \\
    X^{(t+1)}_{\omega'} &= \sum_{\omega \in \Omega} \boldsymbol{W}^{\omega \rightarrow \omega '} ( X^{(t)}_{\omega} ) \label{eq:orbit_update_equation}
\end{align}
\noindent where, as the authors show, each $\boldsymbol{W}^{\omega \rightarrow \omega '}$ corresponds to a linear combination of all pooling-broadcasting operations defined between input-output orbit representation $X_{\omega}, X_{\omega'}$:
\begin{align}
    \boldsymbol{W}^{\omega \rightarrow \omega '} ( X_{\omega} ) &= \sum_{\substack{\mathbb{P} \subseteq [m] \\ \mathbb{B} \subseteq \langle 1, \mathellipsis, m' \rangle \\ |\mathbb{B}| = m - |\mathbb{P}|}} W^{\omega \rightarrow \omega '}_{\mathbb{B,P}} \texttt{Broad}_{\mathbb{B},m'} \big(  \texttt{Pool}_{\mathbb{P}} (X_{\omega}) \big) \label{eq:pool_broadcast}
\end{align}
\noindent with $m$ the size of input orbit representation $X_{\omega}$ and $m'$ the size of the output one. \texttt{Broad} and \texttt{Pool} are defined in~\citep{Albooyeh2019incidence} as follows.

\paragraph{Pooling}
Let $X_{i_1, \mathellipsis i_m}$ be a generic face-vector of size $m$ indexed by ${i_1, \mathellipsis i_m}$. Let $\mathbb{P} = \{p_1, \mathellipsis, p_\ell\} \subseteq [m]$. \texttt{Pool}$_\mathbb{P}$ sums over the indexes in $\mathbb{P}$:
\begin{align}\label{eq:pooling}
    \texttt{Pool}_{\mathbb{P}}\big( X_{i_1, \mathellipsis i_m} \big) = 
    \sum_{\substack{i_{p_1} \neq \mathellipsis \neq i_{p_\ell} \\ i_{p_1} \in [n], \mathellipsis, i_{p_\ell} \in [n]}} X_{i_1, \mathellipsis i_m}
\end{align}
\noindent where the inequality constraints in the summation derive from the fact that $X_{i_1, \mathellipsis i_m}$ represents an orbit of the permutation group. To shorten the notation, we will write pooling operations as:
\begin{equation*}
    \texttt{Pool}_{\mathbb{P}}\big( X_{i_1, \mathellipsis i_m} \big) = \pool{}{out} X_{i_1, \mathellipsis i_m}
\end{equation*}
\noindent where $ out = \{i_p \thinspace | \thinspace p \in [m] \setminus \mathbb{P}\} $ i.e. the set of indexes which are \emph{not} pooled over, from which it follows that the cardinality of $out$ states the size of the resulting object. For example, $\pool{}{i_1} X_{i_1,i_2}$ returns a size-1 object from a size-2 face-vector by summing over index variable $i_2$.

The pooling operation applies in the same way when we repeat index variables in the subscript of the input orbit representation: as more concrete examples, when interpreting $\mathcal{Y}$ as a bag of subgraphs, $\pool{}{i} X_{ijj}$ sums the representations of all non-root nodes (subgraph readout) as:
\begin{align*}
    \Big( \texttt{Pool}_{j}\big( X_{ijj} \big) \Big)_1 = \big( \pool{}{i} X_{ijj} \big)_1 = \Big( \sum_{j \in [n], j \neq i} \big( \mathcal{Y} \big)_{i,j,j} \Big)_{1} = \sum_{j \in [n], j \neq 1} \big( \mathcal{Y} \big)_{1,j,j}
\end{align*}
Similarly, $\pool{}{j} X_{ijj}$ sums non-root node representations across subgraphs (cross-subgraph aggregation). Set $out$ can be empty, in which case pooling amounts to global summation: e.g., $\pool{}{} X_{iii}$ sums the representations of all root nodes across the bag. Pooling boils down to the identity operation when $out$ replicates the free indexes in $X$ -- as no pooling is effectively performed. We will refrain from explicitly writing this operation.

\paragraph{Broadcasting}
Let $X_{i_1, \mathellipsis i_m}$ be a generic face-vector of size $m$ indexed by ${i_1, \mathellipsis i_m}$ and $\mathbb{B} = (b_1, \mathellipsis, b_m)$ a tuple of $m$ indexes, with $b_j \in [m'], j = 1, \mathellipsis m, m' \geq m$. Operation $\texttt{Broad}_{\mathbb{B},m'}$ ``broadcasts'' $X$ over a target size-$m'$ face-vector in the sense that it identifies $X$ by the target index sequence $\mathbb{B}$ and broadcasts across the remaining $m' - m$ indexes:
\begin{align}
    \Big( \texttt{Broad}_{\mathbb{B},m'}\big( X_{i_1, \mathellipsis, i_{m}} \big) \Big)_{i_1, \mathellipsis, i_{m'}} = X_{i_{b_1}, \mathellipsis, i_{b_m}}
\end{align}
For example, if input $X_{i_1, i_2}$ is a size-$2$ face-vector and output $Y_{i_1, i_2, i_3}$ is a size-$3$ face-vector, $\texttt{Broad}_{(1,2),3}\big( X_{i_1, i_2} \big)$ maps $X$ onto the first two indexes of $Y$, and broadcasts along the third. As another example, for output size-$2$ face-vector $Z_{i_1, i_2}$, $\texttt{Broad}_{(2,1),2}\big( X_{i_1, i_2} \big)$ effectively implements the `transpose' operation. Similarly as above, we shorten the notation. Let us define 
$\iota$ such that, $\forall \ell \in [m']$:
\begin{align*}
    \iota(\ell) = \begin{cases}
                    j\ \text{s.t.}\  b_j = \ell & \text{if } \ell \in \mathbb{B}, \\
                    * & \text{otherwise}.
                  \end{cases}
\end{align*}
\noindent where $*$ indicates an index over which broadcasting is performed. We write $\mathbb{C} = (i_{\iota(1)}, \mathellipsis, i_{\iota(m')})$ and rewrite the broadcast operation as:
\begin{equation*}
    \texttt{Broad}_{\mathbb{B},m'}\big( X_{i_1, \mathellipsis, i_{m}} \big) = \broad{}{\mathbb{C}} X_{i_1, \mathellipsis, i_{m}}
\end{equation*}
Accordingly, the examples above are rewritten as $\broad{}{i_1,i_2,*} X_{i_1,i_2}, \broad{}{i_2,i_1} X_{i_1,i_2}$:
\begin{align*}
    \big(Y_{i_1,i_2,i_3} \big)_{1,2,3} &= 
        \big( \texttt{Broad}_{(1,2),3}\big( X_{i_1,i_2} \big) \big)_{1,2,3} =
        \big( \broad{}{i_1,i_2,*} X_{i_1,i_2} \big)_{1,2,3} = 
        \big( X_{i_1,i_2} \big)_{1,2} \\
    \big(Z_{i_1,i_2} \big)_{1,2} &= 
        \big( \texttt{Broad}_{(2,1),2}\big(X_{i_1,i_2}\big) \big)_{1,2} =
        \big( \broad{}{i_2,i_1} X_{i_1,i_2} \big)_{1,2} =
        \big( X_{i_1,i_2} \big)_{2,1}
\end{align*}

In the subscript of $\broad{}{}$, we can conveniently retain the equality pattern of the orbit representation where we are broadcasting onto. For a concrete example, $X_{ijj} = \broad{}{i,*j,*j} X_{iii}$ broadcasts the root node representations over non-root ones, in a way that each non-root node $j$ in subgraph $i$ gets the representation of root node $i$. It has to be noted that, in these cases, the length of tuple $\mathcal{C}$ may not correspond to the size of the output face-vector, i.e. in those cases where indexes repeat, as above. Finally, let us note that, when both input and target face-vectors have size $m$ and $\mathbb{B} = [m]$, the broadcasting operation boils down to the identity, e.g. $\broad{}{ijk} X_{ijk}$. We will omit it in writing in these cases.

The above results suggest that a way to describe a \IGN{3} stacking $\mathcal{I} = L^{(T)} \circ \sigma \mathellipsis \circ \sigma \circ L^{(1)}$ is to specify how each orbit representation $X_{\omega'}$ is updated from time step $(t)$ to $(t+1)$, according to Equation~\ref{eq:orbit_update_equation}, expanded as per Equation~\ref{eq:pool_broadcast}. In other words, stacking $\mathcal{I}$ is described by specifying, for each layer $L^{(t)}$, every linear operator $W^{\omega \rightarrow \omega '}_{\mathbb{B,P}, t}$ in Equation~\ref{eq:pool_broadcast}. This is the main strategy we adopt in the proofs described in the following.

As a last note, in an effort to ease the notation, we will:
\begin{enumerate}
    \item Describe updates for $X_{\omega'}$'s (in the form of Equation~\ref{eq:orbit_update_equation}) only when non-trivial;
    \item (For each of the above) specify only the non-null linear operators $W^{\omega \rightarrow \omega '}_{\mathbb{B,P}, t}$ and corresponding terms.
\end{enumerate}
For example, if layer $L^{(t)}$ only applies linear transformation $W^{iii \rightarrow iii}_{(1), \varnothing, t} = W$ to orbit representation $X^{(t)}_{iii}$, we will simply write:
\begin{equation*}
    X^{(t+1)}_{iii} = W \cdot X^{(t)}_{iii}
\end{equation*}
\noindent implying that, according to 2., all other terms in Equation~\ref{eq:pool_broadcast} are nullified by operator $\mathbf{0}$ and, according to 1., all other orbit updates read as $X^{(t+1)}_{\omega} = I_d \cdot X^{(t)}_{\omega}$.

\subsubsection{Pointwise operations}

We define as `pointwise' those operations which only apply to the feature dimension(s) and implement a form of channel mixing. By operating independently on the feature dimension(s), these are trivially equivariant. As these operations can be performed on any orbit representation, we will not specify them in the following descriptions, and will simply assume to be working with generic tensors $X, Y, Z$.

Linear transformations are the most natural pointwise operations the \IGN{3} framework supports. Some of these are of particular interest as they will be heavily used in our proofs. We describe them below and define convenient notation for them.

\paragraph{Copy/Routing}
Pointwise linear operators can copy some specific feature channels in the input tensor and route them into some other output channels. Let $X$ refer to a tensor representing orbit elements with $d_{in}$ channels, and $Y$ to an output tensor with $d_{out}$ channels. We will write:
\begin{align*}
    Y &= \cp{a}{b}{e}{f} X
\end{align*}
\noindent to refer to the operation which writes channels $a$ to $b$ (included) in $X$ into channels $e$ to $f$ (included) in $Y$, with $a \leq b \leq d_{in}$, $e \leq f \leq d_{out}$, $b - a = f - e$. Here, $\cp{a}{b}{e}{f}$ is used to conveniently denote operator $W$: a matrix $d_{out} \times d_{in}$ where the square submatrix $W_{e:f, a:b} = I_{b-a+1}$ and other entries are $0$. When omitting left indices we start from the first channel, i.e., $\cp{}{b}{e}{f}$ = $\cp{1}{b}{e}{f}$. When omitting right indices we end at the last channel, i.e., $\cp{a}{b}{e}{}$ = $\cp{a}{b}{e}{f}$ for $f$ output channels. We will use notation $\hemcp{a}{b}{c}{d}{e}$ to specify the target has $e$ channels, whenever not clear from the context.

\paragraph{Selective copy/routing}
With a slight notation overload, we also write $\selcp{in}{out}$ --- with $in$, $out$ being index tuples of the same cardinality $\ell$ --- to refer to specific channels in, respectively, input and output tensors. For instance, $\selcp{a,b,c}{d,e,f}$ copies (or routes) input channels $a,b,c$ into output channels $d,e,f$, with such indices being not necessarily contiguous. Again, this operation is linear and is implemented by a matrix $W$ having $1$ in all entries in the form $(out_k, in_k), k \in [\ell]$, $0$ elsewhere.

\paragraph{Concatenation}
Two (compatible) operands can be concatenated by means of the copy/routing operation above along with summation. For example, if the two are $d$ dimensional, it is sufficient to route them into the two distinct halves of a $2d$-dimensional output tensor and then sum the two:
\begin{align*}
    Z &= \cp{}{}{}{d} X + \cp{}{}{d+1}{2d} Y
\end{align*}
 
\paragraph{Concurrent linear transformations}
We note that one single linear operator can apply multiple linear transformations on (a specific subset of) channels of the input tensor. Let $W_1, W_2$ be two operators of size $d_{out} \times d_{in}$. Obtained as the vertical stacking of $W_1, W_2$, operator $W$ can be applied to tensor $X$ with $d_{in}$ channels. It produces an output tensor $Y$ with $2 \cdot d_{out}$ channels, where the first $d_{out}$ channels store the result of applying $W_1$, channels $d_{out} + 1$ to $2 \cdot d_{out}$ store that from the application of $W_2$. We write:
\begin{align*}
    Y &= \glue{W_1}{W_2} X
\end{align*}
Pre-multiplying routing operators allows these transformations to be applied to a selection of a subset of input channels. As an example, let $X$ be an input with $2 \cdot d$ channels, and $W_1, W_2$ be two $d \times d$ linear operators. Expression:
\begin{align*}
    Y &= \glue{W_1 \cdot \cp{}{d}{}{d}}{W_2 \cdot \cp{d+1}{2d}{}{d}} X
\end{align*}
\noindent effectively applies $W_1$ to the first $d$ channels of $X$, $W_2$ to channels $d+1$ to $2d$. In fact, $\glue{W_1 \cdot \cp{}{d}{}{d}}{W_2 \cdot \cp{d+1}{2d}{}{d}}$ consists, as a whole, of a block diagonal matrix made up of operators $W_1$ (leftmost upper block) and $W_2$ (rightmost lower block). Clearly, these operations can be extended to more than two concurrent linear operators.

We note that this notation also captures the following `replication' operation:
\begin{align*}
    Y &= \glue{I_d}{I_d} X
\end{align*}
\noindent which outputs two (stacked) copies of $d$-dimensional input tensor $X$.

\paragraph{Concurrent transformations and channel-wise summation}
Let $X$ be an tensor with $2 \cdot d$ channels, and $W_1, W_2$ two $d \times d$ linear operators. The following expression applies $W_1$ to the first $d$ channels of $X$, $W_2$ to channels $d+1$ to $2d$ \emph{and} sums the result channel by channel:
\begin{align*}
    Y &= \sumglue{W_1}{W_2} X
\end{align*}
In fact, $\sumglue{W_1}{W_2}$ can be interpreted as single linear operator $d \times 2 \cdot d$ obtained by horizontally stacking $W_1, W_2$. As a particular case, we can simply sum the two halves of $X$ by $\sumglue{I_d}{I_d} X$.

\paragraph{A note on non-linearities}
A \IGN{3}-layer stacking is in the form $\mathcal{I} = L^{(T)} \circ \sigma \mathellipsis \circ \sigma \circ L^{(1)}$, where $L$'s are linear equivariant layers and $\sigma$'s are ReLU non-linearities. In principle, these non-linearities in between layers may alter the result of linear computation they perform. For example, in order to perform an exact copy of input representation $X$, it may not be sufficient to simply choose an identity weight matrix $I_d$: the following ReLU non-linearity would clip negative entries to $0$, thus invalidating the correctness of the operation. However, we note that the identity function can be implemented by a ReLU-network, as $y = x = \sigma(x)-\sigma(-x)$. This means that the copy operation can be realised by a \IGN{3}-layer stacking as:
\begin{equation*}
    Y = \sumglue{I_d}{-I_d} \cdot \Big( \sigma \big( \glue{I_d}{-I_d} X \big) \Big)
\end{equation*}
This effectively provides us with a way to `choose', in a \IGN{3} layer stacking, when to apply $\sigma$'s and when not to. Indeed, we can always work with $2d$ channels where all entries are non-negative: the positive entries in the first $d$ channels store the originally positive ones, those in the second $d$ channels store the originally negative ones, negated. This expansion can be realised by one layer as $Y = \sigma \big( \glue{I_d}{-I_d} X \big)$ and is such that ReLU activations act neutrally. Any linear transformation $W$ is now applied as $Y = \sigma \big( U X \big)$, with $U = \glue{I_d}{-I_d} \cdot W \cdot \sumglue{I_d}{-I_d} $. Whenever computation requires the application of ReLUs after linear transformation $W$, it is sufficient to perform the following: $Y = \glue{I_d}{-I_d} \cdot \Big( \sigma \big( V X \big) \Big)$, with $V = W \cdot \sumglue{I_d}{-I_d}$. Operator $\glue{I_d}{-I_d}$ will effectively be absorbed by the linear transformation in the following layer. This doubled representation also allows to apply non-linearities to some particular channels only. This can be realised by $Y = \sigma \big( U' X \big)$, with $U'$ constructed as $U$ above, but with the difference being that entries in its $(d+e)$-th row are set to $0$, if $e$ is an output channel where the ReLU non-linearity takes effect.

These considerations show that, in fact, interleaving linear (equivariant) layers with ReLU non-linearities does not hinder the possibility of performing (partially) linear computation, which can always be recovered at the cost of having additional channels and/or layers. In view of the above observations, and in an effort to ease the notation, in the proofs reported in the following sections we will thus assume to be allowed to stack linear layers with no ReLU activations in between.

\paragraph{Multi-Layer Perceptron}
\IGNs{3} can naturally implement the application of a Multi-Layer Perceptron (MLP) to the feature dimension(s) of a particular orbit representation, as it stacks linear layers, interspersed with non-linearities. Each of these dense linear layers is trivially equivariant, and so is the overall stacking. We generally write:
\begin{equation*}
    Y = \mlp{f} X
\end{equation*}
\noindent to indicate the application of an MLP implementing (or approximating) function $f$. Additionally, we use the notation $\selmlpcp{in}{out}{f}$ to indicate that $\mlp{f}$ acts on the $in$ channels of its input and writes over channels $out$ of the output. This behaviour can be obtained by multiplying the MLP inputs by operator $\selcp{in}{}$ and its output by operator $\selcp{}{out}$.

% https://open.spotify.com/album/1xrXrgQDQzTlGrDYhX8ikT?si=RpdvmoGjRUikoO8K_iR4rQ
\paragraph{MLP-copy}
We can also combine copy and MLP operations together. In other words, we can apply an MLP on some particular channels while copying some others; we write this operation as
\begin{equation*}
    Y = \glue{\selcp{in}{out}}{\selmlpcp{in'}{out'}{}} X
\end{equation*}
This operation can be implemented by `embedding' the weights of the MLP in appropriately sized matrices where remaining entries perform copy/routing operations and are not affected by non-linearities (see discussion above). More in specific, to see how this is implemented, let the MLP $\selmlpcp{in'}{out'}{}$ above have weight matrices $\big( W_1, W_2, \mathellipsis, W_L \big)$ with $W_1 \in \mathbb{R}^{d_1 \times |in'|}$, $W_l \in \mathbb{R}^{d_{l} \times d_{l-1}}, \forall l = 2 \dots L-1$, $W_L \in \mathbb{R}^{|out'| \times d_{L-1}}$. Then, $Y$ is obtained by stacking:
\begin{equation*}
    Y = \big( U_L \circ \sigma \mathellipsis \circ \sigma \circ U_1 \big) X
\end{equation*}
where $U_l, l \in [L]$ are linear operators obtained as follows. 
\begin{align*}
    U_1 &= \glue{I_{(1)}}{-I^{*}_{(1)}} \cdot \glue{\selcp{in}{:}}{W_1 \cdot \selcp{in'}{:}} \\
    &\quad I_{(1)} = I_{|in| + d_1} \\
    &\quad I^{*}_{(1)} = \glue{\sumglue{I_{|in|}}{\mathbf{0}_{|in| \times d_1}}}{\mathbf{0}_{d_1 \times (|in| + d_1)}}
\end{align*}
For any $l = 2 \dots L-1$:
\begin{align*}
    U_l &= \glue{I_{(l)}}{-I^{*}_{(l)}} \cdot V_l \cdot \sumglue{I_{|in| + d_{l-1}}}{-I_{|in| + d_{l-1}}} \\
    &\quad V_l = \glue{\sumglue{I_{|in|}}{\mathbf{0}_{|in|\times d_{l-1}}}}{\sumglue{\mathbf{0}_{d_l \times |in|}}{W_l}} \\
    &\quad I_{(l)} = I_{|in| + d_l} \\
    &\quad I^{*}_{(l)} = \glue{\sumglue{I_{|in|}}{\mathbf{0}_{|in| \times d_l}}}{\mathbf{0}_{d_l \times (|in| + d_l)}}
\end{align*}
Finally:
\begin{align*}
    U_L &= \selcp{:}{(out || out')} \cdot \glue{I_{(L)}}{-I^{*}_{(L)}} \cdot V_L \cdot \sumglue{I_{|in| + d_{L-1}}}{-I_{|in| + d_{L-1}}} \\
    &\quad V_L = \glue{\sumglue{I_{|in|}}{\mathbf{0}_{|in|\times d_{L-1}}}}{\sumglue{\mathbf{0}_{|out'| \times |in|}}{W_L}} \\
    &\quad I_{(L)} = I_{|in| + |out'|} \\
    &\quad I^{*}_{(L)} = \glue{\sumglue{I_{|in|}}{\mathbf{0}_{|in| \times |out'|}}}{\mathbf{0}_{|out'| \times (|in| + |out'|)}}
\end{align*}
\noindent with $(out || out')$ the concatenation of the two output index tuples $out$ and $out'$.

\paragraph{Logical AND}
For inputs in $\{0, 1\}$, it is possible to \emph{exactly} implement the logical $\mathrm{AND}$ function $y = f(a, b) = a \land b $ as $y = \sigma \big ( +2 \cdot a + 2 \cdot b - 3 \big)$, with $\sigma$ being the ReLU non-linearity. As \IGNs{3} can implement pointwise MLPs, this operation can be performed as well when operands are on two distinct channels of the same orbit representation. In particular, for a $d$-dimensional input $X$ we write:
\begin{equation*}
    Y = \selmlpcp{a,b}{c}{\land} X
\end{equation*}
\noindent to indicate the operation that applies a logical AND between channels $a$ and $b$ of input $X$ and writes the result into channel $c$ in output $Y$, with $a,b,c \in [d]$. For this operation in particular, we also imply that the rest of $X$ is written in the remaining output channels, that is, we imply we are, in fact, performing $\glue{\selcp{[d] \setminus \{c\}}{[d] \setminus \{c\}}}{\selmlpcp{a,b}{c}{\land}}$.

\paragraph{Clipping}
An MLP equipped with one hidden layer and ReLU activation can exactly implement the $1$-clipping function $f_\downarrow = \min(x, 1)$. In practice, $f_\downarrow$ clips inputs at the value of $1$. For a scalar input, this is realised as:
\begin{equation*}
    y = - \sigma \big ( - x + 1 \big ) + 1
\end{equation*}
An MLP $\mlp{\mathrm{\downarrow}}$ implementing this function on each channel of a multi-dimensional input $X$ is constructed as:
\begin{equation*}
    Y = - I_d\ \sigma \big ( - I_d\ X + \mathbf{1}_d \big ) + \mathbf{1}_d
\end{equation*}
\noindent where $\mathbf{1}_d$ is the $d$-dimensional one-vector. The existence of such an MLP entails that \IGNs{3} can exactly implement the clipping function as well, as a pointwise operation.

\subsection{Implementing node-based selection policies}\label{app:policy_implementation}
We now show how \IGNs{} can compute node-based selection policies, proving Lemma~\ref{lemma:3IGN_implements_policies}. We will make use of concepts introduced above; the partitioning into orbits and the computational primitives in particular.

\begin{proof}[Proof of Lemma~\ref{lemma:3IGN_implements_policies}]
    We assume to be given an $n$-node graph $G$ in input, represented by tensor $A \in \mathbb{R}^{n \times n \times d}$. $A$ is subject to $S_n$ symmetry, which partitions it into two distinct orbits: on-diagonal terms $A_{ii}$ and off-diagonal terms $A_{ij}, i \neq j$. $A_{ii}$ store $d$-dimensional node features; $A_{ij}$ the binary graph connectivity in its first channels, with others being $0$-padded.
    
    The first operation, common to the implementation of all the policies, consists of lifting the two-way tensor $A$ to the three-way tensor $\mathcal{Y} \in \mathbb{R}^{n^3 \times d}$ interpreted as a bag of subgraphs and partitioned into orbit tensors as described above in Section~\ref{app:3IGN_data_struct}. This step is realised by the following broadcasting operations\footnote{This layer effectively implements the \emph{null} policy, see~\Cref{app:experiments}.}:
    \begin{align*}
        X^{(0)}_{iii} &= \broad{}{i,i,i} A_{ii} \\
        X^{(0)}_{jii} &= \broad{}{*,i,i} A_{ii} \\
        X^{(0)}_{iij} &= \broad{}{i,i,j} A_{ij} \\
        X^{(0)}_{iji} &= \broad{}{i,j,i} A_{ij} \\
        X^{(0)}_{kij} &= \broad{}{*,i,j} A_{ij} 
    \end{align*}
    
    We now focus on each of the considered policies in particular.
    
    [\emph{\textbf{node-deletion}}] Given that $X_{iij}, X_{iji}$ contain all and only those connections involving root nodes, it is sufficient to zero them out to recover a node-deleted bag. We thus perform the following operations:
    \begin{align*}
        X_{iij} &= \mathbf{0} \cdot X^{(0)}_{iij} \\
        X_{iji} &= \mathbf{0} \cdot X^{(0)}_{iji}
    \end{align*}
    
    [\emph{\textbf{node-marking}}] adds a special `mark' to root nodes only. We implement it by adding one additional dimension to node features and by setting that to $1$ only for root nodes via the bias term:
    \begin{align*}
        X_{iii} &= \hemcp{1}{d}{1}{d}{d+1} X^{(0)}_{iii} + \onehot{d+1} \\
        X_{ijj} &= \hemcp{1}{d}{1}{d}{d+1} X^{(0)}_{ijj} + \mathbf{0}_{d+1} \\
        X_{iij} &= \hemcp{1}{d}{1}{d}{d+1} X^{(0)}_{iij} + \mathbf{0}_{d+1} \\
        X_{iji} &= \hemcp{1}{d}{1}{d}{d+1} X^{(0)}_{iji} + \mathbf{0}_{d+1} \\
        X_{ijk} &= \hemcp{1}{d}{1}{d}{d+1} X^{(0)}_{ijk} + \mathbf{0}_{d+1} 
    \end{align*}
    \noindent where $\onehot{d+1}$ is a (one-hot) $(d+1)$-dimensional vector being $1$ in dimension $d+1$, $0$ elsewhere; $\mathbf{0}_{d+1}$ is the $(d+1)$-dimensional zero vector.
    
    [\emph{\textbf{ego-networks($h$)}}] So far, we have mostly made use of the orbit-partitioning that is induced by the $S_n$ symmetry group, which has allowed us to implement the \emph{node-deletion} and \emph{node-marking} policies with one single \IGN{3} layer. We now show that, in order to implement the \emph{ego-networks} policy, multiple layers are required, as a \IGN{3} effectively needs to perform Breadth-First-Search for each node in the graph to construct $h$-hop neighbourhoods. We illustrate the required steps below, which mainly articulate in (i) the construction of $h$-hop neighbourhoods around nodes; (ii) extraction of egonets from such neighbourhoods. The yet-to-describe \IGN{3} will realise part (i) by storing a `reachability' patterns in an additional channel in $X_{ijj}$: element $\big( X^{d+1}_{ijj} \big)_{i=v, j=w}$ will be set to $1$ if node $w$ is reachable from node $v$ in $h$ hops, $0$ otherwise. In part (ii), these patterns will be utilised, for each subgraph, to nullify the connectivity involving unreachable nodes, thus effectively extracting $h$-hop egonets. We start by describing the layers implementing part (i).
    
    (\emph{Immediate neighbourhood}) For immediate neighbours, the reachability pattern is already (implicitly) stored in $X_{iij}$, as it contains the direct connectivity involving root nodes. We therefore copy this information into an additional channel in $X_{ijj}$. This value will be updated iteratively as we explore higher-order neighborhoods. In the following, we conveniently set $e=d+1$ to ease the notation.
    \begin{align*}
        X_{ijj}^{(1)} &= \hemcp{}{d}{}{d}{e} X^{(0)}_{ijj} + \cp{}{1}{e}{} \broad{}{i,j,j} X^{(0)}_{iij}
    \end{align*}
    
    (\emph{Higher-order neighbourhoods}) We repeat the following steps $(h-1)$ times, and describe the generic $l$-th step. We first broadcast the current reachability pattern into $X_{ijk}$, writing it into the second channel (the first contains the original graph connectivity). Essentially, this operation `propagates' the subgraph-wise reachability pattern copying it row-by-row.
    \begin{align*}
        X_{ijk}^{(l,1)} &= X_{ijk}^{(l-1)} + \cp{e}{}{2}{2} \broad{}{i,*,j} X_{ijj}^{(l-1)}
    \end{align*}
    Having placed the pattern as described above, we perform a logical AND between the first two channels of $X_{ijk}$ and write back the result into the second channel. 
    \begin{align*}
        X_{ijk}^{(l,2)} &= \mlphemcp{1}{2}{2}{2}{e}{\land} X_{ijk}^{(l,1)}
    \end{align*}
    For a specific node $w$ in a given subgraph $v$, this operation spots the neighbours of $w$ that are currently marked in $v$, effectively propagating the reachability information one hop farther. At this point, pooling over rows counts the number of such neighbours: if at least one is reachable, then node $w$ becomes reachable as well, and its corresponding entry must be set to $1$ in the updated reachability pattern. We therefore complete the $l$-th hop step by updating the pattern accordingly and by clipping it to $1$.
    \begin{align*}
        X_{ijj}^{(l,3)} &= X_{ijj}^{(l,2)} + \cp{2}{2}{e}{} \broad{}{i,j,j} \pool{}{i,j} X_{ijk}^{(l,2)} \\
        X_{ijj}^{(l)} &= \glue{\hemcp{}{d}{}{d}{e}}{\mlpcp{e}{e}{e}{e}{\downarrow}} X_{ijj}^{(l,3)}
    \end{align*}
    
    (\emph{Egonet extraction}) We complete the implementation of the policy by leveraging the computed reachability pattern to extract the required egonets. We do this by nullifying the connectivity entries in $X_{ijk}$ for those nodes still unreached, i.e. we zero out row- and column-elements for nodes whose entry in the reachability pattern is $0$. To perform this operation we appropriately broadcast the pattern and use it as an argument into a logical AND:
    \begin{align*}
        X_{ijk}^{(x,1)} &= X_{ijk}^{(h)} + \cp{e}{}{2}{2} \broad{}{i,j,*} X_{ijj}^{(h)} \\
        X_{ijk}^{(x,2)} &= \mlpcp{1}{2}{2}{2}{\land} X_{ijk}^{(x,1)} \\
        X_{ijk}^{(x,3)} &= X_{ijk}^{(x,2)} + \cp{e}{}{3}{3} \broad{}{i,*,j} X_{ijj}^{(x,2)} \\
        X_{ijk}^{(x,4)} &= \mlpcp{1}{3}{3}{3}{\land} X_{ijk}^{(x,3)} \\
        X_{ijk}^{(x)} &= \mlpcp{2}{3}{}{1}{\land} X_{ijk}^{(x,4)}
    \end{align*}
    Finally, we save the reachability pattern in channel $2$ of $X_{iij}$; this information effectively conveys, for each subgraph, the membership of each node to that specific subgraph. At the same time, we bring all other orbit tensors to the original dimension $d$:
    \begin{align*}
        X_{iii} &= \cp{}{d}{}{d} X_{iii}^{(x)} \\
        X_{ijj} &= \cp{}{d}{}{d} X_{ijj}^{(x)} \\
        X_{iij} &= \hemcp{}{1}{}{1}{d} X_{iij}^{(x)} + \hemcp{e}{}{2}{2}{d} \broad{}{i,i,j} X_{ijj}^{(x)} \\
        X_{iji} &= \cp{}{d}{}{d} X_{iji}^{(x)} \\
        X_{ijk} &= \cp{}{d}{}{d} X_{ijk}^{(x)}
    \end{align*}
    
    [\textbf{\emph{Retaining original connectivity}}] In all the derivations above we have overwritten the first channel of orbit representations $X_{iij}, X_{iji}, X_{ijk}$ with the computed subgraph connectivity. However, certain Subgraph GNNs may require to retain the original graph connectivity, see, e.g., \Cref{eq:DSSGNN_orig,eq:NestedGNN}. In that case it is sufficient to first replicate the first channel of the aforementioned orbit representations into another one before altering it to obtain the subgraph connectivity.
\end{proof}

\subsection{Implementing Subgraph GNN layers}
Before proceeding to prove~\Cref{lemma:3IGN_implements_SubgraphNetworks}, we report the equation of a GNN layer in the form due to~\citet{morris2019weisfeiler}, which we assume is the base-encoder of the Subgraph GNN layers considered in the following.
\begin{equation}\label{eq:morris}
    x^{(t+1)}_i = \sigma \big( W_{1,t} \cdot x^{(t)}_i + W_{2,t} \cdot \sum_{j \sim i} x^{(t)}_j \big)
\end{equation}

\begin{proof}[Proof of Lemma~\ref{lemma:3IGN_implements_SubgraphNetworks}]\label{proof:3IGN_implements_SubgraphNetworks}
    
    We will implement the update equations for a generic bag $B \in \{B_1^{(t)}, B_2^{(t)}\}$ represented by $\mathbb{R}^{n^3 \times d}$ tensor $\mathcal{Y}^{(t)} \cong X^{(t)}_{iii} \sqcup X^{(t)}_{ijj} \sqcup X^{(t)}_{iij} \sqcup X^{(t)}_{iji} \sqcup X^{(t)}_{ijk}$. When necessary, we will use an additional subscript to indicate which of the two input bags the tensor is representing, as in $X^{(t)}_{(1), ijk}$. In the following we assume $B_1$ is the bag for graph $G_1$ of $n_1$ nodes, $B_2$ is the bag for graph $G_2$ of $n_2$ nodes.
    
    [\emph{\textbf{DS-GNN}}] When equipped with \citet{morris2019weisfeiler} base-encoder, DS-GNN updates the representation of node $i$ in subgraph $k$ as:
    \begin{align}\label{eq:DSGNN_Morris}
        x^{k,(t+1)}_i &= \sigma \big ( W_{1,t} \cdot x^{k,(t)}_i + W_{2,t} \cdot \sum_{j \sim_k i} x^{k,(t)}_j \big ) 
    \end{align}
    
    ([1] \emph{Message broadcasting}) One first \IGN{3} layer propagates the current node representations in a way to prepare them for the following aggregation. Node representations $X_{ijj}$ are written over $X_{ijk}, X_{iij}$; the ones in $X_{iii}$ over $X_{iji}$: This will allow matching them with the subgraph connectivity stored in the first channel of such tensors.
    \begin{align*}
        X^{(t,1)}_{iij} &= \hemcp{}{d}{}{d}{2d} X^{(t)}_{iij} + \cp{}{}{d+1}{2d} \broad{}{i,i,j} X^{(t)}_{ijj} \\
        X^{(t,1)}_{iji} &= \hemcp{}{d}{}{d}{2d} X^{(t)}_{iji} + \cp{}{}{d+1}{2d} \broad{}{i,*,i} X^{(t)}_{iii} \\
        X^{(t,1)}_{ijk} &= \hemcp{}{d}{}{d}{2d} X^{(t)}_{ijk} + \cp{}{}{d+1}{2d} \broad{}{i,*,j} X^{(t)}_{ijj}
    \end{align*}
    
    ([2] \emph{Message sparsification \& aggregation}) \IGNs{3} only possess \emph{global} pooling as computational primitive, while message-passing requires a form of local aggregation in accordance to the connectivity at hand. For each node, we realise this by first nullifying messages from non-adjacent nodes followed by global summation. We make use of a result by~\citet{yun2019small} to show that the aforementioned nullification, over the two bags in input, can be exactly implemented by a (small) ReLU network memorising a properly assembled dataset. Let us first report the result of interest~\citep[Theorem 3.1]{yun2019small}:
    \begin{theorem}[Theorem 3.1 from~\citet{yun2019small}]\label{thm:memorisation_theirs}
        Consider any datasaset $\{ (x_i, y_i) \}_{i=1}^N$ such that all $x_i$'s are distinct and all $y_i \in [-1, +1]^{d_y}$. If a $3$-layer ReLU-like MLP $f_{\boldsymbol{\theta}}$ satisfies $4 \lfloor \nicefrac{d_1}{4} \rfloor \lfloor \nicefrac{d_2}{(4d_y)} \rfloor \geq N$, then there exists a parameter $\boldsymbol{\theta}$ such that $y_i = f_{\boldsymbol{\theta}}(x_i)$ for all $i \in [N]$.
    \end{theorem}
    \noindent The theorem guarantees the existence of a properly sized ReLU network able to memorise an input dataset satisfying the reported conditions. We note that one of these conditions can, in a sense, be relaxed:
    \begin{proposition}[Memorisation]\label{prop:memorisation_ours}
        Consider any dataset $D = \{ (x_i, y_i) \}_{i=1}^N$ such that all $x_i$'s $\in \mathbb{R}^{d_x}$ are distinct and all $y_i \in \mathbb{R}^{d_y}$. There exists a $3$-layer ReLU-like MLP $\mlp{D}$ such that $y_i = \mlp{D}(x_i)$ for all $i \in [N]$.
    \end{proposition}
    \begin{proof}
        Let $M = \max \{ \max_{i \in [N]} |y_i| , \mathbf{1}_{d_y} \}$, where $\max$ is intended to be applied element-wise. Let $\tilde{D} = \{ (x_i, \tilde{y}_i) \thinspace | \thinspace \tilde{y}_i = y_i \oslash M \}_{i=1}^N$, where $\oslash$ is element-wise division. Dataset $\tilde{D}$ satisfies the assumptions in Theorem~\ref{thm:memorisation_theirs}, as targets $\tilde{y}_i$ are now all necessarily in $[-1, +1]^{d_y}$. Hence, there exists a $3$-layer ReLU MLP $f_{\boldsymbol{\theta}}$ memorising $\tilde{D}$. Let $(W_l, b_l)$ refer to its weight matrix and bias vector at the $l$-th layer, $l = 1, 2, 3$. Let $\bar{W} = \mathrm{diag}(M)$, i.e. a diagonal matrix in $\mathbb{R}^{d_{y} \times d_{y}}$ such that $\bar{W}_{ii} = M_i, i \in [d_y]$. To conclude the proof it is sufficient to construct a $3$-layer ReLU MLP $\mlp{D}$ with parameter stacking $\{(W^D_l, b^D_l)\}_{l=1,2,3}$ such that: $(W^D_1, b^D_1) = (W_1, b_1), (W^D_2, b^D_2) = (W_2, b_2), (W^D_3, b^D_3) = (\bar{W} \cdot W_3, \bar{W} \cdot b_3)$.
    \end{proof}

    We now continue our proof. Formally, we seek to find two MLPs $\mlp{\odot_{iij}}, \mlp{\odot_{ijk}}$ implementing, respectively, functions $f^{\odot}_{iij}, f^{\odot}_{ijk}$ satisfying the following. $f^{\odot}_{iij}$ is such that $\forall a, b \in [n], a \neq b$:
    \begin{align*}
        f^{\odot}_{iij} \big( X^{(t,1)}_{aab} \big) =
            \begin{cases}
                \mathbf{0}_d & \text{if } X^{(t,1), 1}_{aab} = 0, \\
                X^{(t,1), d+1:}_{aab} & \text{otherwise}.
            \end{cases}
    \end{align*}
    Likewise, $f^{\odot}_{ijk}$ is such that: $\forall a, b, c \in [n], a \neq b \neq c$:
    \begin{align*}
        f^{\odot}_{ijk} \big(  X^{(t,1)}_{abc} \big) =
            \begin{cases}
                \mathbf{0}_d & \text{if } X^{(t,1), 1}_{abc} = 0, \\
                X^{(t,1), d+1:}_{abc} & \text{otherwise}.
            \end{cases}
    \end{align*}
    We construct two datasets:
    \begin{align*}
        D_{iij} &= \Big \{ \big( x, f^{\odot}_{iij}(x) \big) \thinspace \big | \thinspace x = X^{(t,1)}_{(1),aab} \enspace \forall a, b \in [n_1], a \neq b \Big \}\ \cup \\
            &\qquad \cup \Big \{ \big( x, f^{\odot}_{iij}(x) \big) \thinspace \big | \thinspace x = X^{(t,1)}_{(2),aab} \enspace \forall a, b \in [n_2], a \neq b \Big \} \\
        D_{ijk} &= \Big \{ \big( x, f^{\odot}_{ijk}(x) \big) \thinspace \big | \thinspace x = X^{(t,1)}_{(1),abc} \enspace \forall a, b, c \in [n_1], a \neq b \neq c \Big \}\ \cup \\
            &\qquad \cup \Big \{ \big( x, f^{\odot}_{ijk}(x) \big) \thinspace \big | \thinspace x = X^{(t,1)}_{(2),abc} \enspace \forall a, b, c \in [n_2], a \neq b \neq c \Big \}
    \end{align*}
    Here, as all targets are the output of a well-defined function, these datasets satisfy, by construction, the hypothesis of Preposition~\ref{prop:memorisation_ours}, which we apply on both. This guarantees the existence of $\mlp{\odot_{iij}}$, $\mlp{\odot_{ijk}}$; their application allows global pooling to effectively recover sparse message aggregation. We also notice that, when updating representations of non-root nodes, roots may be amongst their neighbours, so that it may be needed to additionally sum their representations in $X_{iii}$ to $X_{ijj}$, conditioned on the subgraph connectivity information stored in $X_{iji}$. Accordingly, let us define function $f^{\odot}_{iji}$:
    \begin{align*}
        f^{\odot}_{iji} \big(  X^{(t,1)}_{aba} \big) =
            \begin{cases}
                \mathbf{0}_d & \text{if } X^{(t,1), 1}_{aba} = 0, \\
                X^{(t,1), d+1:}_{aba} & \text{otherwise}.
            \end{cases}
    \end{align*}
    We construct dataset:
    \begin{align*}
        D_{iji} &= \Big \{ \big( x, f^{\odot}_{iji}(x) \big) \thinspace \big | \thinspace x = X^{(t,1)}_{(1),aba} \enspace \forall a, b \in [n_1], a \neq b \Big \}\ \cup \\
            &\qquad \cup \Big \{ \big( x, f^{\odot}_{iji}(x) \big) \thinspace \big | \thinspace x = X^{(t,1)}_{(2),aba} \enspace \forall a, b \in [n_2], a \neq b \Big \}
    \end{align*}
    Invoking Proposition~\ref{prop:memorisation_ours} on $D_{iji}$ guarantees the existence of $\mlp{\odot_{iji}}$ memorising $D_{iji}$. We let the described \IGN{3} implement such networks:
    \begin{align*}
        X^{(t,2)}_{iij} &= \glue{ \cp{}{d}{}{d}}{\mlpcp{}{}{d+1}{}{\odot_{iij}}} X^{(t,1)}_{iij} \\
        X^{(t,2)}_{iji} &= \glue{ \cp{}{d}{}{d}}{\mlpcp{}{}{d+1}{}{\odot_{iji}}} X^{(t,1)}_{iji} \\
        X^{(t,2)}_{ijk} &= \glue{ \cp{}{d}{}{d}}{\mlpcp{}{}{d+1}{}{\odot_{ijk}}} X^{(t,1)}_{ijk}
    \end{align*}
    This last layer completes message aggregation:
    \begin{align}
        X^{(t,3)}_{iii} &= \cp{}{d}{}{d} X^{(t,2)}_{iii} + \cp{d+1}{}{d+1}{} \broad{}{iii} \pool{}{i} X^{(t,2)}_{iij} \label{eq:3IGN_DSGNN_aggr_1} \\
        X^{(t,3)}_{ijj} &= \cp{}{d}{}{d} X^{(t,2)}_{ijj} + \cp{d+1}{}{d+1}{} \broad{}{ijj} X^{(t,2)}_{iji} + \cp{d+1}{}{d+1}{} \broad{}{ijj} \pool{}{ij} X^{(t,2)}_{ijk} \label{eq:3IGN_DSGNN_aggr_2}
    \end{align}
    
    ([3] \emph{Update}) We finally describe the statements implementing linear transformations operated by parameters $W_{1,t}, W_{2,t}$, other than bringing the other orbit representations to dimension $d$:
    \begin{align}
        X^{(t+1)}_{iii} &= \sigma \big(\ \sumglue{W_{1,t}}{W_{2,t}} X^{(t,3)}_{iii}\ \big) \label{eq:3IGN_DSGNN_up_1} \\
        X^{(t+1)}_{ijj} &= \sigma \big(\ \sumglue{W_{1,t}}{W_{2,t}} X^{(t,3)}_{ijj}\ \big) \label{eq:3IGN_DSGNN_up_2} \\
        X^{(t+1)}_{iij} &= \hemcp{}{1}{}{1}{d} X^{(t,3)}_{iij} \nonumber \\
        X^{(t+1)}_{iji} &= \hemcp{}{1}{}{1}{d} X^{(t,3)}_{iji} \nonumber \\
        X^{(t+1)}_{ijk} &= \hemcp{}{1}{}{1}{d} X^{(t,3)}_{ijk} \nonumber
    \end{align}

    [\emph{\textbf{DSS-GNN}}] When equipped with \citet{morris2019weisfeiler} base-encoder, DSS-GNN updates representation of node $i$ in subgraph $k$ as:
    \begin{align}\label{eq:DSSGNN_Morris}
        x^{k,(t+1)}_i &= \sigma \big ( W^{1}_{1,t} \cdot x^{k,(t)}_i + W^{1}_{2,t} \cdot \sum_{j \sim_k i} x^{k,(t)}_j + W^{2}_{1,t} \cdot \sum_{h} x^{h,(t)}_i + W^{2}_{2,t} \cdot \sum_{j \sim i} \sum_{h} x^{h,(t)}_j \big ) 
    \end{align}
    
    ([0] \emph{Cross-bag aggregation}) We start by performing those operations above of the form $\sum_h$: 
    % Commented instructions are to perform connectivity aggregation
    \begin{align*}
        X^{(t,0)}_{iii} &= \glue{I_d}{I_d} X^{(t)}_{iii} + \cp{}{}{d+1}{2d} \broad{}{j,j,j} \pool{}{j} X^{(t)}_{ijj} \\
        % X^{(t,0)}_{iij} &= \glue{\hemcp{1}{1}{1}{1}{2d}}{\hemcp{1}{1}{2}{2}{2d}} X^{(t)}_{iij} + \hemcp{1}{1}{2}{2}{2d} \broad{}{*,j,k} \pool{}{jk} X^{(t)}_{ijk} \\
        % X^{(t,0)}_{iji} &= \glue{\hemcp{1}{1}{1}{1}{2d}}{\hemcp{1}{1}{2}{2}{2d}} X^{(t)}_{iji} + \hemcp{1}{1}{2}{2}{2d} \broad{}{*,j,k} \pool{}{jk} X^{(t)}_{ijk} \\
        X^{(t,0)}_{ijj} &= \cp{}{}{}{d} X^{(t)}_{ijj} + \cp{}{}{d+1}{2d} \broad{}{*,j,j} \pool{}{j} X^{(t)}_{ijj} + \cp{}{}{d+1}{2d} \broad{}{*,i,i} X^{(t)}_{iii} \\
        % X^{(t,0)}_{ijk} &= \hemcp{1}{1}{1}{1}{2d} X^{(t)}_{ijk} + \hemcp{1}{1}{2}{2}{2d} \broad{}{*,j,k} \pool{}{jk} X^{(t)}_{ijk} + \hemcp{1}{1}{2}{2}{2d} \broad{}{*,i,j} X^{(t)}_{iij}
    \end{align*}
    
    ([1] \emph{Message broadcasting}) Similarly as in DS-GNN, we propagate node representations --- and their cross-bag aggregated counterparts --- on those orbits storing (sub)graph connectivity.
    \begin{align*}
        X^{(t,1)}_{iij} &= \hemcp{}{d}{}{d}{3d} X^{(t,0)}_{iij} + \cp{}{}{d+1}{3d} \broad{}{i,i,j} X^{(t,0)}_{ijj} \\
        X^{(t,1)}_{iji} &= \hemcp{}{d}{}{d}{3d} X^{(t,0)}_{iji} + \cp{}{}{d+1}{3d} \broad{}{i,*,i} X^{(t,0)}_{iii} \\
        X^{(t,1)}_{ijk} &= \hemcp{}{d}{}{d}{3d} X^{(t,0)}_{ijk} + \cp{}{}{d+1}{3d} \broad{}{i,*,j} X^{(t,0)}_{ijj}
    \end{align*}
    
    ([2] \emph{Message sparsification \& aggregation}) We now follow the same rationale as for DS-GNN, and construct datasets that allow the invocation of Proposition~\ref{prop:memorisation_ours}. This will guarantee the existence of an MLP that can be applied to retain, for each node, only those messages coming from direct neighbours, according to the subgraph connectivity (stored in channel 1) or the original one (which we assume to be stored in channel 2). Precisely, we would like to memorise the following functions:
    \begin{align*}
        f^{\odot, \sim_i}_{iij} \big( X^{(t,1)}_{aab} \big) &=
            \begin{cases}
                \mathbf{0}_d & \text{if } X^{(t,1), 1}_{aab} = 0, \\
                X^{(t,1), d+1:2d}_{aab} & \text{otherwise}.
            \end{cases}
        &\!\!\!f^{\odot, \sim_i}_{iji} \big(  X^{(t,1)}_{aba} \big) =
            \begin{cases}
                \mathbf{0}_d & \text{if } X^{(t,1), 1}_{aba} = 0, \\
                X^{(t,1), d+1:2d}_{aba} & \text{otherwise}.
            \end{cases} \\
        f^{\odot, \sim_i}_{ijk} \big( X^{(t,1)}_{abc} \big) &=
            \begin{cases}
                \mathbf{0}_d & \text{if } X^{(t,1), 1}_{abc} = 0, \\
                X^{(t,1), d+1:2d}_{abc} & \text{otherwise}.
            \end{cases} \\
        f^{\odot, \sim}_{iij} \big( X^{(t,1)}_{aab} \big) &=
            \begin{cases}
                \mathbf{0}_d & \text{if } X^{(t,1), 2}_{aab} = 0, \\
                X^{(t,1), 2d+1:}_{aab} & \text{otherwise}.
            \end{cases}
        &\!\!\!f^{\odot, \sim}_{iji} \big(  X^{(t,1)}_{aba} \big) =
            \begin{cases}
                \mathbf{0}_d & \text{if } X^{(t,1), 2}_{aba} = 0, \\
                X^{(t,1), 2d+1:}_{aba} & \text{otherwise}.
            \end{cases} \\
        f^{\odot, \sim}_{ijk} \big( X^{(t,1)}_{abc} \big) &=
            \begin{cases}
                \mathbf{0}_d & \text{if } X^{(t,1), 2}_{abc} = 0, \\
                X^{(t,1), 2d+1:}_{abc} & \text{otherwise}.
            \end{cases} \\
    \end{align*}
    \noindent and construct the corresponding datasets:
    \begin{align*}
        D^{\sim_i}_{iij} &= \Big \{ \big( x, f^{\odot, \sim_i}_{iij}(x) \big) \thinspace \big | \thinspace x = X^{(t,1)}_{(1),aab} \enspace \forall a, b \in [n_1], a \neq b \Big \}\ \cup \\
        & \qquad \qquad \cup \Big \{ \big( x, f^{\odot, \sim_i}_{iij}(x) \big) \thinspace \big | \thinspace x = X^{(t,1)}_{(2),aab} \enspace \forall a, b \in [n_2], a \neq b \Big \} \\
        D^{\sim_i}_{iji} &= \Big \{ \big( x, f^{\odot, \sim_i}_{iji}(x) \big) \thinspace \big | \thinspace x = X^{(t,1)}_{(1),aba} \enspace \forall a, b \in [n_1], a \neq b \Big \}\ \cup \\
        & \qquad \qquad \cup \Big \{ \big( x, f^{\odot, \sim_i}_{iij}(x) \big) \thinspace \big | \thinspace x = X^{(t,1)}_{(2),aba} \enspace \forall a, b \in [n_2], a \neq b \Big \} \\
        D^{\sim_i}_{ijk} &= \Big \{ \big( x, f^{\odot, \sim_i}_{ijk}(x) \big) \thinspace \big | \thinspace x = X^{(t,1)}_{(1),abc} \enspace \forall a, b, c \in [n_1], a \neq b \neq c \Big \}\ \cup \\
        & \qquad \qquad \cup \Big \{ \big( x, f^{\odot, \sim_i}_{ijk}(x) \big) \thinspace \big | \thinspace x = X^{(t,1)}_{(2),abc} \enspace \forall a, b, c \in [n_2], a \neq b \neq c \Big \} \\
        D^{\sim}_{iij} &= \Big \{ \big( x, f^{\odot, \sim}_{iij}(x) \big) \thinspace \big | \thinspace x = X^{(t,1)}_{(1),aab} \enspace \forall a, b \in [n_1], a \neq b \Big \}\ \cup \\
        & \qquad \qquad \cup \Big \{ \big( x, f^{\odot, \sim}_{iij}(x) \big) \thinspace \big | \thinspace x = X^{(t,1)}_{(2),aab} \enspace \forall a, b \in [n_2], a \neq b \Big \} \\
        D^{\sim}_{iji} &= \Big \{ \big( x, f^{\odot, \sim}_{iji}(x) \big) \thinspace \big | \thinspace x = X^{(t,1)}_{(1),aba} \enspace \forall a, b \in [n_1], a \neq b \Big \}\ \cup \\
        & \qquad \qquad \cup \Big \{ \big( x, f^{\odot, \sim}_{iij}(x) \big) \thinspace \big | \thinspace x = X^{(t,1)}_{(2),aba} \enspace \forall a, b \in [n_2], a \neq b \Big \} \\
        D^{\sim}_{ijk} &= \Big \{ \big( x, f^{\odot, \sim}_{ijk}(x) \big) \thinspace \big | \thinspace x = X^{(t,1)}_{(1),abc} \enspace \forall a, b, c \in [n_1], a \neq b \neq c \Big \}\ \cup \\
        & \qquad \qquad \cup \Big \{ \big( x, f^{\odot, \sim}_{ijk}(x) \big) \thinspace \big | \thinspace x = X^{(t,1)}_{(2),abc} \enspace \forall a, b, c \in [n_2], a \neq b \neq c \Big \}
    \end{align*}
    These, by Proposition~\ref{prop:memorisation_ours}, are memorised by, respectively, MLPs $\mlp{\odot^{\sim_i}_{iij}}, \mlp{\odot^{\sim_i}_{iji}}, \mlp{\odot^{\sim_i}_{ijk}}, \mlp{\odot^{\sim}_{iij}}, \mlp{\odot^{\sim}_{iji}}, \mlp{\odot^{\sim}_{ijk}}$. We let our \IGN{3} model apply these:
    \begin{align*}
        X^{(t,2)}_{iij} &= \big[ \cp{}{d}{}{d}\  \mlpcp{}{}{d+1}{2d}{\odot^{\sim_i}_{iij}}\  \mlpcp{}{}{2d+1}{}{\odot^{\sim}_{iij}} \big] X^{(t,1)}_{iij} \\
        X^{(t,2)}_{iji} &= \big[ \cp{}{d}{}{d}\  \mlpcp{}{}{d+1}{2d}{\odot^{\sim_i}_{iji}}\  \mlpcp{}{}{2d+1}{}{\odot^{\sim}_{iji}} \big] X^{(t,1)}_{iji} \\
        X^{(t,2)}_{ijk} &= \big[ \cp{}{d}{}{d}\  \mlpcp{}{}{d+1}{2d}{\odot^{\sim_i}_{ijk}}\  \mlpcp{}{}{2d+1}{}{\odot^{\sim}_{ijk}} \big] X^{(t,1)}_{ijk}
    \end{align*}
    It is only left to aggregate messages via global pooling:
    \begin{align*}
        X^{(t,3)}_{iii} &= \hemcp{}{2d}{}{2d}{4d} X^{(t,2)}_{iii} + \cp{d+1}{}{2d+1}{4d}\broad{}{iii} \pool{}{i} X^{(t,2)}_{iij} \\
        X^{(t,3)}_{ijj} &= \hemcp{}{2d}{}{2d}{4d} X^{(t,2)}_{ijj} + \cp{d+1}{}{2d+1}{4d} \broad{}{ijj} X^{(t,2)}_{iji} + \cp{d+1}{}{2d+1}{4d} \broad{}{ijj} \pool{}{ij} X^{(t,2)}_{ijk}
    \end{align*}
    
    ([4] \emph{Update}) We describe the statements implementing the final linear transformations:
    \begin{align*}
        X^{(t+1)}_{iii} &= \sigma \Big([ W^1_{1,t}\ ||\ W^2_{1,t} \ ||\ W^1_{2,t}\ ||\ W^2_{2,t} ]\ X^{(t,3)}_{iii} \Big)\\
        X^{(t+1)}_{ijj} &= \sigma \Big( [ W^1_{1,t}\ ||\ W^2_{1,t} \ ||\ W^1_{2,t}\ ||\ W^2_{2,t} ]\ X^{(t,3)}_{ijj} \Big) \\
        X^{(t+1)}_{iij} &= \hemcp{}{2}{}{2}{d} X^{(t,3)}_{iij} \\
        X^{(t+1)}_{iji} &= \hemcp{}{2}{}{2}{d} X^{(t,3)}_{iji} \\
        X^{(t+1)}_{ijk} &= \hemcp{}{2}{}{2}{d} X^{(t,3)}_{ijk}
    \end{align*}
    
    [\emph{\textbf{GNN-AK-ctx}}] When equipped with \citet{morris2019weisfeiler} base-encoder, GNN-AK-ctx updates representation of node $i$ in subgraph $k$ as:
    \begin{align}
        x^{k,(t,0)}_i &= x^{k,(t)}_i \nonumber \\
        x^{k,(t,l+1)}_i &= \sigma \big ( W_{1,t,l} \cdot x^{k,(t,l)}_i + W_{2,t,l} \cdot \sum_{j \sim_k i} x^{k,(t,l)}_j \big ), \enspace l = 0, \mathellipsis, L - 1 & [S] \label{eq:GNNAK_Morris_S} \\
        x^{k,(t+1)}_i &= x^{i,(t,L)}_i + \sum_j x^{j,(t,L)}_i + \sum_j x^{i,(t,L)}_j & [A] \label{eq:GNNAK_A}
    \end{align}
    
    ([$S$]) In order to implement block [$S$], it is sufficient to repeat steps [1--3] in the DS-GNN derivation $L$ times, i.e. the desired number of message-passing steps. We obtain representations $X^{(t,L)}$.

    ([$A$]) Block [$A$] is implemented as:
    \begin{align*}
        X^{(t+1)}_{iii} &= 3 \cdot X^{(t,L)}_{iii} + \broad{}{i,i,i} \pool{}{i} X^{(t,L)}_{ijj} + \broad{}{j,j,j} \pool{}{j} X^{(t,L)}_{ijj} \\
        X^{(t+1)}_{ijj} &= 3 \cdot \broad{}{*,i,i} \cdot X^{(t,L)}_{iii} + \broad{}{*,i,i} \pool{}{i} X^{(t,L)}_{ijj} + \broad{}{*,j,j} \pool{}{j} X^{(t,L)}_{ijj}
    \end{align*}
    
    In the original paper~\citep{zhao2022from}, the second and third terms in block [$A$] only operate on the nodes which are members of the ego-networks at hand:
    \begin{align*}
        x^{k,(t+1)}_i &= x^{i,(t,L)}_i + \sum_{j \in V^i} x^{j,(t,L)}_i + \sum_{j \in V^i} x^{i,(t,L)}_j & [A]
    \end{align*}
    We show that \IGNs{3} can implement this block formulation as well, by resorting to the same sparsification technique employed in the derivation of DS-GNN. Let us recall that, as already mentioned above, the \emph{ego-networks} policy can store reachability patterns in orbit representation $X_{iij}$: they convey, for each node $j$, its membership to subgraph $i$. This information can be used to sparsify node representations being aggregated by the global pooling operations taking place in the equations above. We start by placing node representations $X_{ijj}$ onto $X_{iij}, X_{iji}$. Let us also replicate reachability patterns into the second channel of $X_{iji}$.
    \begin{align*}
        X^{(t,p)}_{iij} &= \cp{}{d}{}{d} X^{(t,L)}_{iij} + \cp{}{d}{d+1}{2d} \broad{}{i,i,j} X^{(t,L)}_{ijj} \\
        X^{(t,p)}_{iji} &= \cp{}{d}{}{d} X^{(t,L)}_{iji} + \cp{2}{2}{2}{2} \broad{}{i,j,i} X^{(t,L)}_{iij} + \cp{}{d}{d+1}{2d} \broad{}{j,i,j} X^{(t,L)}_{ijj} % mind this last broadcasting
    \end{align*}
    We would like to memorise the following sparsification functions:
    \begin{align*}
        f^{\odot, V^i}_{iij} \!\big( X^{(t,p)}_{aab} \big) &=
            \begin{cases}
                \mathbf{0}_d & \text{if } X^{(t,p), 2}_{aab} = 0, \\
                X^{(t,p), d+1:}_{aab} & \text{otherwise}.
            \end{cases}
         &\!\!\!f^{\odot, V^i}_{iji} \!\big(  X^{(t,p)}_{aba} \big) =
            \begin{cases}
                \mathbf{0}_d & \text{if } X^{(t,p), 2}_{aba} = 0, \\
                X^{(t,p), d+1:}_{aba} & \text{otherwise}.
            \end{cases}
    \end{align*}
    \noindent so we construct the following datasets:
    \begin{align*}
        D^{V^i}_{iij} &= \Big \{ \big( x, f^{\odot, V^i}_{iij}(x) \big) \thinspace \big | \thinspace x = X^{(t,p)}_{(1),aab} \enspace \forall a, b \in [n_1], a \neq b \Big \}\ \cup \\
        & \qquad \qquad \cup \Big \{ \big( x, f^{\odot, V^i}_{iij}(x) \big) \thinspace \big | \thinspace x = X^{(t,p)}_{(2),aab} \enspace \forall a, b \in [n_2], a \neq b \Big \} \\
        D^{V^i}_{iji} &= \Big \{ \big( x, f^{\odot, V^i}_{iji}(x) \big) \thinspace \big | \thinspace x = X^{(t,p)}_{(1),aba} \enspace \forall a, b \in [n_1], a \neq b \Big \}\ \cup \\
        & \qquad \qquad \cup \Big \{ \big( x, f^{\odot, V^i}_{iji}(x) \big) \thinspace \big | \thinspace x = X^{(t,p)}_{(2),aba} \enspace \forall a, b \in [n_2], a \neq b \Big \}
    \end{align*}
    Again, we invoke Proposition~\ref{prop:memorisation_ours}, which guarantees the existence of MLPs $\mlp{\odot^{V^i}_{iij}}, \mlp{\odot^{V^i}_{iji}}$. Let the \IGN{3} model implement them:
    \begin{align*}
        X^{(t,p+1)}_{iij} &= \glue{\cp{}{d}{}{d}}{ \mlpcp{}{}{d+1}{2d}{\odot^{V^i}_{iij}}} X^{(t,p)}_{iij} \\
        X^{(t,p+1)}_{iji} &= \glue{\cp{}{d}{}{d}}{ \mlpcp{}{}{d+1}{2d}{\odot^{V^i}_{iji}}} X^{(t,p)}_{iji}
    \end{align*}
    Then, we perform the last global pooling step to complete the implementation of block [$A$]:
    \begin{align}
        X^{(t+1)}_{iii} &= 3 \cdot \cp{}{d}{}{d} X^{(t,p+1)}_{iii} + \cp{d+1}{2d}{}{d} \broad{}{i,i,i} \pool{}{i} X^{(t,p+1)}_{iij} + \cp{d+1}{2d}{}{d} \broad{}{i,i,i} \pool{}{i} X^{(t,p+1)}_{iji} \label{eq:3IGN_GNNAK_pooling_1} \\
        X^{(t+1)}_{ijj} &= 3 \cdot \cp{}{d}{}{d} \broad{}{*,i,i} X^{(t,p+1)}_{iii} + \cp{d+1}{2d}{}{d} \broad{}{*,i,i} \pool{}{i} X^{(t,p+1)}_{iij} + \cp{d+1}{2d}{}{d} \broad{}{*,i,i} \pool{}{i} X^{(t,p+1)}_{iji} \label{eq:3IGN_GNNAK_pooling_2} \\
        X^{(t+1)}_{iij} &= \hemcp{1}{2}{1}{2}{d} X^{(t,p+1)}_{iij} \nonumber \\
        X^{(t+1)}_{iji} &= \hemcp{1}{1}{1}{1}{d} X^{(t,p+1)}_{iji} \nonumber \\
        X^{(t+1)}_{ijk} &= \hemcp{1}{1}{1}{1}{d} X^{(t,p+1)}_{ijk} \nonumber
    \end{align}
    
    [\emph{\textbf{GNN-AK}}] In the case of [$A$] operating only on those nodes in the ego-networks at hand, it is sufficient to rewrite Equations~\ref{eq:3IGN_GNNAK_pooling_1}, \ref{eq:3IGN_GNNAK_pooling_2} as:
    \begin{align*}
        X^{(t+1)}_{iii} &= 2 \cdot \cp{}{d}{}{d} X^{(t,p+1)}_{iii} + \cp{d+1}{2d}{}{d} \broad{}{i,i,i} \pool{}{i} X^{(t,p+1)}_{iij} \\
        X^{(t+1)}_{ijj} &= 2 \cdot \cp{}{d}{}{d} \broad{}{*,i,i} X^{(t,p+1)}_{iii} + \cp{d+1}{2d}{}{d} \broad{}{*,i,i} \pool{}{i} X^{(t,p+1)}_{iij}
    \end{align*}
    These equations would also implement the more general block [$A$] in~\Cref{eq:GNNAK_A}.
    
    [\emph{\textbf{ID-GNN}}] With~\citet{morris2019weisfeiler} base-encoder, ID-GNN updates node representations as:
    \begin{equation}\label{eq:IDGNN_Morris}
        x^{k, (t+1)}_i = \sigma \big( W_{1,t} x^{k,(t)}_i + W_{2,t} \sum_{j \sim_k i, j \neq k} x^{k,(t)}_j + \mathds{1}_{[k \sim_k i]} \cdot W_{3,t} x^{k,(t)}_k \big)
    \end{equation}
    Message passing is performed according to the same $3$-IGN programme as in DS-GNN, with the only modifications required to Equations~\ref{eq:3IGN_DSGNN_aggr_1},\ref{eq:3IGN_DSGNN_aggr_2},\ref{eq:3IGN_DSGNN_up_1},\ref{eq:3IGN_DSGNN_up_2}. Equations~\ref{eq:3IGN_DSGNN_aggr_1} and \ref{eq:3IGN_DSGNN_aggr_2} are rewritten as:
    \begin{align*}
        X^{(t,3)}_{iii} &= \cp{}{d}{}{d} X^{(t,2)}_{iii} + \cp{}{d}{d+1}{} W_{2,t} \cp{d+1}{}{}{d} \broad{}{iii} \pool{}{i} X^{(t,2)}_{iij} \\
        X^{(t,3)}_{ijj} &= \cp{}{d}{}{d} X^{(t,2)}_{ijj} + \cp{}{d}{d+1}{} W_{3,t} \cp{d+1}{}{}{d} \broad{}{ijj} X^{(t,2)}_{iji} + \cp{}{d}{d+1}{} W_{2,t} \cp{d+1}{}{}{d} \broad{}{ijj} \pool{}{ij} X^{(t,2)}_{ijk}
    \end{align*}
    \noindent whereas Equations~\ref{eq:3IGN_DSGNN_up_1}, \ref{eq:3IGN_DSGNN_up_2} as:
    \begin{align*}
        X^{(t+1)}_{iii} &= \sigma \big(\ \sumglue{W_{1,t}}{I_d} X^{(t,3)}_{iii}\ \big) \\
        X^{(t+1)}_{ijj} &= \sigma \big(\ \sumglue{W_{1,t}}{I_d} X^{(t,3)}_{ijj}\ \big)
    \end{align*}
    
    [\emph{\textbf{NGNN}}] Using a~\citet{morris2019weisfeiler} base-encoder, the update equation for the \emph{inner} siamese GNN in a Nested GNN~\citep{zhang2021nested} exactly match that of Equation~\ref{eq:DSGNN_Morris}. It is therefore sufficient for the \IGN{3} to execute the same programme employed in the DS-GNN derivation.
\end{proof}

\subsection{Upperbounding Subgraph GNNs}

\begin{proof}[Proof of Thereom~\ref{thm:3IGN_upperbounds_SubgraphNetworks}]
    Subgraph GNN $\mathcal{N}_\Theta$ distinguishes $G_1, G_2$ if they are assigned distinct representations, that is: $y_{G_1} = \mathcal{N}_\Theta \big( A_1, X_1 \big) \neq \mathcal{N}_\Theta \big( A_2, X_2 \big) = y_{G_2}$. Naturally, a \IGN{3} instance $\mathcal{M}_{\Omega}$ implementing $\mathcal{N}_\Theta$ on the same pair of graphs would distinguish them as well. We prove the theorem by showing that such an instance exists.
    
    We seek to find a \IGN{3} model $\mathcal{M}_{\Omega}$ in the form of Equation~\ref{eq:3IGN_model} such that:  $\mathcal{M}_{\Omega} \big( A_1, X_1 \big) = \mathcal{N}_\Theta \big( A_1, X_1 \big) = y_{G_1}$ and $\mathcal{M}_{\Omega} \big( A_2, X_2 \big) = \mathcal{N}_\Theta \big( A_2, X_2 \big) = y_{G_2}$. According to Equation~\ref{eq:subgraph_gnn}, $\mathcal{N}_\Theta \big( \cdot \big) = ( \mu \circ \rho \circ \mathcal{S} \circ \pi \big )_{\Theta} ( \cdot )$. We will show how to construct $\mathcal{M}_{\Omega}$ as an appropriate stacking of \IGN{3} layers exactly implementing each of the components $\pi, \mathcal{S}, \rho, \mu$ when applied to graphs $G_1, G_2$. We assume, w.l.o.g., that stacking $\mathcal{S}$ has the form $\mathcal{S} = L^{(T)} \circ L^{(T-1)} \circ \mathellipsis \circ L^{(1)}$, where $L$'s are $\mathcal{N}$-layers.
    
    By the definition of class $\Upsilon$, $\pi$ in $\mathcal{N}_{\Theta}$ is such that $\pi \in \Pi$, thus, by Lemma~\ref{lemma:3IGN_implements_policies}, there exists a stacking of \IGN{3} layers $\mathcal{M}_{\pi}$ implementing $\pi$. At the same time, for each $\mathcal{N}$-layer $L^{(t)}$, Lemma~\ref{lemma:3IGN_implements_SubgraphNetworks} has its hypotheses satisfied, hence there exists a \IGN{3}-stacking $\mathcal{M}^{(t)}$ implementing $L^{(t)}$ on both $G_1, G_2$. We can compose such stacks so that, overall, we have: $\big( \mathcal{M}^{(T)} \circ \mathellipsis \circ \mathcal{M}^{(1)} \circ \mathcal{M_\pi} \big) \cong_{\{ G_1, G_2\}} \big( \mathcal{S} \circ \pi \big)$, $\cong_{\{ G_1, G_2\}}$ denoting \emph{implementation} over set $\{ G_1, G_2 \} \subset \mathcal{G}$.
    
    We are left with implementing blocks $\mu, \rho$. We show this for every Subgraph GNN in $\Upsilon$.
    
    [\emph{\textbf{DS-GNN} \& \textbf{DSS-GNN}}] perform graph readout on each subgraph and then apply a Deep Sets network to these obtained representations:
    \begin{align*}
        x^{k, (T)} &= \sum_{i} x_{i}^{k, (T)} \\
        y_G &= \psi \big( \sum_{k} \phi(x^{k, (T)}) \big) 
    \end{align*}
    
    The following instruction implements subgraph readout as a \texttt{3-} to \texttt{1-} equivariant layer:
    \begin{align*}
        X^{\rho, (1)}_i &= \pool{}{i} X^{(T)}_{ijj} + X^{(T)}_{iii}
    \end{align*}
    
    Transformation $\phi$ is implemented by a stacking of \IGN{1} layers:
    \begin{align}\label{eq:mlp_on_subgraph_readout}
        X^{\rho, (2)}_i &= \mlp{\phi} X^{\rho, (1)}_{i}
    \end{align}
    
    Finally, we let module $h$ in the \IGN{3} model implement summation $\sum_{k}$, and choose MLP $m$ in the \IGN{3} such that $m \equiv \psi$:
    \begin{align}
        x_G &= h \big( X^{\rho, (2)}_{i} \big) = \sum_{i} X^{\rho, (2)}_{i} \label{eq:DSSGNN_final_readout}\\
        y_G &= m \big( x_G \big) = \psi \big( x_G \big) \nonumber
    \end{align}
    
    It is possible for the DeepSets network to implement a late invariant-aggregation strategy, so that $\mu \circ \rho$ is realised as:
    \begin{align*}
        x^{k, (T)} &= \sum_{i} x_{i}^{k, (T)} \\
        x^{k, (T+1)} &= \sigma \big( W^1_{T} x^{k, (T)} + \sum_{h} W^2_{T} x^{h, (T)} \big) \\
        \mathellipsis \\
        x^{k, (T+L)} &= \sigma \big( W^1_{T+L-1} x^{k, (T+L-1)} + \sum_{h} W^2_{T+L-1} x^{h, (T+L-1)} \big) \\
        y_G &= \psi \big( \sum_{k} x^{k, (T+L)} \big) 
    \end{align*}
    
    In this case, it is sufficient to rewrite Equation~\ref{eq:mlp_on_subgraph_readout} as:
    \begin{align*}
        X^{\rho, (1+l)}_i &= \sigma \big( W^1_{l} X^{\rho, (l)}_{i} + W^2_{l} \broad{}{*} \pool{}{} X^{\rho, (l)}_{i} \big)
    \end{align*}
    \noindent with $l$ ranging from $1$ to $L$, so that then~\Cref{eq:DSSGNN_final_readout} becomes:
    \begin{align*}
        x_G &= h \big( X^{\rho, (1+L)}_{i} \big) = \sum_{i} X^{\rho, (1+L)}_{i}
    \end{align*}
    
    [\emph{\textbf{GNN-AK}, \textbf{GNN-AK-ctx}} \& \textbf{ID-GNN}] do not perform subgraph pooling, rather pool the representations of root nodes directly:
    \begin{align*}
        y_G = \mu \big( \sum_h x^{h, (T)}_h \big)
    \end{align*}
    
    Block $h$ implements pooling on roots:
    \begin{align*}
        x_G &= h(\mathcal{Y}^{(T)}) = \sum_{i} X^{(T)}_{iii}
    \end{align*}
    \noindent and it is then sufficient to choose block $m$ such that $m \equiv \mu$.
    
    [\emph{\textbf{NGNN}}], in its most general form, performs $L$ layers of message passing on subgraph pooled representations over the original graph connectivity:
     \begin{align*}
        x^{(T)}_v &= \sum_{w \in V^v} x^{v, (T)}_{w} \\
        x^{(T+1)}_v &= \sigma \big( W^1_{T} x^{(T)}_v + W^2_{T} \sum_{w \sim v} x^{(T)}_w \big) \\
        \mathellipsis \\
        x^{(T+L)}_v &= \sigma \big( W^1_{T+L-1} x^{(T+L-1)}_v + W^2_{T+L-1} \sum_{w \sim v} x^{(T+L-1)}_w \big) \\
        y_G &= \mu \big( \sum_w x^{(T+L)}_w \big)
    \end{align*}
    
   We assume the original graph connectivity has been retained in the third channel of orbit representations $X_{iij}, X_{iji}, X_{ijk}$, while the second channel in $X_{iij}$ hosts reachability patterns (see Section~\ref{app:policy_implementation}, [\textbf{ego-networks($h$)}] and [\emph{\textbf{Retaining original connectivity}}]). First, we pool representations of nodes in each subgraph, excluding those nodes not belonging to the ego-nets. We need to extend the summation only to those nodes belonging to the ego-networks. This information is stored in the reachability pattern in $X_{iij}$, and we make use of this information to mask node representations before aggregating them. First, we place node representations in $X_{ijj}$ over $X_{iij}$:
    \begin{align*}
        X^{\rho, (1)}_{iij} &= \hemcp{}{d}{}{d}{2d} X_{iij} + \cp{}{}{d+1}{2d} \broad{}{iij} X^{(T)}_{ijj}
    \end{align*}
    
    We note that it is needed to memorise the following sparsification function:
    \begin{align*}
        f^{\odot, V^i}_{iij} \big( X^{\rho, (1)}_{aab} \big) &=
            \begin{cases}
                \mathbf{0}_d & \text{if } X^{\rho, (1), 2}_{aab} = 0, \\
                X^{\rho, (1), d+1:}_{aab} & \text{otherwise}.
            \end{cases}
    \end{align*}
    \noindent and construct the following dataset:
    \begin{align*}
        D^{V^i}_{iij} &= \Big \{ \big( x, f^{\odot, V^i}_{iij}(x) \big) \thinspace \big | \thinspace x = X^{\rho, (1)}_{(1),aab} \enspace \forall a, b \in [n_1], a \neq b \Big \}\ \cup \\
        & \qquad \qquad \cup \Big \{ \big( x, f^{\odot, V^i}_{iij}(x) \big) \thinspace \big | \thinspace x = X^{\rho, (1)}_{(2),aab} \enspace \forall a, b \in [n_2], a \neq b \Big \} \\
    \end{align*}
    
    Proposition~\ref{prop:memorisation_ours} can be invoked, guaranteeing the existence of MLP $\mlp{\odot^{V^i}_{iij}}$ memorising such dataset. We let the \IGN{3} implement it:
    \begin{align*}
        X^{\rho, (2)}_{iij} &= \glue{\cp{}{d}{}{d}}{ \mlpcp{}{}{d+1}{2d}{\odot^{V^i}_{iij}}} X^{\rho, (1)}_{iij}  
    \end{align*}
    \noindent and complete the subgraph readout via a global pooling operation:
    \begin{align*}
        X^{\rho, (3)}_{iii} &= \cp{}{d}{}{d} X^{\rho, (3)}_{iii} + \cp{d+1}{}{}{d} \broad{}{iii} \pool{}{i} X^{\rho, (2)}_{iij}
    \end{align*}
    
    At this point it is left to perform message passing on these pooled representations for $L$ steps on the original connectivity. We note that it is sufficient to broadcast these onto $X_{ijj}$ and run the same message passing steps in parallel on each subgraph, using the same original graph connectivity:
    \begin{align*}
        X^{\rho, (4)}_{ijj} &= \broad{}{*,i,i} X^{\rho, (3)}_{iii}
    \end{align*}
    
    Such message-passing steps are implemented with the same programme provided in the Proof of Lemma~\ref{lemma:3IGN_implements_SubgraphNetworks} for DS-GNN, with the only difference being that the sparsification functions are defined based on the third channel of $X_{iij}, X_{iji}, X_{ijk}$:
    \begin{align*}
        f^{\odot}_{iij} \big( X^{\rho, (4+l)}_{aab} \big) =
            \begin{cases}
                \mathbf{0}_d & \text{if } X^{\rho, (4+l), 3}_{aab} = 0, \\
                X^{\rho, (4+l), d+1:}_{aab} & \text{otherwise}.
            \end{cases}
        \\
        f^{\odot}_{ijk} \big(  X^{\rho, (4+l)}_{abc} \big) =
            \begin{cases}
                \mathbf{0}_d & \text{if } X^{\rho, (4+l), 3}_{abc} = 0, \\
                X^{\rho, (4+l), d+1:}_{abc} & \text{otherwise}.
            \end{cases}
        \\
        f^{\odot}_{iji} \big(  X^{\rho, (4+l)}_{aba} \big) =
            \begin{cases}
                \mathbf{0}_d & \text{if } X^{\rho, (4+l), 3}_{aba} = 0, \\
                X^{\rho, (4+l), d+1:}_{aba} & \text{otherwise}.
            \end{cases}
    \end{align*}
    Datasets to be memorised are defined accordingly. This construction is repeated $L$ times. Afterwards, blocks $h$ and $m$ in the \IGN{3} model pool the root representations and apply MLP $\mu$ on the obtained embedding. These are implemented as shown above for GNN-AK, GNN-AK-ctx, ID-GNN. The proof concludes.
\end{proof}

\begin{proof}[Proof of Corollary~\ref{cor:upperbound}]~\label{proof:upperbound}
    We prove the corollary by contradiction. Suppose there exist non-isomorphic but \WL{3}-equivalent graphs $G_1, G_2$ distinguished by instance $\mathcal{N}_\Theta$ of $\mathcal{N} \in \Upsilon$. That is, $\mathcal{N}_\Theta \big( G_1 \big) \neq \mathcal{N}_\Theta \big( G_2 \big)$. In view of Theorem~\ref{thm:3IGN_upperbounds_SubgraphNetworks}, there must exists a \IGN{3} instance $\mathcal{M}_\Omega$ such that $\mathcal{M}_\Omega \big( G_1 \big) \neq \mathcal{M}_\Omega \big( G_2 \big)$. We note that the expressive power of \IGNs{k} has been fully characterised by \citet{azizian2020characterizing,geerts2020expressive}. In particular, let us report~\citet[Theorem 2]{geerts2020expressive}:
    \begin{theorem}[Expressive power of \IGNs{k}, Theorem $2$ in \citet{geerts2020expressive}]\label{thm:kIGN_upperbound}
        For any two graphs $G_1$ and $G_2$, if \WL{k} does not distinguish $G_1, G_2$ then any \IGN{k} does not distinguish them either, i.e., it assigns $G_1, G_2$ the same (tensorial) representations.
    \end{theorem}
    This theorem equivalently asserts that if there exists a \IGN{k} distinguishing $G_1, G_2$, then these two graphs must be distinguished by the \WL{k} algorithm. Thus, given the existence of \IGN{3} model $\mathcal{M}_\Omega$, the theorem ensures us that \WL{3} distinguishes graphs $G_1, G_2$, against our hypothesis.
\end{proof}

Let us conclude this section by reporting the following
\begin{remark}
    Any Subgraph Network $\mathcal{N} \in \Upsilon$ equipped with policy $\pi_{\mathrm{EGO}(h)}$ (or $\pi_{\mathrm{EGO+}(h)}$) is at most as expressive as \WL{3}, for \emph{any} $h > 0$.
\end{remark}
In other words, given the results proved above, deeper ego-networks may increase the expressive power of a model, but not in a way to exceed that of \WL{3}.

\section{Illustrated comparison of Subgraph GNNs}\label{app:figure}

\begin{figure}[t]
    \centering%\vspace{-1mm}
    \includegraphics[width=\linewidth]{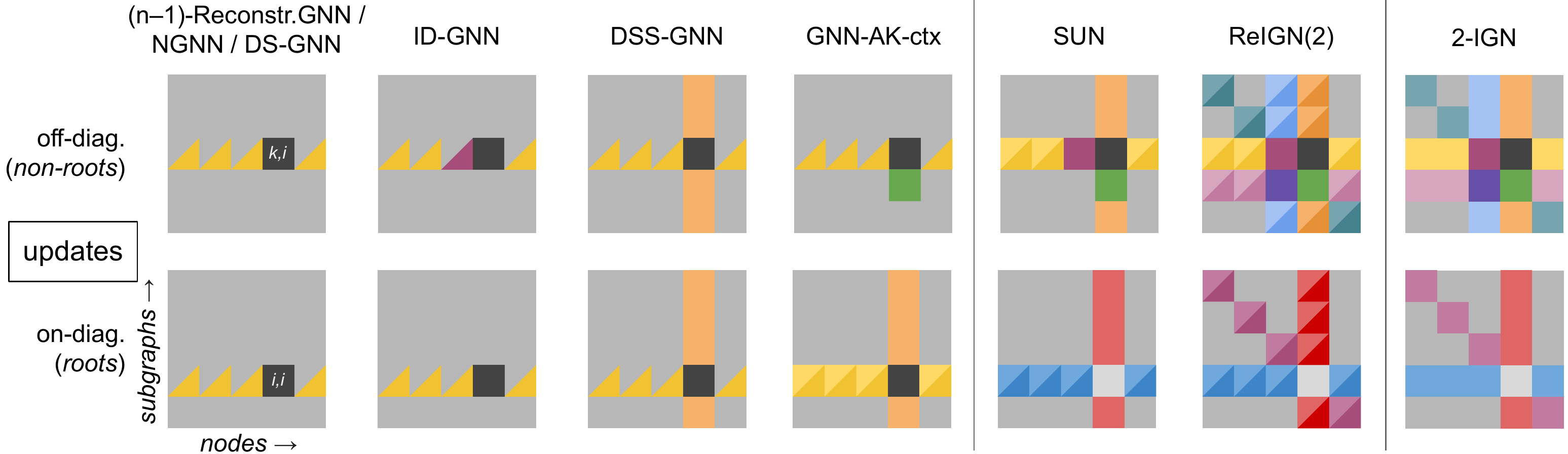}%\vspace{-2mm}
    \caption{A comparison of aggregation and update rules in Subgraph GNNs, illustrated on an $n\times n$ matrix holding $n$ subgraphs with $n$ node features. Top row shows off-diagonal updates, bottom row shows diagonal (root node) updates. Each colour represents a different parameter. Full squares represent global sum pooling; triangles represent local pooling. Two triangles represent both local and global pooling. }\vspace{-4mm}
    \label{fig:ops-app}
\end{figure}

In this section we explain in detail \Cref{fig:ops-app} (corresponding to \Cref{fig:ops} in the main paper) by linking the coloured updates of each Subgraph GNN to its corresponding formulation in \Cref{app:subgraphs}. We consider grids of $n$ subgraphs with $n$ nodes.
The figure shows the aggregation and update rules in Subgraph GNNs for both diagonal $(i,i)$ and off-diagonal $(k,i)$ entries, which correspond respectively to updates of root and non-root nodes. Each color represents a different parameter. We use squares to indicate global pooling and triangles for local pooling. 

\textbf{Reconstruction GNN / NGNN / DS-GNN.} These methods do not distinguish between root and non-root nodes, effectively sharing the parameters between the two (same yellow colour). The representation of a node in a subgraph is obtained via message passing and aggregation within the subgraph, thus, locally.

\textbf{ID-GNN.} ID-GNN performs message passing on each subgraph but distinguishes messages coming from the root (purple instead of yellow), resulting in an additional parameter for non-root updates.

\textbf{DSS-GNN.} DSS allows information sharing between subgraphs. Indeed, it does not only perform message passing within each subgraph (yellow), but also on the aggregated adjacency matrix. The message passing on the aggregated adjacency matrix uses the original connectivity and it is therefore still local (triangle, yellow). However, it uses node representations obtained by aggregating node representations globally across subgraphs (orange).

\textbf{GNN-AK-ctx.} GNN-AK-ctx distinguishes between root and non-root updates. Non-root node are updated by first copying the diagonal representation of the corresponding node (green) and then performing message passing locally (yellow). Root nodes are updated first by performing a local message passing (yellow), and then by also considering the subgraph readout (yellow squares) and the aggregated representation of the node across subgraphs (orange).

\textbf{SUN.} \camera{\SUN{}} distinguishes root and non-root updates. Each non-root node is updated with: (1) the representation of the node at the previous iteration (black), (2) the representation of the root of the subgraph in which the current node is located (purple), (3) the representation of the node in the subgraph where it is root (green), (4) the readout on the subgraph (yellow squares), (5) message passing on the subgraph (yellow triangles), (6) message passing on the aggregated connectivity (yellow triangles and orange squares). For root nodes many of these terms collapse and the node is updated by only considering (1) (bright green), (4) (blue squares), (5) (blue triangles), (6) (blue triangles and red squares).

\textbf{2-IGN and \camera{ReIGN(2)}.}
The node representations are updated according to \Cref{eq:2IGN}. In the \camera{\IGN{2}} case, the operations are all global (squares), \camera{whilst} for \camera{\ReIGN{2}} each aggregation can also be performed locally (triangles), as prescribed by the connectivity of the subgraph or by the original connectivity. 

\section{Proofs for Section \ref{sec:space} -- Subgraph GNNs and ReIGN(2)}\label{app:space}
\subsection{ReIGN(2) expansion of aggregation terms}

We report in Table~\ref{tab:expansion} the expansion rules which allow to derive the \ReIGN{2} equations from the \IGN{2} ones (Equation~\ref{eq:2IGN}), which we copy here below for convenience.
\begin{align*}
    x^{i,(t+1)}_{i} \!\!\!&=\! \upsilon_{\theta_1} \big ( x^{i,(t)}_{i}\!, \agg_{j} x^{j,(t)}_{j}\!, \agg_{j \neq i} x^{i,(t)}_{j}\!, \agg_{h \neq i} x^{h,(t)}_{i}\!, \agg_{h \neq j}
    x^{h,(t)}_{j} \big ) \label{eq:2IGN} \\
    x^{k,(t+1)}_{i} \!\!\!&=\! \upsilon_{\theta_2} \big ( x^{k,(t)}_{i}\!, x^{i,(t)}_{k}\!, \agg_{h \neq j} x^{h,(t)}_{j}\!, \agg_{h \neq i} x^{h,(t)}_{i}\!, \agg_{j \neq k} x^{k,(t)}_{j}\!, \agg_{j \neq i} x^{i,(t)}_{j}\!, \agg_{h \neq k} x^{h,(t)}_{k}\!, x^{k,(t)}_{k}\!, x^{i,(t)}_{i}\!, \agg_j x^{j,(t)}_{j} \big ) \nonumber
\end{align*}

In Table~\ref{tab:expansion} we assign each term an identifier (id.) which will allow us to easily refer to specific terms in the proofs we report below. We additionally provide an interpretation for each of the three expanded terms in the last column. Global and local `vertical' aggregations are dubbed `needle's in analogy with what the authors in~\citet[Definition~5]{bevilacqua2022equivariant} define as `needle' colours in their DSS- Weisfeiler-Leman variant.

\begin{table}[ht]
    \scriptsize
    \centering
    \caption{\ReIGN{2} expansion rules. For each of the \texttt{2-IGN} global aggregation terms in Equation~\ref{eq:2IGN}, \ReIGN{2} additionally considers two more \emph{local} aggregation terms, which sparsify the aggregation to only include factors adjacent according to the subgraph or original graph connectivities.}
    \label{tab:expansion}
    \begin{tabular}{ crccp{0.34\textwidth} }
    \toprule
        target & id. & \IGN{2} term & \ReIGN{2} expansion & interpretation\\
        \midrule
        \midrule
        % on-diagonal
        \multirow{4}{*}{$x^{i}_{i}$} &
        \#1.on &
        $\agg_{j}  x^{j}_{j}$ &
            $\big [ \sum_{j / j \sim_i i / j \sim i} x^{j}_{j} \big ]$ &
            root-readout / root-msg for $i$ (on $i$) \\
        % \midrule
        \cline{2-5} &
        \#2.on &
        $\agg_{j \neq i} x^i_j$ &
            $\big [ \sum_{j \neq i / j \sim_i i / j \sim i} x^i_j \big ]$ &
            $i$-readout / $i$-msg for $i$ (on $i$) \\
        % \midrule
        \cline{2-5} &
        \#3.on &
        $\agg_{h \neq i} x^h_i$ &
            $\big [ \sum_{h \neq i / h \sim_i i / h \sim i} x^h_i \big ]$ &
            $i$-`needle' / local $i$-`needle' for $i$ (on $i$) \\
        % \midrule
        \cline{2-5} &
        \#4.on &
        $\agg_{h \neq j} x^h_j$&
            $\big [ \sum_{h / h \sim_i i / h \sim i} \sum_{j / j \sim_h h / j \sim h} x^h_j \big ]$ &
            non-root-readout / joint local needle for $i$ (on $i$) and msg for $h$ (on $h$) \\
        \midrule
        \midrule
        % off-diagonal
        \multirow{6}{*}{$x^k_i$} &
        \#1.off &
        $\agg_{h \neq j} x^h_j$ &
            $\big [ \sum_{h / h \sim_k i / h \sim i} \sum_{j / j \sim_h i / j \sim i} x^h_j \big ]$  &
            non-root-readout / joint local needle and msg for $i$ (on $k$ and $h$) \\
        % \midrule
        \cline{2-5} &
        \#2.off &
        $\agg_{h \neq i} x^h_i$ &
            $\big [ \sum_{h \neq i / h \sim_k i / h \sim i} x^h_i \big ]$ & 
            $i$-`needle' / local $i$-`needle' for $i$ (on $k$) \\
        % \midrule
        \cline{2-5} &
        \#3.off &
        $\agg_{j \neq k} x^k_j$ &
            $\big [ \sum_{j \neq k / j \sim_k i / j \sim i} x^k_j \big ]$ &
            $k$-readout / $k$-msg for $i$ (on $k$) \\
        % \midrule
        \cline{2-5} &
        \#4.off &
        $\agg_{h \neq i} x^i_h$ &
            $\big [ \sum_{h \neq i / h \sim_k i / h \sim i} x^i_h \big ]$ &
            $i$-readout / $i$-msg for $i$ (on $k$) \\
        % \midrule
        \cline{2-5} &
        \#5.off &
        $\agg_{h \neq k} x^h_k$ &
            $\big [ \sum_{h \neq k / h \sim_k i / h \sim i} x^h_k \big ]$ &
            $k$-`needle' / local $k$-`needle' for $i$ (on $k$) \\
        % \midrule
        \cline{2-5} &
        \#6.off &
        $\agg_j x^j_j$ &
            $\big [ \sum_{j / j \sim_k i / j \sim i} x^j_j \big ]$ &
            root-readout / root-msg for $i$ (on $k$) \\
        \bottomrule
    \end{tabular}
\end{table}

Interestingly, expansions in~\Cref{tab:expansion} could be extended to also include global summations extending only over subgraph nodes as already proposed in~\citet{zhang2021nested,zhao2022from}. For example, term [\#3.off] could also include summation $\sum_{j \in V^{k}} x^{k}_{j}$, where $V^{k}$ is the vertex set of subgraph $k$.
As a last note, we remark how all subgraph-local aggregations updating target $x^{k}_{i}$ (second expansions) consider the connectivity of subgraph $k$. We believe it could be possible to extend \ReIGN{2} to also make use of a different subgraph connectivity, e.g. to define a variant of term [\#2.off] which includes expansion $\sum_{h \sim_i i} x^h_i $. We defer these enquiries to future works.

\subsection{Proofs for Section~\ref{sec:reign}}

\begin{proof}[Proof of Theorem~\ref{thm:ReIGN_implements_SubgraphNetworks}]\label{proof:ReIGN_implements_SubgraphNetworks}

    Any Subgraph GNN in class $\Upsilon$ has the following form:
    \begin{align*}
        \mathcal{N} = \big( \mu \circ \rho \circ \mathcal{S} \circ \pi \big)
    \end{align*}
    \noindent with $\mu$ any MLP, $\rho$ a permutation invariant pooling function, $\pi \in \Pi$ and $\mathcal{S}$ a stacking $\mathcal{S} = L_{T} \circ \mathellipsis L_{1}$. Any \ReIGN{2} model has exactly the same form, with the only (important) difference that layers $L_{t}$'s in the stacking are \ReIGN{2} layers. Therefore, the theorem is proved by showing that, for any $\mathcal{N} \in \Upsilon$, the $\mathcal{N}$-layer equations can be implemented by an appropriate \ReIGN{2} layer stacking. This effort consists in describing (a series of) linear functions $\upsilon_1, \upsilon_2$ as in Equation~\ref{eq:2IGN} implementing the layer equation for model $\mathcal{N} \in \Upsilon$. In practice, this will involve specifying which linear transformation $W$ is applied to each of the terms in Equation~\ref{eq:2IGN} after expanding each summation according to the rules in Table~\ref{tab:expansion}. For convenience, we will directly omit terms assigned a `nullifying' linear transformation $\mathbf{0}$.
    
    [\emph{\textbf{DS-GNN}}] has its layer equations in the form of Equation~\ref{eq:DSGNN_Morris}. These are recovered with one \ReIGN{2} layer by linearly transforming the second expansion for aggregated terms [\#2.on] and [\#3.off], and by sharing parameters between off- and on-diagonal updates:
    \begin{align*}
        x^{i, (t+1)}_i &= \sigma \Big ( W_{1,t} x^{i, (t)}_i + W_{2,t} \sum_{j \sim_i i} x^{i, (t)}_j \Big ) \\
        x^{k, (t+1)}_i &= \sigma \Big ( W_{1,t} x^{k, (t)}_i + W_{2,t} \sum_{j \sim_k i} x^{k, (t)}_j \Big )
    \end{align*}

    [\emph{\textbf{DSS-GNN}}] has its update rule in the form of Equation~\ref{eq:DSSGNN_Morris}. It is needed to stack two linear \ReIGN{2} layers to recover these. The first layer computes message passing on each subgraph as in DS-GNN and cross-bag aggregation of features. We expand the hidden dimension to $2d$ so that the first half stores the result from the former, the second that of the latter. We `utilise' terms [\#2.on] and [\#3.off] in their second expansion for the message-passing operation and terms [\#3.on] and [\#2.off] in their global version for cross-bag aggregation:
    \begin{align*}
        x^{i, (t, 1)}_i &= \glue{W^1_{1,t}}{I_d} x^{i, (t)}_i + \cp{}{}{}{d} W^1_{2,t} \sum_{j \sim_i i} x^{i, (t)}_j + \cp{}{}{d+1}{2d} \sum_{h \neq i} x^{h, (t)}_i \\
        x^{k, (t, 1)}_i &= \cp{}{}{}{d} W^1_{1,t} x^{k, (t)}_i + \cp{}{}{}{d} W^1_{2,t} \sum_{j \sim_k i} x^{k, (t)}_j + \cp{}{}{d+1}{2d} \sum_{h \neq i} x^{h, (t)}_i + \cp{}{}{d+1}{2d} x^{i, (t)}_i
    \end{align*}
    
    The second layer reduces the dimensionality back to $d$ and completes the implementation by performing message-passing of the cross-bag-aggregated representations over the original graph. This is realised by employing terms [\#2.on] and [\#3.off] in their third expansion.
    \begin{align*}
        x^{i, (t+1)}_i &= \sigma \Big ( \sumglue{I_d}{W^2_{1,t}} x^{i, (t,1)}_i + W^2_{2,t} \cp{d+1}{2d}{}{d} \sum_{j \sim i} x^{i, (t,1)}_j \Big ) \\
        x^{k, (t+1)}_i &= \sigma \Big ( \sumglue{I_d}{W^2_{1,t}} x^{k, (t,1)}_i + W^2_{2,t} \cp{d+1}{2d}{}{d} \sum_{j \sim i} x^{k, (t,1)}_j \Big )
    \end{align*}
    
    [\emph{\textbf{GNN-AK-ctx} \& \textbf{GNN-AK}}] Equations~\ref{eq:GNNAK_Morris_S} and~\ref{eq:GNNAK_A} describe the update equations for these models. First, block [S] performs independent message-passing on each subgraph for $L$ steps. The $l$-th update, $l=1,\mathellipsis,(L-1)$ is implemented as for DS-GNN, that is by employing terms [\#2.on] and [\#3.off] as follows:
    \begin{align*}
        x^{i, (t,l+1)}_i &= \sigma \Big ( W_{1,t,l} x^{i, (t,l)}_i + W_{2,t,l} \sum_{j \sim_i i} x^{i, (t,l)}_j \Big ) \\
        x^{k, (t,l+1)}_i &= \sigma \Big ( W_{1,t,l} x^{k, (t,l)}_i + W_{2,t,l} \sum_{j \sim_k i} x^{k, (t,l)}_j \Big )
    \end{align*}
    Block [A] in GNN-AK-ctx is implemented by one \ReIGN{2} layer via aggregated terms [\#2.on], [\#3.on], [\#2.off], [\#4.off] all in their global version:
    \begin{align*}
        x^{i, (t+1)}_i &= 3 \cdot I_d\ x^{i, (t,L)}_i + \sum_{h \neq i} x^{h, (t,L)}_i + \sum_{j \neq i} x^{i, (t,L)}_j \\
        x^{k, (t+1)}_i &= 3 \cdot I_d\ x^{i, (t,L)}_i + \sum_{h \neq i} x^{h, (t,L)}_i + \sum_{j \neq i} x^{i, (t,L)}_j
    \end{align*}
    In the case of GNN-AK, Block [A] is implemented more simply as 
    \begin{align*}
        x^{i, (t+1)}_i &= 2 \cdot I_d\ x^{i, (t,L)}_i + \sum_{j \neq i} x^{i, (t,L)}_j \\
        x^{k, (t+1)}_i &= 2 \cdot I_d\ x^{i, (t,L)}_i + \sum_{j \neq i} x^{i, (t,L)}_j
    \end{align*}
    \noindent where terms [\#3.on], [\#2.off] are nullified.
    
    [\emph{\textbf{ID-GNN}}] Implements the update rule in Equation~\ref{eq:IDGNN_Morris}, which we report here for convenience:
    \begin{equation*}
        x^{k, (t+1)}_i = \sigma \big( W_{1,t} x^{k,(t)}_i + W_{2,t} \sum_{j \sim_k i, j \neq k} x^{k,(t)}_j + \mathds{1}_{[k \sim_k i]} \cdot W_{3,t} x^{k,(t)}_k \big)
    \end{equation*}
    We observe that we can rewrite this equation as follows:
    \begin{equation}\label{eq:IDGNN_v2}
        x^{k, (t+1)}_i = \sigma \big( W_{1,t} x^{k,(t)}_i + \sum_{j \sim_k i} W_{(\mathds{1}_{[j=k]}+2),t} x^{k,(t)}_j \big)
    \end{equation}
    \noindent where $(\mathds{1}_{[j=k]}+2) = 3$ if $j=k$, $2$ otherwise. That is: (i) we do not explicitly exclude the root node from the set of $i$'s neighbours; (ii) messages are computed via $W_{3,t}$ if from root nodes, via $W_{2,t}$ otherwise. This update equation can be implemented by two \ReIGN{2} layers. Similarly as in the DSS-GNN derivation, the first layer expands the hidden dimension to $2d$. The first $d$ channels will store node representations transformed according to $W_{1,t}$, whilst the remaining $d+1$ to $2d$ channels will store node representations transformed according to $W_{3,t}, W_{2,t}$ for, respectively, on- and off-diagonal terms, namely root and non-root nodes:
    \begin{align*}
        x^{i, (t, 1)}_i &= \glue{W_{1,t}}{W_{3,t}} x^{i, (t)}_i \\
        x^{k, (t, 1)}_i &= \glue{W_{1,t}}{W_{2,t}} x^{k, (t)}_i
    \end{align*}
    The second \ReIGN{2} layer effectively completes implementing message-passing and brings back the dimension to $d$. It combines and aggregates the resulting transformations via terms [\#2.on] and [\#3.off] in their second expansion as follows:
    \begin{align*}
        x^{i, (t+1)}_i &= \sigma \Big( \cp{}{d}{}{d} x^{i, (t,1)}_i + \cp{d+1}{2d}{}{d} \sum_{j \sim_i i} x^{i, (t,1)}_j \Big) \\
        x^{k, (t+1)}_i &= \sigma \Big( \cp{}{d}{}{d} x^{k, (t,1)}_i + \cp{d+1}{2d}{}{d} \sum_{j \sim_k i} x^{k, (t,1)}_j \Big)
    \end{align*}
    
     [\emph{\textbf{NGNN}}] Computes independent subgraph-wise message passing as in DS-GNN: the same \ReIGN{2} layer implements this one.
\end{proof}

\begin{proof}[Proof of Proposition~\ref{prop:3IGN_implements_ReIGN}]
    \ReIGN{2} model $\mathcal{R}_{\rho,\Theta,\pi}$ distinguishes $G_1, G_2$ if they are assigned distinct representations, that is: $y_{G_1} = \mathcal{R}_{\rho,\Theta,\pi} \big( A_1, X_1 \big) \neq \mathcal{R}_{\rho,\Theta,\pi} \big( A_2, X_2 \big) = y_{G_2}$. A \IGN{3} instance $\mathcal{M}_{\Omega}$ implementing $\mathcal{R}_{\rho,\Theta,\pi}$ on the same pair of graphs would distinguish the same graphs as well. We will show the existence of such a model instance. 
    
    We seek to find a \IGN{3} model $\mathcal{M}_{\Omega}$ in the form of Equation~\ref{eq:3IGN_model} such that:  $\mathcal{M}_{\Omega} \big( A_1, X_1 \big) = \mathcal{R}_{\rho,\Theta,\pi} \big( A_1, X_1 \big) = y_{G_1}$ and $\mathcal{M}_{\Omega} \big( A_2, X_2 \big) = \mathcal{R}_{\rho,\Theta,\pi} \big( A_2, X_2 \big) = y_{G_2}$, where $\mathcal{R}_{\rho,\Theta,\pi} \big( \cdot \big) = ( \mu \circ \rho \circ \mathcal{S} \circ \pi \big )_{\Theta} ( \cdot )$. From the hypotheses of the theorem, block $\rho$ is \IGN{3}-computable, so there must exist a \IGN{3} layer stacking $\mathcal{M}_{\rho}$ such that $\mathcal{M}_{\rho} \cong \rho$. At the same time, $\pi \in \Pi$ so Lemma~\ref{lemma:3IGN_implements_policies} ensures the existence of a \IGN{3} stacking $\mathcal{M}_{\pi}$ such that $\mathcal{M}_{\pi} \cong \pi$. It is left to show the existence of a \IGN{3} stacking $\mathcal{M_S}$ implementing the \ReIGN{2} stacking $\mathcal{S}$ on the same pair of graphs. Without loss of generality, it is sufficient to show the existence of a \IGN{3} stacking $\mathcal{M}_L$ implementing one single \ReIGN{2} intermediate layer $L$ defined as per Equations~\ref{eq:2IGN} and the aggregated term expansions in Table~\ref{tab:expansion}. We will then show how to construct $\mathcal{M}_L$ for a generic step $t$ in the following. In order to more explicitly reflect the index notation used in Table~\ref{tab:expansion} ($i$ refers to nodes, $k$ to subgraphs) we will refer to the \IGN{3} orbit representations with different subscripts: $\mathcal{Y} \cong X_{iii} \sqcup X_{iij} \sqcup X_{kik} \sqcup X_{kii} \sqcup X_{kij}$, that is we rename $X_{iji}$ as $X_{kik}$, $X_{ijj}$ as $X_{kii}$, $X_{ijk}$ as $X_{kij}$.
    
    We note that the summation of all non-aggregated and globally aggregated terms is recovered by just one \IGN{3} layer, including their linear transformations. In the yet-to-construct stacking, the first layer performs this computation and stores the result into $d$-auxiliary channels. The layer also replicates the current representations $x^{k,(t),i}, x^{i,(t),i}$ in the first $d$ channels. The implementation of local aggregations, i.e. second and third expansions in Table~\ref{tab:expansion}, will require a larger number of layers: this aforementioned input replication allows them to operate on the original representations even after having implemented non-aggregated globally aggregated terms, as required. The result of their computation will then be used to update the intermediate term in channels $d+1$ to $2d$, as we shall see next. Let us start by describing the first \IGN{3} layer, constructed as follows:
    \begin{align*}
        X^{(t,1)}_{iii} &= \glue{I_d}{\theta^{(i,i)}_{1,t}} X^{(t)}_{iii} + \sum_{o_1 \in O_1} \cp{}{}{d+1}{2d} o_1 \\
        X^{(t,1)}_{kii} &= \glue{I_d}{\theta^{(k,i)}_{2,t}} X^{(t)}_{kii} + \sum_{o_2 \in O_2} \cp{}{}{d+1}{2d} o_2
    \end{align*}
    \noindent with set $O_1$ being:
    \begin{align*}
        O_1\!\! &=\!\! \Big \{ 
            \theta^{(j,j)}_{1,t} \broad{}{*i,*i,*i} \pool{}{} X^{(t)}_{iii},
            \theta^{(i,j)}_{1,t} \broad{}{i,i,i} \pool{}{i} X^{(t)}_{ijj},
            \theta^{(h,i)}_{1,t} \broad{}{j,j,j} \pool{}{j} X^{(t)}_{ijj},
            \theta^{(h,j)}_{1,t} \broad{}{*i,*i,*i} \pool{}{} X^{(t)}_{ijj}
        \Big \}
    \end{align*}
    \noindent and set $O_2$ being:
    \begin{align*}
        O_2\!\! &=\!\! \Big \{ 
            \theta^{(i,k)}_{2,t} \broad{}{i,k,k} X^{(t)}_{kii},
            \theta^{(h,j)}_{2,t} \broad{}{*k,*i,*i} \pool{}{} X^{(t)}_{kii},
            \theta^{(h,i)}_{2,t} \broad{}{*,i,i} \pool{}{i} X^{(t)}_{kii},
            \theta^{(k,j)}_{2,t} \broad{}{k,*i,*i} \pool{}{k} X^{(t)}_{kii} \\
    &\qquad \theta^{(i,j)}_{2,t} \broad{}{*,k,k} \pool{}{k} X^{(t)}_{kii},
            \theta^{(h,k)}_{2,t} \broad{}{i,*k,*k} \pool{}{i} X^{(t)}_{kii},
            \theta^{(k,k)}_{2,t} \broad{}{i,*k,*k} X^{(t)}_{iii},
            \theta^{(i,i)}_{2,t} \broad{}{*,i,i} X^{(t)}_{iii}, \\
    &\qquad\quad
            \theta^{(j,j)}_{2,t} \broad{}{*k,*i,*i} \pool{}{} X^{(t)}_{iii}
        \Big \}
    \end{align*}
    
    We now focus on the following layers. The $l$-th local aggregation in Table~\ref{tab:expansion} (second and third expansions) can be obtained by a \IGN{3} layer stacking implementing the steps of: (1) \emph{Message broadcasting}, (2) \emph{Message sparsification}, (3) \emph{Message aggregation}, (4) \emph{Update}, similarly to what already shown for the Proof of Lemma~\ref{lemma:3IGN_implements_SubgraphNetworks}. More in detail, the underlying construction will be such such that: first, messages are placed on the third axis of the cubed tensor on which the \IGN{3} operates (broadcasting, 1); they are then sparsified consistently with the (sub)graph connectivity (sparsification, 2); aggregated via pooling operations on the same axis (aggregation, 3); finally, linearly transformed with their specific linear operator and used to update the intermediate representation(s) in channels $d+1$ to $2d$ (update, 4). We note that for each local aggregation term, these steps essentially differ in the way messages are propagated (1) and in the specific linear transformations applied (4), while the same computation is shared for the sparsification (2) and aggregation (3) steps. Thus, we deem it convenient to first describe these steps and then show how specific choices for (1), (4) recover each desired term.
     
    Here, we assume that $\mathcal{M}_\pi$ writes in $X_{kij}, i \neq j$, the connectivity between nodes $i,j$ in subgraph $k$ (first channel) as well as that prescribed by the original graph connectivity (second channel) --- once more, see discussion [\emph{\textbf{Retaining original connectivity}}] in the Proof of~\Cref{lemma:3IGN_implements_policies}. In order to implement this step on the input graph pair, it is sufficient to have \IGN{3} layers applying an MLP which sparsifies messages according to the aforementioned connectivities and then to aggregate the sparsified messages by global summation, effectively realising \emph{both} the second and third expansion for terms in Table~\ref{tab:expansion}. Step (2) is realised as:
    \begin{align}
        X^{(t,l,\mathrm{sp.})}_{iij} &= \big[ \cp{}{d}{}{d}\ \mlpcp{}{}{d+1}{2d}{\odot^{\sim_i}_{iij}}\  \mlpcp{}{}{2d+1}{}{\odot^{\sim}_{iij}} \big] X^{(t,l,\mathrm{broad.})}_{iij}  \label{eq:sparsification_iij}\\
        X^{(t,l,\mathrm{sp.})}_{kik} &= \big[ \cp{}{d}{}{d}\ \mlpcp{}{}{d+1}{2d}{\odot^{\sim_k}_{kik}}\  \mlpcp{}{}{2d+1}{}{\odot^{\sim}_{kik}} \big] X^{(t,l,\mathrm{broad.})}_{kik}  \label{eq:sparsification_kik}\\
        X^{(t,l,\mathrm{sp.})}_{kij} &= \big[ \cp{}{d}{}{d}\ \mlpcp{}{}{d+1}{2d}{\odot^{\sim_k}_{kij}}\  \mlpcp{}{}{2d+1}{}{\odot^{\sim}_{kij}} \big] X^{(t,l,\mathrm{broad.})}_{kij} \label{eq:sparsification_kij}
    \end{align}
    \noindent where $X^{(t,l,\mathrm{broad.})}_{iij}, X^{(t,l,\mathrm{broad.})}_{kik}, X^{(t,l,\mathrm{broad.})}_{kij}$ are computed by the yet-to-describe step (1) and MLPs $\mlp{\odot^{\sim_i}_{iij}}, \mlp{\odot^{\sim}_{iij}}, \mlp{\odot^{\sim_k}_{kik}}, \mlp{\odot^{\sim}_{kik}}, \mlp{\odot^{\sim_k}_{kij}}, \mlp{\odot^{\sim}_{kij}}$ compute the required sparsifications. We do not describe how to construct such MLPs, but their existence is guaranteed by Proposition~\ref{prop:memorisation_ours}, which we can invoke by constructing the same datasets as shown in the Proof of Lemma~\ref{lemma:3IGN_implements_SubgraphNetworks} for the DSS-GNN derivation. Step (3) aggregates these sparsified messages; concurrently the same layer linearly transforms the result and adds it to the current, intermediate node representations, performing step (4):
    \begin{align*}
        X^{(t,l+1)}_{iii} &= \cp{}{2d}{}{2d} X^{(t,l,\mathrm{sp.})}_{iii} + \cp{}{}{d+1}{2d} [ \mathbf{0}\ ||\ \theta^{\sim_i}_{1,l,t}\ ||\ \theta^{\sim}_{1,l,t} ]\  \broad{}{iii} \pool{}{i} X^{(t,l,\mathrm{sp.})}_{iij} \\
        X^{(t,l+1)}_{kii} &= \cp{}{2d}{}{2d} X^{(t,l,\mathrm{sp.})}_{kii} + \cp{}{}{d+1}{2d} [ \mathbf{0}\ ||\ \theta^{\sim_k}_{2,l,t}\ ||\ \theta^{\sim}_{2,l,t} ]\ \broad{}{kii} X^{(t,l,\mathrm{sp.})}_{kik} + \\
        &\qquad +\cp{}{}{d+1}{2d} [ \mathbf{0}\ ||\ \theta^{\sim_k}_{2,l,t}\ ||\ \theta^{\sim}_{2,l,t} ]\ \broad{}{kii} \pool{}{ki} X^{(t,l,\mathrm{sp.})}_{kij}
    \end{align*}
    \noindent Here, parameters $\theta^{\sim_i}_{1,l,t}, \theta^{\sim}_{1,l,t}, \theta^{\sim_k}_{2,l,t}, \theta^{\sim}_{2,l,t}$ will depend on the specific term being implemented.
    
    As for step (1), one \IGN{3} layer suffices to properly broadcast the current input representations $X^{(t, l)}_{iii}, X^{(t, l)}_{kii}$ over, respectively, $X^{(t, l,\mathrm{broad.})}_{iji}$ and $X^{(t, l,\mathrm{broad.})}_{kik}, X^{(t, l,\mathrm{broad.})}_{kij}$. We will now show the required broadcasting operations in a way that, then, the following layers described above will effectively implement the second and third expansions of the terms in Table~\ref{tab:expansion}:
    \begin{align*}
        \mathrm{[\#1.on]} &\quad X^{(t, l,\mathrm{broad.})}_{iij} = \cp{}{d}{}{d} X^{(t,l)}_{iij} + \cp{}{2d}{d+1}{3d}\glue{\cp{}{d}{}{d}}{\cp{}{d}{}{d}} \broad{}{*,*,i} X^{(t,l)}_{iii}  \\
        \mathrm{[\#2.on]} &\quad X^{(t, l,\mathrm{broad.})}_{iij} = \cp{}{d}{}{d} X^{(t,l)}_{iij} + \cp{}{2d}{d+1}{3d}\glue{\cp{}{d}{}{d}}{\cp{}{d}{}{d}} \broad{}{k,k,i} X^{(t,l)}_{kii}  \\
        \mathrm{[\#3.on]} &\quad X^{(t, l,\mathrm{broad.})}_{iij} = \cp{}{d}{}{d} X^{(t,l)}_{iij} + \cp{}{2d}{d+1}{3d}\glue{\cp{}{d}{}{d}}{\cp{}{d}{}{d}} \broad{}{i,i,k} X^{(t,l)}_{kii}  \\
        \mathrm{[\#2.off]} &\quad X^{(t, l,\mathrm{broad.})}_{kik} = \cp{}{d}{}{d} X^{(t,l)}_{kik} + \cp{}{2d}{d+1}{3d}\glue{\cp{}{d}{}{d}}{\cp{}{d}{}{d}} \broad{}{*,i,*} X^{(t,l)}_{iii} \\
                           &\quad X^{(t, l,\mathrm{broad.})}_{kij} = \cp{}{d}{}{d} X^{(t,l)}_{kij} + \cp{}{2d}{d+1}{3d}\glue{\cp{}{d}{}{d}}{\cp{}{d}{}{d}} \broad{}{*,i,k} X^{(t,l)}_{kii} \\
        \mathrm{[\#3.off]} &\quad X^{(t, l,\mathrm{broad.})}_{kik} = \cp{}{d}{}{d} X^{(t,l)}_{kik} + \cp{}{2d}{d+1}{3d}\glue{\cp{}{d}{}{d}}{\cp{}{d}{}{d}} \broad{}{i,*,i} X^{(t,l)}_{iii} \\
                           &\quad X^{(t, l,\mathrm{broad.})}_{kij} = \cp{}{d}{}{d} X^{(t,l)}_{kij} + \cp{}{2d}{d+1}{3d}\glue{\cp{}{d}{}{d}}{\cp{}{d}{}{d}} \broad{}{k,*,i} X^{(t,l)}_{kii} \\
        \mathrm{[\#4.off]} &\quad X^{(t, l,\mathrm{broad.})}_{kik} = \cp{}{d}{}{d} X^{(t,l)}_{kik} + \cp{}{2d}{d+1}{3d}\glue{\cp{}{d}{}{d}}{\cp{}{d}{}{d}} \broad{}{*,i,*} X^{(t,l)}_{iii} \\
                           &\quad X^{(t, l,\mathrm{broad.})}_{kij} = \cp{}{d}{}{d} X^{(t,l)}_{kij} + \cp{}{2d}{d+1}{3d}\glue{\cp{}{d}{}{d}}{\cp{}{d}{}{d}} \broad{}{*,k,i} X^{(t,l)}_{kii} \\
        \mathrm{[\#5.off]} &\quad X^{(t, l,\mathrm{broad.})}_{kik} = \cp{}{d}{}{d} X^{(t,l)}_{kik} + \cp{}{2d}{d+1}{3d}\glue{\cp{}{d}{}{d}}{\cp{}{d}{}{d}} \broad{}{i,*,i} X^{(t,l)}_{iii} \\
                           &\quad X^{(t, l,\mathrm{broad.})}_{kij} = \cp{}{d}{}{d} X^{(t,l)}_{kij} + \cp{}{2d}{d+1}{3d}\glue{\cp{}{d}{}{d}}{\cp{}{d}{}{d}} \broad{}{i,*,k} X^{(t,l)}_{kii} \\
        \mathrm{[\#6.off]} &\quad X^{(t, l,\mathrm{broad.})}_{kik} = \cp{}{d}{}{d} X^{(t,l)}_{kik} + \cp{}{2d}{d+1}{3d}\glue{\cp{}{d}{}{d}}{\cp{}{d}{}{d}} \broad{}{i,*,i} X^{(t,l)}_{iii} \\
                           &\quad X^{(t, l,\mathrm{broad.})}_{kij} = \cp{}{d}{}{d} X^{(t,l)}_{kij} + \cp{}{2d}{d+1}{3d}\glue{\cp{}{d}{}{d}}{\cp{}{d}{}{d}} \broad{}{*,*,i} X^{(t,l)}_{iii} \\
    \end{align*}
    
    Terms [\#4.on], [\#1.off] involve two sparse summations and thus require a slightly different construction. In particular, they are concurrently implemented by computing steps (1,2,3) twice, in sequence, followed by computation of step (4). The first brodcasting step (1) is realised as follows:
    \begin{align*}
        X^{(t, l, 1^{st} \mathrm{broad.})}_{iij} &= \cp{}{d}{}{d} X^{(t,l)}_{iij} + \cp{}{2d}{d+1}{3d}\glue{\cp{}{d}{}{d}}{\cp{}{d}{}{d}} \broad{}{k,k,i} X^{(t,l)}_{kii} \\
        X^{(t, l, 1^{st} \mathrm{broad.})}_{kik} &= \cp{}{d}{}{d} X^{(t,l)}_{kik} + \cp{}{2d}{d+1}{3d}\glue{\cp{}{d}{}{d}}{\cp{}{d}{}{d}} \broad{}{i,*,i} X^{(t,l)}_{iii} \\
        X^{(t, l, 1^{st} \mathrm{broad.})}_{kij} &= \cp{}{d}{}{d} X^{(t,l)}_{kij} + \cp{}{2d}{d+1}{3d}\glue{\cp{}{d}{}{d}}{\cp{}{d}{}{d}} \broad{}{k,*,i} X^{(t,l)}_{kii} \\
    \end{align*}
    \noindent that is, in the same way as in the implementation of terms [\#2.on], [\#3.off]. Then, the first sparsification step (2) is computed as per~\Cref{eq:sparsification_iij,eq:sparsification_kik,eq:sparsification_kij}, generating $X^{(t, l, 1^{st} \mathrm{sp.})}_{iij}, X^{(t, l, 1^{st} \mathrm{sp.})}_{kik}, X^{(t, l, 1^{st}\mathrm{sp.})}_{kij}$. The first aggregation step (3) is performed jointly with the second broadcasting step (1) as:
    \begin{align*}
        X^{(t, l, 2^{nd} \mathrm{broad.})}_{iij} &= \cp{}{d}{}{d} X^{(t,l,1^{st} \mathrm{sp.})}_{iij} + \cp{d+1}{3d}{d+1}{3d} \broad{}{*,*,i} \pool{}{i} X^{(t,l,1^{st} \mathrm{sp.})}_{iij} \\
        X^{(t, l, 2^{nd} \mathrm{broad.})}_{kik} &= \cp{}{d}{}{d} X^{(t,l,1^{st} \mathrm{sp.})}_{kik} + \cp{d+1}{3d}{d+1}{3d} \broad{}{k,i,k} \pool{}{k,i} X^{(t,l,1^{st} \mathrm{sp.})}_{kij} + \\
        &\qquad + \cp{}{2d}{d+1}{3d}\glue{\cp{}{d}{}{d}}{\cp{}{d}{}{d}} \broad{}{i,*,i} X^{(t,l,1^{st} \mathrm{sp.})}_{iii}
        \\
        X^{(t, l, 2^{nd} \mathrm{broad.})}_{kij} &= \cp{}{d}{}{d} X^{(t,l,1^{st} \mathrm{sp.})}_{kij} + \cp{d+1}{3d}{d+1}{3d} \broad{}{*,i,k} X^{(t,l,1^{st} \mathrm{sp.})}_{kik} + \cp{d+1}{3d}{d+1}{3d} \broad{}{*,i,k} \pool{}{k,i} X^{(t,l,1^{st} \mathrm{sp.})}_{kij}
    \end{align*}
    \noindent where, crucially, the results from pooling are broadcast back into orbit tensors $X_{iij}, X_{kik}, X_{kij}$, given that one more local summation is required. Next, one more sparsification takes place in the form of~\Cref{eq:sparsification_iij,eq:sparsification_kik,eq:sparsification_kij}, generating $X^{(t, l, 2^{nd} \mathrm{sp.})}_{iij}, X^{(t, l, 2^{nd} \mathrm{sp.})}_{kik}, X^{(t, l, 2^{nd} \mathrm{sp.})}_{kij}$. Finally, second aggregation step (3) is performed jointly with the final update step (4), which writes back into orbit tensors $X_{iii}, X_{kii}$:
    \begin{align*}
        X^{(t,l+1)}_{iii} &= \cp{}{2d}{}{2d} X^{(t,l,2^{nd} \mathrm{sp.})}_{iii} + \cp{}{}{d+1}{2d} [ \mathbf{0}\ ||\ \theta^{\sim_i}_{1,l,t}\ ||\ \theta^{\sim}_{1,l,t} ]\ \broad{}{i,i,i} \pool{}{i} X^{(t,l,2^{nd} \mathrm{sp.})}_{iij} \\
        X^{(t,l+1)}_{kii} &= \cp{}{2d}{}{2d} X^{(t,l,2^{nd} \mathrm{sp.})}_{kii} + \cp{}{}{d+1}{2d} [ \mathbf{0}\ ||\ \theta^{\sim_k}_{2,l,t}\ ||\ \theta^{\sim}_{2,l,t} ]\ \broad{}{kii} X^{(t,l,2^{nd} \mathrm{sp.})}_{kik} + \\
        &\qquad +\cp{}{}{d+1}{2d} [ \mathbf{0}\ ||\ \theta^{\sim_k}_{2,l,t}\ ||\ \theta^{\sim}_{2,l,t} ]\ \broad{}{kii} \pool{}{k,i} X^{(t,l,2^{nd} \mathrm{sp.})}_{kij}
    \end{align*}
    
    When all terms are implemented and combined together, it is only left to bring back the dimensionality to $d$, overwriting the previous node representations with the newly computed one:
    \begin{align*}
        X^{(t+1)}_{iii} &= \sigma \Big( \cp{d+1}{}{}{d} X^{(t,L)}_{iii} \Big) \\
        X^{(t+1)}_{kii} &= \sigma \Big( \cp{d+1}{}{}{d} X^{(t,L)}_{kii} \Big) 
    \end{align*}
\end{proof}

\begin{proof}[Proof of Corollary~\ref{cor:ReIGN_upperbound}]
    We proceed  by contradiction as in the Proof of Corollary~\ref{cor:upperbound}. Suppose there exist non-isomorphic but \WL{3}-equivalent graphs $G_1, G_2$ distinguished by instance $\mathcal{R}_{\rho,\Theta,\bar{\pi}}$. That is, $\mathcal{R}_{\rho,\Theta,\bar{\pi}} \big( G_1 \big) \neq \mathcal{R}_{\rho,\Theta,\bar{\pi}} \big( G_2 \big)$. In view of Theorem~\ref{prop:3IGN_implements_ReIGN}, there must exists a \IGN{3} instance $\mathcal{M}_\Omega$ such that $\mathcal{M}_\Omega \big( G_1 \big) \neq \mathcal{M}_\Omega \big( G_2 \big)$. By Theorem~\ref{thm:kIGN_upperbound}, if there exists a \IGN{3} distinguishing $G_1, G_2$, then these two graphs must be distinguished by the \WL{3} algorithm, against our hypothesis.
\end{proof}
\subsection{Proofs for Section~\ref{sec:sun}}

\paragraph{\SUN{} in linear form}

\begin{align}
    x_i^{i,(t+1)} &= \sigma \Big( U^2_{r,t}\cdot x_i^{i,(t)}  + U^3_{r,t}\cdot \sum_{j} x_j^{i,(t)} + \nonumber \\
    &\quad 
    + U^4_{r,t}\cdot \sum_{j \sim_i i} x_j^{i,(t)} + U^5_{r,t}\cdot \sum_h x_i^{h,(t)} + U^6_{r,t}\cdot \sum_{j \sim i} \sum_h x_j^{h,(t)} \Big) \label{eq:linear_SUN}\\
    x_i^{k,(t+1)} &= \sigma \Big( U^0_{t}\cdot x_i^{i,(t)} + U^1_{t}\cdot x_k^{k,(t)} + U^2_{t}\cdot x_i^{k,(t)}  + U^3_{t}\cdot \sum_{j} x_j^{k,(t)} + \nonumber \\
    &\quad 
    + U^4_{t}\cdot \sum_{j \sim_k i} x_j^{k,(t)} + U^5_{t}\cdot \sum_h x_i^{h,(t)} + U^6_{t}\cdot \sum_{j \sim i} \sum_h x_j^{h,(t)} \Big) \label{eq:linear_SUN_2}
\end{align}

\begin{proof}[Proof of Proposition~\ref{prop:SUN_is_in_ReIGN}]
    We construct a stacking of $2$ \ReIGN{2} layers implementing one \SUN{} layer as per Equations~\ref{eq:linear_SUN} and~\ref{eq:linear_SUN_2}. The first layer expands the dimension of the hidden representations to $2d$. The first $d$ channels store the sum of linear transformations operated by $U^2_{r,t}, U^3_{r,t}, U^4_{r,t}$ in Equation~\ref{eq:linear_SUN} and those operated by $U^0_{t}, U^1_{t}, U^2_{t}, U^3_{t}, U^4_{t}$ in Equation~\ref{eq:linear_SUN_2}. Channels $d+1$ to $2d$ will store terms $\sum_h x^{h, (t)}_i$:
    \begin{align*}
        x_i^{i,(t,1)} &=          \glue{(U^2_{r,t} + I_d)}{I_d} x_i^{i,(t)} 
                                + \hemcp{}{}{}{d}{2d} U^3_{r,t}\cdot \sum_{j \neq i} x_j^{i,(t)} + \\ 
                        &\qquad + \hemcp{}{}{}{d}{2d} U^4_{r,t}\cdot \sum_{j \sim_i i} x_j^{i,(t)}
                                + \cp{}{}{d+1}{2d} \sum_{h \neq i} x^{h, (t)}_i \\
        x_i^{k,(t,1)} &=          \glue{ U^0_{t}}{I_d} x_i^{i,(t)} + \hemcp{}{}{}{d}{2d} (U^1_{t} + I_d)\cdot x_k^{k,(t)} + \hemcp{}{}{}{d}{2d} U^2_{t}\cdot x_i^{k,(t)} + \\
                        &\qquad + \hemcp{}{}{}{d}{2d} U^3_{t}\cdot \sum_{j \neq k} x_j^{k,(t)} + \hemcp{}{}{}{d}{2d} U^4_{t}\cdot \sum_{j \sim_k i} x_j^{k,(t)} + \cp{}{}{d+1}{2d} \sum_{h \neq i} x^{h, (t)}_i
    \end{align*}
    \noindent where we have used the following aggregated terms. First equation: [\#2.on] in its global version and second expansion, [\#3.on] in its global version. Second equation: [\#2.off] in its global version, [\#3.off] in its global version and second expansion. Non-appearing \ReIGN{2} terms are nullified. The second \ReIGN{2} layer completes the computation by implementing linear transformations $U^5_{t,r}, U^6_{t,r}, U^5_{t}, U^6_{t}$, and by contracting the dimensionality back to $d$:
    \begin{align*}
        x_i^{i,(t+1)} &= \sigma \Big( \sumglue{I_d}{U^5_{r,t}} x_i^{i,(t,1)} 
                                      + U^6_{r,t}\cdot \cp{d+1}{2d}{}{d} \sum_{j \sim i} x_j^{i,(t,1)} \Big) \\
        x_i^{k,(t+1)} &= \sigma \Big( \sumglue{I_d}{U^5_{t}} x_i^{k,(t,1)} 
                                      + U^6_{t}\cdot \cp{d+1}{2d}{}{d} \sum_{j \sim i} x_j^{k,(t,1)} \Big)
    \end{align*}
    \noindent where we have used aggregated terms [\#2.on] and [\#3.off] in their third expansion.
\end{proof}

\begin{proof}[Proof of Proposition~\ref{prop:SUN_implements_SubgraphNetworks}]
    We describe how a stacking of \SUN{} layers in the form of Equations~\ref{eq:linear_SUN} and~\ref{eq:linear_SUN_2} implements layers of models in $\Upsilon$ with~\citet{morris2019weisfeiler} base-encoders. As usual, we proceed model by model.
    
    [\emph{\textbf{DS-GNN}}] updates root and non-root nodes in the same manner. Thus, we seek to find a choice of linear operators in Equations~\ref{eq:linear_SUN} and~\ref{eq:linear_SUN_2} in a way that the two coincide and exactly correspond to Equation~\ref{eq:DSGNN_Morris}. To this aim, it is sufficient to set:
    \begin{itemize}
        \item $U^2_{r,t} = U^2_{t} = W_{1,t}$
        \item $U^4_{r,t} = U^4_{t} = W_{2,t}$
    \end{itemize}
    \noindent and all other weight matrices $U$ to $\mathbf{0}$.
    
    [\emph{\textbf{DSS-GNN}}] implements Equation~\ref{eq:DSSGNN_Morris}. We proceed similarly as above, setting:
    \begin{itemize}
        \item $U^2_{r,t} = U^2_{t} = W^1_{1,t}$
        \item $U^4_{r,t} = U^4_{t} = W^1_{2,t}$
        \item $U^5_{r,t} = U^5_{t} = W^2_{1,t}$
        \item $U^6_{r,t} = U^6_{t} = W^2_{2,t}$
    \end{itemize}
    \noindent and all other weight matrices $U$ to $\mathbf{0}$.
    
    [\emph{\textbf{GNN-AK-ctx} \& \textbf{GNN-AK}}] We seek to recover Equations~\ref{eq:GNNAK_Morris_S} ([S]) and~\ref{eq:GNNAK_A} ([A]). Each message passing layer in [S] is implemented by one \SUN{} layer similarly as above, that is by setting:
     \begin{itemize}
        \item $U^2_{r,t,l} = U^2_{t,l} = W_{1,t,l}$
        \item $U^4_{r,t,l} = U^4_{t,l} = W_{2,t,l}$
    \end{itemize}
    \noindent and all other weight matrices $U$ to $\mathbf{0}$.
    Then, block [A] is implemented by one \SUN{} layer by setting:
    \begin{itemize}
        \item $U^2_{r,t} = U^3_{r,t} = U^5_{r,t} = I$
        \item $U^0_{t} = U^3_{t} = U^5_{t} = I$
    \end{itemize}
    \noindent and all other weight matrices $U$ to $\mathbf{0}$.
    In the case of GNN-AK, we instead require $U^5_{r,t} = U^5_{t} = \mathbf{0}$. No activation $\sigma$ is applied after this layer.
    
    [\emph{\textbf{ID-GNN}}] Two \SUN{} layers can implement one ID-GNN layer as in Equation~\ref{eq:IDGNN_v2}, similarly as shown for \ReIGN{2} in the Proof of Theorem~\ref{thm:ReIGN_implements_SubgraphNetworks}. The first layer doubles the representation dimension and applies projections $W_{1,t}, W_{2,t}, W_{3,t}$, by setting:
    \begin{itemize}
        \item $U^2_{r,t,1} = \glue{W_{1,t}}{W_{3,t}}$
        \item $U^2_{t,1} = \glue{W_{1,t}}{W_{2,t}}$
    \end{itemize}
    \noindent and all other weight matrices $U$ to $\mathbf{0}$. Here, as usual, $\big[ \cdot \cdot \big]$ indicates vertical concatenation. No activation $\sigma$ is applied after this layer.
    
    The second layer has its weight matrices set to:
    \begin{itemize}
        \item $U^2_{r,t,2} = U^2_{t,2} = \cp{}{d}{}{d}$
        \item $U^4_{r,t,2} = U^4_{t,2} = \cp{d+1}{2d}{}{d}$
    \end{itemize}
    and all other weight matrices $U$ to $\mathbf{0}$.
    
    [\emph{\textbf{NGNN}}] layers perform independent message passing on each subgraph. \SUN{} implements these as shown for DS-GNN.
\end{proof}

\section{Future research directions}\label{app:future}
The following are promising directions for future work:
\begin{enumerate}
    \item \emph{Extension to higher-order node policies.} The prior works of \citet{cotta2021reconstruction,papp2021dropgnn} suggested using more complex policies that depend on tuples of nodes rather than a single node. Since there are exactly $n^k$ distinct $k$-tuples, and each subgraph is defined by a second-order adjacency tensor, we conjecture that Subgraph GNNs applied to such policies are bounded by \WL{(k+2)}. See Appendix \ref{app:higher} for additional details.
    \item \emph{Beyond \WL{3}.} Our results suggest  two directions for breaking the \WL{3} representational limit: (i) Using policies not computable by \IGNs{3} (ii) Using higher-order node-based policies as mentioned above.
    \item \emph{Layers vs. policies.} We make an interesting observation regarding the relationship between layer structure and subgraph selection policies: Having a non-shared set of parameters for root and non-root nodes, \SUN{} may be capable of learning the policies $\pi_{\textrm{NM}}, \pi_{\textrm{ND}}$ by itself. This raises the question of whether we should let the model learn a policy or specify one in advance.
    \item \emph{Lower bound on \SUN{} and \ReIGN{2}.} \revision{In this work we have proved a \WL{3} \emph{upper bound} on the expressive power of \SUN{}, \ReIGN{2} and other node-based Subgraph GNNs by showing their computation on a given graph pair can be simulated by a \IGN{3}. It is natural to ask whether a (tighter) \emph{lower bound} exists as well. In this sense, it is reasonable to believe that node-based Subgraph GNNs are not capable of implementing \IGNs{3}, as they inherently operate on a second-order object. However, this does not necessarily imply these models are less expressive than \IGNs{3}, when considering graph separation. For example, they may still be able to distinguish between the same pairs of graphs distinguished by 3-IGNs, hence attaining \WL{3} expressive power. It is because of this reason that we believe studying the expressivity gap between \ReIGN{2} (or any of the subsumed Subgraph GNNs) and \WL{3} represents an interesting open question that could be addressed in future work}.
    \item \revision{\emph{Expressive power of subgraph selection policies.} Another interesting direction for future work would be to better characterise the impact of subgraph selection policies on the expressive power of Subgraph GNNs. \citet{bevilacqua2022equivariant} already showed that the DS-GNN model can distinguish some Strongly Regular graphs in the same family when equipped with edge-deletion policy, but not with node-deletion or depth-$n$ ego-networks~\citep[Proposition~3]{bevilacqua2022equivariant}. However, edge-deletion is \emph{not} a node-based policy since subgraphs are not in a bijection with nodes in the original graph. It still remains unclear whether a stratification in expressive power exists amongst node-based policies in particular, and under which conditions -- if any -- this last holds.}
\end{enumerate}
We note that, related to 3. and concurrently to the present work,~\citet{Qian2022osan} experiment with directly learning policies by back-propagating through discrete structures via perturbation-based differentiation~\citep{niepert2021implicit}.

\section{Extension to higher-order node policies}\label{app:higher}
Constructing a higher-order subgraph selection policy (\Cref{app:future}, direction 1.), amounts to defining \emph{selection function} $f$ on a graph and a $k$-\emph{tuple} of its nodes: For a graph $G \in \mathcal{G}$, the subgraphs of such a policy are obtained as $G_{(v_1,\dots,v_k)} = f\big(G,(v_1,\dots,v_k)\big)$. The policy contains a subgraph for each possible tuple $(v_1,\dots,v_k)$. We refer to such policies as \emph{$k$-order node policies}. The $k$-node deletion policy suggested by \citet{cotta2021reconstruction} is a natural example as the bag of subgraphs contains all subgraphs that are obtained by removing $k$ distinct nodes from the original graph. Since there are exactly $n^k$ distinct tuples, and each subgraph is defined by a second-order adjacency tensor in $\mathbb{R}^{n^{2}}$, these bags of subgraphs can be arranged into tensors in $\mathbb{R}^{n^{k+2}}$. Noting that the symmetry of these tensors can be naturally defined by the diagonal action of $S_n$ on $\{1, \mathellipsis, n\}^{k+2}$ we raise the following generalisation of \Cref{cor:upperbound}:

\begin{conjecture}
    Subgraph GNNs equipped with $k$-order \emph{node-deletion}, $k$-order \emph{node-marking} or $k$-order \emph{ego-networks} policies are bounded by $(k+2)$-WL.  
\end{conjecture}

We believe that proving this conjecture can be accomplished by following the same steps as our proof, i.e., by showing that \IGN{(k+2)} can implement the bag and the update steps of Subgraph GNNs. We also note that $\ReIGN{k}$, a higher order analogue of \ReIGN{2}, can be obtained by following the steps in Section \ref{sec:space}. We leave both directions for future work. 

We end this section by noting that the statement in our conjecture is considered in a work concurrent to ours by~\citet{Qian2022osan}.

\section{Experimental details and additional results}\label{app:experiments}
\begin{table}[t!]
    %\vspace{-35pt}
    \centering
    \tiny
    \caption{TUDatasets. The top three are highlighted by \textbf{\textcolor{red}{First}}, \textbf{\textcolor{violet}{Second}}, \textbf{Third}.}
    \label{tab:tud_datasets_full}
    \resizebox{\linewidth}{!}{%
    \begin{tabular}{l | lllll | ll}
        \toprule
        Dataset & 
        \textsc{MUTAG} &
        \textsc{PTC} &
        \textsc{PROTEINS} &
        \textsc{NCI1} &
        \textsc{NCI109} &
        \textsc{IMDB-B} &
        \textsc{IMDB-M}
        \\

        \midrule
         
        \textsc{DCNN}~\citep{DCNN_2016} & 
        N/A &  
        N/A &
        61.3$\pm$1.6 &
        56.6$\pm$1.0 &
        N/A &
        49.1$\pm$1.4 &
        33.5$\pm$1.4
        \\

        \textsc{DGCNN}~\citep{zhang2018end} & 
        85.8$\pm$1.8 & 
        58.6$\pm$2.5 & 
        75.5$\pm$0.9 &
        74.4$\pm$0.5 &
        N/A &
        70.0$\pm$0.9 & 
        47.8$\pm$0.9
        \\
        
        \textsc{IGN}~\citep{maron2018invariant} &
        83.9$\pm$13.0 &
        58.5$\pm$6.9 &
        76.6$\pm$5.5 &
        74.3$\pm$2.7 &
        72.8$\pm$1.5 &
        72.0$\pm$5.5 & 
        48.7$\pm$3.4 
        \\

        \textsc{PPGNs}~\citep{maron2019provably} &
        90.6$\pm$8.7 &
        66.2$\pm$6.6 &
        \first{77.2}$\pm$4.7 &
        83.2$\pm$1.1 &
        \third{82.2}$\pm$1.4 &
        73.0$\pm$5.8 & 
        50.5$\pm$3.6
         \\

        \textsc{Natural GN}~\citep{de2020natural} &
        89.4$\pm$1.6 &
        66.8$\pm$1.7 &
        71.7$\pm$1.0 &
        82.4$\pm$1.3 &
        N/A &
        73.5$\pm$2.0 &
        51.3$\pm$1.5
        \\

        \textsc{GSN}~\citep{bouritsas2022improving}&
        \second{92.2}$\pm$7.5 &
        \second{68.2}$\pm$7.2 & 
        76.6$\pm$5.0 &
        83.5$\pm$2.0 &
        N/A &
        \first{77.8}$\pm$3.3 & 
        \first{54.3}$\pm$3.3
        \\
        
        \textsc{SIN}~\citep{bodnar2021weisfeilerA} & 
        N/A  &
        N/A & 
        76.4$\pm$3.3 &
        82.7$\pm$2.1 &
        N/A &
        75.6$\pm$3.2 & 
        52.4$\pm$2.9
        \\

        \textsc{CIN}~\citep{bodnar2021weisfeilerB} & 
        \first{92.7}$\pm$6.1 &
        \second{68.2}$\pm$5.6 &
        \third{77.0}$\pm$4.3 &
        83.6$\pm$1.4 &
        \first{84.0}$\pm$1.6 &
        75.6$\pm$3.7 & 
        52.7$\pm$3.1
        \\

        \midrule

        \textsc{GIN}~\citep{xu2019how} & 
        89.4$\pm$5.6 & 
        64.6$\pm$7.0 & 
        76.2$\pm$2.8 &
        82.7$\pm$1.7 &
        \third{82.2}$\pm$1.6 &
        75.1$\pm$5.1 &
        52.3$\pm$2.8
        \\
        
        \midrule
        \midrule
        \textsc{ID-GNN (GIN)}~\citep{you2021identity} & 
        90.4$\pm$5.4 & 
        67.2$\pm$4.3 & 
        75.4$\pm$2.7 &
        82.6$\pm$1.6 &
        82.1$\pm$1.5 &
        76.0$\pm$2.7 &
        52.7$\pm$4.2
        
        \\

        \midrule
        
        \textsc{DropEdge}~\citep{rong2019dropedge}& 
        91.0$\pm$5.7 & 
        64.5$\pm$2.6 & 
        73.5$\pm$4.5 &
        82.0$\pm$2.6 &
        \third{82.2}$\pm$1.4 &
        76.5$\pm$ 3.3 &
        52.8$\pm$ 2.8 
        
        \\

        \midrule

        {\textsc{DS-GNN (GIN) (ND)}}~\citep{bevilacqua2022equivariant}   & 
         89.4$\pm$4.8 &
         66.3$\pm$7.0 &
         \second{77.1}$\pm$4.6 &
         \third{83.8}$\pm$2.4 &
         82.4$\pm$1.3 &
         75.4$\pm$2.9 &
         52.7$\pm$2.0
        \\

        {\textsc{DS-GNN (GIN) (EGO)}}~\citep{bevilacqua2022equivariant}   & 
         89.9$\pm$6.5 &
         68.6$\pm$5.8 &
         76.7$\pm$5.8 &
         81.4$\pm$0.7 &
         79.5$\pm$1.0 &
         76.1$\pm$2.8 &
         52.6$\pm$2.8
        \\
        
        {\textsc{DS-GNN (GIN) (EGO+)}}~\citep{bevilacqua2022equivariant}   & 
         91.0$\pm$4.8 &
         68.7$\pm$7.0 &
         76.7$\pm$4.4 &
         82.0$\pm$1.4 &
         80.3$\pm$0.9 &
         \second{77.1}$\pm$2.6 &
         53.2$\pm$2.8
        \\
        
        \midrule

        {\textsc{DSS-GNN (GIN) (ND)}}~\citep{bevilacqua2022equivariant}   & 
         91.0$\pm$3.5 &
         66.3$\pm$5.9 &
         76.1$\pm$3.4 &
         83.6$\pm$1.5 &
         \second{83.1}$\pm$0.8 &
         76.1$\pm$2.9 &
         \second{53.3}$\pm$1.9
        \\

        {\textsc{DSS-GNN (GIN) (EGO)}}~\citep{bevilacqua2022equivariant}   & 
         91.0$\pm$4.7 &
         68.2$\pm$5.8 &
         76.7$\pm$4.1 &
         83.6$\pm$1.8 &
         82.5$\pm$1.6 &
         76.5$\pm$2.8 &
         \second{53.3}$\pm$3.1
        \\
        
        {\textsc{DSS-GNN (GIN) (EGO+)}}~\citep{bevilacqua2022equivariant}   & 
        91.1$\pm$7.0 &
        \first{69.2}$\pm$6.5 &
        75.9$\pm$4.3 &
        83.7$\pm$1.8 &
        82.8$\pm$1.2 &
        \second{77.1}$\pm$3.0&
        53.2$\pm$2.4
        \\
        
        \midrule
        {\textsc{GIN-AK+}}~\citep{zhao2022from} &
        \third{91.3}$\pm$7.0 	&
        \third{67.8}$\pm$8.8	&
        \second{77.1}$\pm$5.7	&
        \first{85.0}$\pm$2.0	&
        N/A	&
        75.0$\pm$4.2 &
        N/A
        \\

        \midrule

        {\bf \textsc{SUN (GIN) (NULL)}}   & 
         91.6$\pm$4.8&
         67.5$\pm$6.8&
         76.8$\pm$4.4&
         84.1$\pm$2.0&
         83.0$\pm$0.9&
         76.2$\pm$1.9&
         52.6$\pm$3.2
        \\
        {\bf \textsc{SUN (GIN) (NM)}}   & 
         91.0$\pm$4.7&
         67.0$\pm$4.8&
         75.7$\pm$3.4&
         84.2$\pm$1.5&
         \second{83.1}$\pm$1.5&
         76.1$\pm$2.9&
         \third{53.1}$\pm$2.5
        \\
        
        {\bf \textsc{SUN (GIN) (EGO)}}   & 
         \first{92.7}$\pm$5.8&
         67.2$\pm$5.9&
         76.8$\pm$5.0&
         83.7$\pm$1.3&
         83.0$\pm$1.0&
         \third{76.6}$\pm$3.4&
         52.7$\pm$2.3
        \\
        
        {\bf \textsc{SUN (GIN) (EGO+)}}   & 
         92.1$\pm$5.8&
         67.6$\pm$5.5&
         76.1$\pm$5.1&
         \second{84.2}$\pm$1.5&
         \second{83.1}$\pm$1.0&
         76.3$\pm$1.9&
        52.9$\pm$2.8
        \\

        \bottomrule

    %\end{adjustwidth}
    \end{tabular}
    }
\end{table}

\begin{table}[ht]
    \scriptsize
    \caption{Test mean metric on the Graph Properties dataset. All Subgraph GNNs employ a GIN base-encoder.
     }\label{tab:simulation2}
    \centering
    \begin{tabular}{l  ccc
    }
    \toprule
     \multirow{2}{*}{Method} &
     \multicolumn{3}{c}{Graph Properties ($\log_{10}$(MSE))}
     \\
     \cmidrule(l{2pt}r{2pt}){2-4} 
                                   &  IsConnected & Diameter & Radius 
    \\ 
    \midrule  
     \textsc{GCN}~\citep{kipf2016semi}     & -1.7057 & -2.4705 & -3.9316
     \\
     \textsc{GIN}~\citep{xu2019how}       & -1.9239 & -3.3079 & -4.7584
     \\
     \textsc{PNA}~\citep{corso2020pna}       &-1.9395 &-3.4382 &-4.9470
     \\
     \textsc{PPGN}~\citep{maron2019provably}        & -1.9804 & -3.6147 & -5.0878
     \\
     \midrule
    \textsc{GNN-AK}~\citep{zhao2022from}   & -1.9934 & -3.7573 & -5.0100
    \\
    \textsc{GNN-AK-ctx}~\citep{zhao2022from} & -2.0541 & -3.7585 & -5.1044
    \\
    \textsc{GNN-AK$+$}~\citep{zhao2022from}  & -2.7513 & -3.9687 & -5.1846
    \\
    \midrule
    \textsc{SUN (EGO)} & -2.0001 & -3.6671 & -5.5720
    \\
    \textsc{SUN (EGO+)} & -2.0651 & -3.6743	& -5.6356
    \\
    \bottomrule         
    \end{tabular}
\end{table}

\begin{table}[ht]
    \scriptsize
    \caption{Test mean and std for the corresponding metric on the synthetic tasks. A comparison with other methods can be found in \Cref{tab:count-zinc,tab:simulation2}.
     }\label{tab:synth-std}
    \centering
    \begin{minipage}{\textwidth}
    \centering
    \begin{tabular}{l cccc }
    \toprule
     \multirow{2}{*}{Method} & \multicolumn{4}{c}{Counting Substructures (MAE)} 
     \\
     \cmidrule(l{2pt}r{2pt}){2-5}
                                   & Triangle & Tailed Tri. & Star & 4-Cycle
    \\ 
    \midrule 
    \textsc{SUN (EGO)} &  0.0092$\pm$0.0002 & 0.0105$\pm$0.0010 & 0.0064$\pm$0.0006 & 0.0140$\pm$0.0014 \\
    \textsc{SUN (EGO+)} & 0.0079$\pm$0.0003 & 0.0080$\pm$0.0005 & 0.0064$\pm$0.0003  & 0.0105$\pm$0.0006 \\
    \bottomrule         
    \end{tabular}
    \vspace{15pt}
    \end{minipage}
    \begin{minipage}{\textwidth}
    \centering
    \begin{tabular}{l  ccc
    }
    \toprule
     \multirow{2}{*}{Method} & \multicolumn{3}{c}{Graph Properties ($\log_{10}$(MSE))}
     \\
     \cmidrule(l{2pt}r{2pt}){2-4}  &  IsConnected & Diameter & Radius 
    \\ 
    \midrule  
        \textsc{SUN (EGO)} & -2.0001$\pm$0.0211 & -3.6671$\pm$0.0078 & -5.5720$\pm$0.0423
    \\
    \textsc{SUN (EGO+)} & -2.0651$\pm$0.0533 & -3.6743$\pm$0.0178	& -5.6356$\pm$0.0200 \\
    \bottomrule
    \end{tabular}
    \end{minipage}
\end{table}

\begin{figure}
    \centering
    \includegraphics[width=0.5\textwidth]{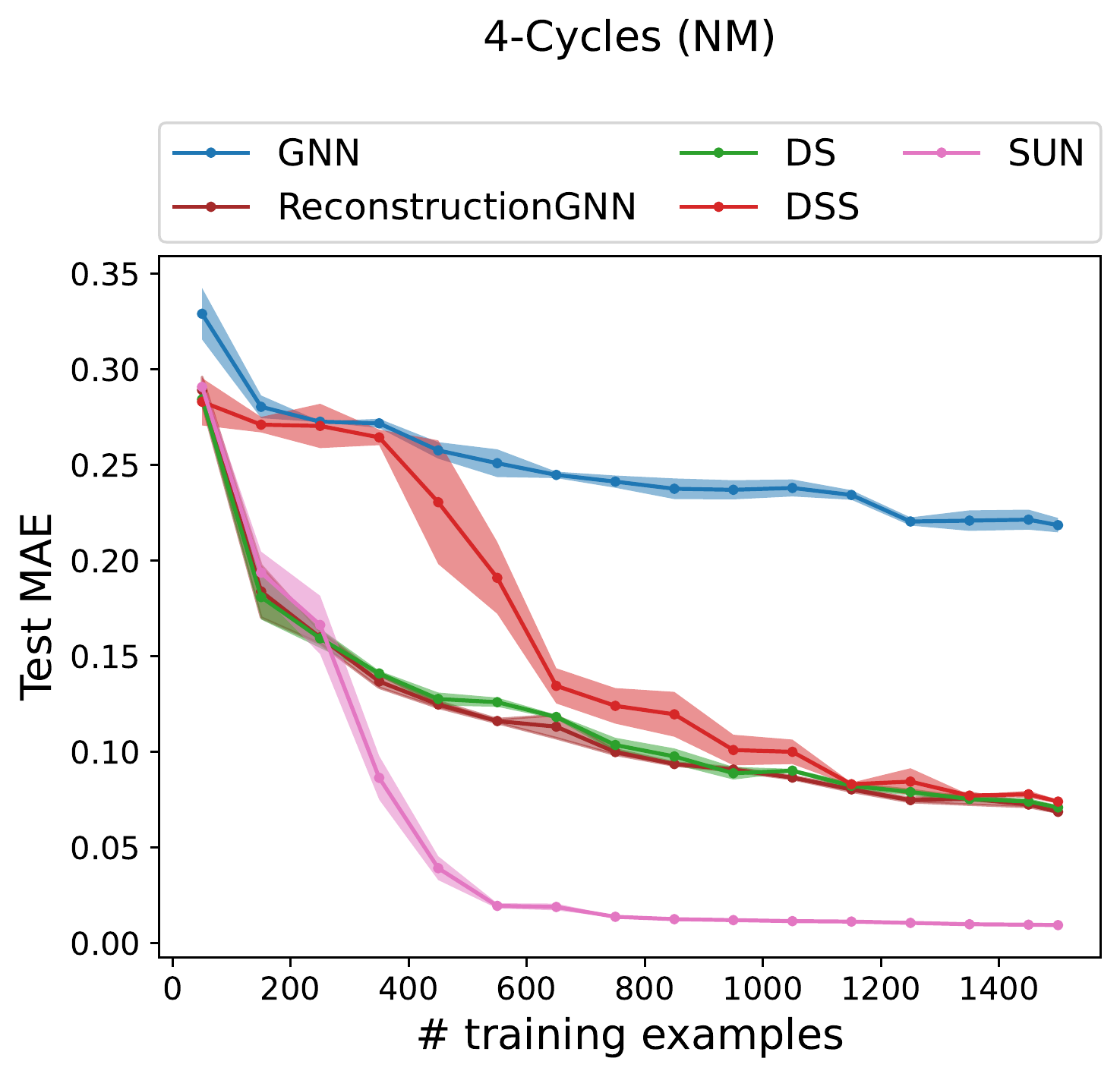}
    \caption{Generalisation capabilities of Subgraph GNNs in the counting prediction task with the node-marking (NM) selection policy.}
    \label{fig:node_marked}
\end{figure}

\begin{table}[ht]
    \caption{Performances for the 4-Cycles counting task. The Trivial Predictor  always outputs the mean training target.}
    \centering
    \begin{tabular}{l llll}
    \toprule
        &  Best Train & Train & Val & Test  \\
    \midrule
    Trivial Predictor & 0.9097 & 0.9097 & 0.9193 & 0.9275 \\
    GIN~\citep{xu2019how} & 0.0283$\pm$0.0032 & 0.1432$\pm$0.0526 &  0.2148$\pm$0.0051 &  0.2185$\pm$0.0061\\
    SUN (EGO+) & 0.0072$\pm$0.0002 & 0.0072$\pm$0.0001 & 0.0097$\pm$0.0005 & 0.0105$\pm$0.0002\\ 
    \bottomrule
    \end{tabular}
    \label{tab:train-perf}
\end{table}

\subsection{Additional experiments}

\textbf{TUDatasets.} We experimented the performances of \SUN{} on the widely used datasets from the TUD repository \citep{morris2020tudataset}, and include a comparison of different subgraph selection policies. Marking a first, preliminary, step in the future research direction (3) in~\Cref{app:future}, we also experiment with a `NULL' policy, i.e. by constructing bags by simply replicating the original graph $n$ times, without any marking or connectivity alteration. Results are reported in~\Cref{tab:tud_datasets_full}. Notably, the EGO policies obtain the best results in 6 out of 7 datasets, while the NULL policy does not seem an advantageous strategy on this benchmarking suite. 
On average, \SUN{} compares well with best performing approaches across domains, while featuring smaller result variations w.r.t.\ to GNN-AK+~\citep{zhao2022from}.

\textbf{Synthetic -- Graph property prediction.} \Cref{tab:simulation2} reports mean test $\log_{10}(\text{MSE})$ on the Graph Properties dataset. \SUN{} achieves state-of-the art results on the ``Radius'' task, where each target is defined as the largest (in absolute value)  eigenvalue of the graph's adjacency matrix. \Cref{tab:synth-std} gathers the standard deviation on the results for these benchmarks as well as ``Counting Substructures'' ones over 3 seeds as in \citet{zhao2022from}.

\textbf{Generalisation from limited data -- Node Marking.} \Cref{fig:node_marked} tests the generalisation abilities of Subgraph GNNs on the 4-Cycles task using the node-marking selection policy (NM). Similarly to \Cref{fig:4cycles-ego,fig:4cycles-ego-plus}, \SUN{} outperforms all other Subgraph GNNs by a large margin and, except for a short initial phase where all Subgraph architectures perform similarly, \SUN{} generalises better.

\textbf{Analysis of generalisation.} The GNN's poor performance on this set of experiments may be due to different reasons, e.g. underfitting vs. overfitting behaviours. In this sense, here we deepen our understanding on the $4$-Cycles counting task by additionally inspecting characteristics of the train samples and performance thereon. \Cref{tab:train-perf} reports evaluation results on training and validation sets at the epoch of best validation performance, as well as the best overall training performance. We include the GNN and \SUN{} models, along with a trivial predictor which always outputs the mean training target. First, we observe that the GNN exhibits a relatively large gap between the two reported training MAEs if compared to \SUN{}. This rules out scenarios of complete underfitting, especially considering the trivial predictor performs much worse than the GNN. This led us to evaluate the expressiveness class required to disambiguate all training and test samples: out of all possible graph pairs, we only found one not distinguished by a \WL{1} test running $6$ colour refinement rounds --- same as the number of message passing layers in our GNN. As a consequence, the GNN baseline can effectively assign unique representations to almost all graphs, this justifying its superior performance w.r.t.\ the trivial predictor. Yet, \SUN{} achieves much better results on all sets, while displaying a smaller train-test gap. This is an indication that, although the hypothesis class of the GNN is sufficiently large to avoid underfitting, it renders the overall learning procedure difficult, leading to suboptimal solutions and partial memorisation phenomena. These results are in line with the observations in~\citet[Appendix G.1]{cotta2021reconstruction}, where the authors have performed a similar analysis on real-world benchmarks.

\revision{
\textbf{Ablation study.}
\begin{table}[t!]
\centering
    \caption{\revision{Test results on ZINC dataset (GIN base-encoder). Each row reports a particular ablation applied on top of the ones in the upper rows.}}
    \label{tab:ablation}
    \small
    \begin{tabular}{l c c}
        \toprule
        \multirow{2}{*}{Method} & \multicolumn{2}{c}{\textsc{ZINC (MAE $\downarrow$)}}
     \\
          \cmidrule(l{2pt}r{2pt}){2-3}
                                   & EGO & EGO+ \\
        \midrule
        \textsc{\bf SUN} & 0.083$\pm$0.003 &  0.084$\pm$0.002\\
        \midrule
        w/o $x^{i,(t)}_{i}, x^{k,(t)}_{k}$ & 0.089$\pm$0.004 & 0.089$\pm$0.002 \\
        $\theta_1 = \theta_2$ & 0.093$\pm$0.003 &  0.093$\pm$0.004 \\
        w/o $\sum_{j} x^{k,(t)}_{j}$ & 0.093$\pm$0.004 & 0.090$\pm$0.004 \\
        w/o $\sum_{h} x^{h,(t)}_{i}, \sum_{j \sim i} \sum_{h} x^{h,(t)}_{j}$ & 0.111$\pm$0.005 & 0.101$\pm$0.007 \\
        \bottomrule
    \end{tabular}
\end{table}
To assess the impact of the terms in the SUN layer, we perform an ablation study by making sequential changes to \Cref{eq:sun_layer_on,eq:sun_layer_off}
until recovering an architecture similar to NGNN, DS-GNN. 
We considered the ZINC-12k molecular dataset, using GIN as base graph encoder. \Cref{tab:ablation} reports the performances for the EGO and EGO+ policies.
As it can be seen, each ablation generally produces some performance degradation, with the removal of $\sum_{j} x^{k,(t)}_{j}$ having no significant impact (EGO policy) or even being beneficial when the other changes are made (EGO+ policy).
Interestingly, in the EGO+ policy case, although root nodes are explicitly marked, the architecture seems to still benefit from not sharing parameters between root and non-root updates. Indeed, imposing the weight sharing $\theta_1 = \theta_2$ deteriorates the overall performance, which gets similar to the one obtained for the EGO policy. These results indicates that, in the SUN layer, most of the terms concur to the strong empirical performance of the architecture, including the choice of not sharing parameters between root and non-root updates.
}

\subsection{Experimental details}
We implemented our model using Pytorch~\citep{pytorch} and Pytorch Geometric~\citep{fey2019fast} (available respectively under the BSD and MIT license).
We ran our experiments on NVIDIA DGX V100, GeForce 2080, and TITAN V GPUs. We performed hyperparameter tuning using the Weight and Biases framework~\citep{wandb}. The time spent on each run depends on the dataset under consideration, with the largest being ogbg-molhiv which takes around 8 hours for 200 epochs and asam optimizer. The time for a single ZINC run is 1 hour and 10 minutes for 400 epochs.
\SUN{} uses the mean aggregator for the feature matrix and directly employs the adjacency matrix of the original graph as the aggregated adjacency (\Cref{eq:sun_layer_on,eq:sun_layer_off}). We used the sum aggregator for all the other terms. Unless otherwise specified, \SUN{} uses the following update equations:
\begin{align}
    x_v^{v, (t+1)} =& \enspace \sigma \Big(
                  \mu^{2}_{t, r} \big ( x_v^{v,(t)} \big ) +
                  \mu^{3}_{t, r} \big ( \sum_{w} x_w^{v, (t)} \big ) + \nonumber \\
        & \quad  + \gamma^{0}_{t, r} \big (
                                            x_v^{v, (t)},
                                            \sum_{w \sim_v v} x_w^{v, (t)} \big ) +
                  \gamma^{1}_{t, r} \big (
                                            \sum_{h} x_v^{h, (t)},
                                            \sum_{w \sim v} \sum_{h} x_w^{h, (t)} \big ) \Big) \label{eq:sun_exp_on} \\
    x_v^{k, (t+1)} =& \enspace \sigma \Big(
                  \mu^{0}_t \big ( x_v^{v,(t)} \big ) +
                  \mu^{1}_t \big ( x_k^{k,(t)} \big ) +
                  \mu^{2}_t \big ( x_v^{k,(t)} \big ) +
                  \mu^{3}_t \big ( \sum_{w} x_w^{k, (t)} \big ) + \nonumber \\
        & \quad  + \gamma^{0}_t \big( x_v^{k, (t)}, \sum_{w \sim_k v} x_w^{k, (t)} \big ) +
                  \gamma^{1}_t \big (
                                \sum_{h} x_v^{h, (t)},
                                \sum_{w \sim v} \sum_{h} x_w^{h, (t)} \big ) \Big) \label{eq:sun_exp_off}
\end{align}
\noindent where $\mu$'s are two-layer MLPs and each $\gamma$ consists of one GIN~\citep{xu2019how} convolutional layer whose internal MLP matches the dimensionality of $\mu$'s, e.g.,
\begin{align*}
    \gamma^{0}_t \big( x_v^{k, (t)}, \sum_{w \sim_k v} x_w^{k, (t)} \big ) = \hat{\mu}^{0}_t \Big ((1+\epsilon) x_v^{k, (t)} + \sum_{w \sim_k v} x_w^{k, (t)}\Big)
\end{align*}
where $\hat{\mu}^{0}_t$ is an MLP.
Details on the hyperparameter grid and architectural choices specific for each dataset are reported in the following subsections.

\subsubsection{Synthetic datasets}\label{sec:synt}
We used the dataset splits and evaluation procedure of \citet{zhao2022from}. We considered a batch size of 128 and used Adam optimiser with a learning rate of $0.001$ which is decayed by 0.5 every 50 epochs. Training is stopped after 250 epochs. We used GIN as base encoder, and tuned the number of layers in $\{5, 6\}$, and the embedding dimension in $\{64, 96, 110\}$. The depth of the ego-networks is set 2 and 3 in, respectively, the Counting and Graph Property tasks, in accordance with~\citet{zhao2022from}. Results of existing baselines reported in \Cref{tab:count-zinc,tab:simulation2} are taken from~\citet{zhao2022from}.

\subsubsection{ZINC-12k}\label{sec:zinc}
We used the same dataset splits of \citet{dwivedi2021graph}, and followed the evaluation procedure prescribed therein.
We used Mean Absolute Error as training loss and evaluation metric.
We considered batch size of 128, and Adam optimizer with initial learning rate of 0.001 which is decayed by 0.5 after the validation metric does not improve for a patience that we set of 40 epochs. Training is stopped after the learning rate reaches the value of $0.00001$, at which time we compute the test metric.
We re-trained all Subgraph GNNs to comply with the 500k parameter budget, and also to the above standard procedure in the case of GNN-AK and GNN-AK-ctx. For GNN-AK+, we reported the result $0.086 \pm ???$ specified by the authors in the rebuttal phase on Openreview, where the question marks indicate that the standard deviation was not provided. We also re-ran GNN-AK+ with the aforementioned standard procedure (learning rate decay and test at the time of early stopping) and obtained $0.091 \pm 0.011$. All Subgraph GNNs use 6 layers and ego-networks of depth 3. We use GIN as the base encoder and we set the embedding dimension to 128 for NGNN, DS- and DSS-GNN, to 100 for GNN-AK variants and to 64 for \SUN{}.
DS-GNN employs invariant deep sets layers~\citep{zaheer2017deep} of the form $\rho(\frac{1}{n} \sum_{i=1}^n \phi(x_i))$ where $x_i$ denotes the representation of subgraph $i$. We tuned $\phi$ and $\rho$ to be either a 2-layers MLP or a single layer with dimensions in $\{64, 128\}$. All other parameters are left as in the original implementation of the corresponding method.
We repeat the experiments with 10 different initialisation seeds, and report mean and standard deviation. 

\subsubsection{OGBG-molhiv dataset}
We used the evaluation procedure proposed in \citet{hu2020open}, which prescribes running each experiment with 10 different seeds and reporting the results at the epoch achieving the best validation metric. Following \citet{zhao2022from}, we disabled the subgraph aggregation components $\mu^{3}_{t, r} \big ( \sum_{w} x_w^{v, (t)} \big )$ and $\mu^{3}_t \big ( \sum_{w} x_w^{k, (t)} \big )$ in
\Cref{eq:sun_exp_on,eq:sun_exp_off}. We used the same architectural choices of \citet{zhao2022from}, namely depth-3 ego-networks, 2 GIN layers, residual connections and dropout of 0.3. We set the embedding dimension of the GNNs to be 64. Early experimentation with the common Adam optimiser revealed large fluctuations in the validation metric, which we found to considerably oscillate across optimisation steps even for small learning rate values. Thus, given the non-uniform strategy adopted to generate train, validation and test splits, we considered employing the ASAM optimiser~\citep{kwon2021asam}. ASAM considers the sharpness of the training loss in each gradient descent step, effectively driving the optimisation towards flatter minima. We left its $\rho$ parameter to its default value of $0.5$. Additionally, to further prevent overfitting, we adopted linear layers in place of MLPs, as shown in~\Cref{eq:sun_exp_hiv_on,eq:sun_exp_hiv_off}. These choices showed to greatly reduce the aforementioned fluctuations. Finally, we tuned the learning rate in $\{0.01, 0.005\}$ and the batch size in $\{32, 64\}$. The result in~\Cref{tab:ogbg-hiv-baselines} corresponds to the configuration attaining best overall validation performance (ROC AUC $85.19 \pm 0.82$), with a batch size of $32$ and a learning rate of $0.01$. We note that other configurations performed comparably well. Amongst others, the configuration with a batch size of $64$ and a learning rate of $0.005$ attained a Test ROC AUC of $80.41 \pm 0.76$ with a Validation ROC AUC of $84.87 \pm 0.55$. We remark how these \SUN{} configurations perform comparably well when contrasted with state-of-the-art GNN approaches which explicitly model (molecular) rings, crucially, \emph{both} on test and validation sets, despite the non-uniform splitting procedure. As an example, CIN~\citep{bodnar2021weisfeilerB} reports a Test ROC AUC of $80.94 \pm 0.57$ with a Validation ROC AUC of $82.77 \pm 0.99$.

\begin{align}
    x_v^{v, (t+1)} =& \enspace \sigma \Big(
                  U^{2}_{t, r} \cdot x_v^{v,(t)} +
        \gamma^{0}_{t, r} \big (
                                            x_v^{v, (t)},
                                            \sum_{w \sim_v v} x_w^{v, (t)} \big ) +
                  \gamma^{1}_{t, r} \big (
                                            \sum_{h} x_v^{h, (t)},
                                            \sum_{w \sim v} \sum_{h} x_w^{h, (t)} \big ) \Big) \label{eq:sun_exp_hiv_on} \\
    x_v^{k, (t+1)} =& \enspace \sigma \Big(
                  U^{0}_t \cdot x_v^{v,(t)} +
                  U^{1}_t \cdot x_k^{k,(t)} +
                  U^{2}_t \cdot x_v^{k,(t)} +
                  \nonumber \\
        & \quad  + \gamma^{0}_t \big( x_v^{k, (t)}, \sum_{w \sim_k v} x_w^{k, (t)} \big ) +
                  \gamma^{1}_t \big (
                                \sum_{h} x_v^{h, (t)},
                                \sum_{w \sim v} \sum_{h} x_w^{h, (t)} \big ) \Big) \label{eq:sun_exp_hiv_off}
\end{align}

\subsubsection{TUDatasets}
We followed the evaluation procedure described in \citet{xu2019how}. We conducted 10-fold cross validation and reported the performances at the epoch achieving the best averaged validation accuracy across the folds. We used the same hyperparameter grid of \citet{bevilacqua2022equivariant}. We used GIN as base encoder, setting the number of layers to 4 and tuning its embedding dimension in $\{16, 32\}$. We used Adam optimizer with batch size in $\{32,128\}$, and initial learning rate in $\{0.01, 0.001\}$, which is decayed by 0.5 every 50 epochs. Training is stopped after 350 epochs. All ego-networks are of depth 2.

\subsubsection{Generalisation from limited data}
We select each architecture by tuning the hyperparameters with the entire training and validation
sets, and choosing the configuration achieving the best validation performances.
The hyperparameter grids for \SUN{} are the ones in \Cref{sec:synt,sec:zinc}. In the 4-Cycles task, for NGNN, DS- and DSS-GNN we used the same grid but we tuned the embedding dimension in $\{64, 128, 256\}$ to allow them to have a similar number of parameters as \SUN{}. For GNN-AK variants we used the best performing parameters as provided in \citet{zhao2022from}.

\revision{
\subsubsection{Ablation study}
For every ablation we tuned the embedding dimension in $\{64,96,110,128\}$ and chose the model obtaining the lowest validation MAE while still being complaint with the 500K parameter budget. The evaluation procedure and all the other hyperparameters are as specified in \Cref{sec:zinc}.
}

\newpage

%\bibliographystyleA{plainnat}
%\bibliographyA{app_references}

\end{document}